\documentclass{article}

    \usepackage[final, nonatbib]{neurips_2020}

\usepackage[utf8]{inputenc} 
\usepackage[T1]{fontenc}    
\usepackage{hyperref}       
\usepackage{url}            
\usepackage{booktabs}       
\usepackage{amsfonts}       
\usepackage{nicefrac}       
\usepackage{microtype}      

\usepackage{graphicx, xcolor}
\usepackage{subfigure}
\usepackage{mathtools}
\usepackage{tikz}
\usepackage{multirow}
\usepackage{tabularx,  ragged2e}

\usepackage{times}

\usepackage{caption}

\RequirePackage{fancyhdr}
\RequirePackage{color}
\RequirePackage{algorithm}
\RequirePackage{algorithmic}
\RequirePackage[numbers]{natbib}
\RequirePackage{eso-pic} 
\RequirePackage{forloop}

\usepackage{bm}
\usepackage{amsmath, amsthm, amssymb}
\usepackage{wrapfig}

\newtheorem{theorem}{Theorem}
\newtheorem{lemma}{Lemma}
\newtheorem{assumption}{Assumption}
\theoremstyle{definition}
\newtheorem{definition}{Definition}[section]

\newtheorem{innercustomthm}{Theorem}
\newenvironment{customthm}[1]
  {\renewcommand\theinnercustomthm{#1}\innercustomthm}
  {\endinnercustomthm}

\newcommand{\real}{\mathbb{R}}
\newcommand{\encoder}{W_1}
\newcommand{\decoder}{W_2}

\newcommand{\weightcolumn}{w}
\newcommand{\x}{X}
\newcommand{\y}{Y}

\newcommand{\loss}{\mathcal{L}}
\newcommand{\reconloss}{\frac{1}{n}||\x - \decoder \encoder \x||_F^2}
\newcommand{\indset}{\mathcal{I}}
\newcommand{\transpose}{^{\top}}
\newcommand{\identity}{I}

\newcommand{\decoderbar}{\bar{W}_2}
\newcommand{\encodertilde}{\tilde{W}_1}
\newcommand{\decodertilde}{\tilde{W}_2}
\newcommand{\vecop}[1]{\mathrm{vec}(#1)}
\newcommand{\M}{M}
\newcommand{\trace}{\mathrm{Tr}}
\newcommand{\Rij}{R_{ij}}
\newcommand{\fnorm}[1]{||{#1}||_F}
\newcommand{\sq}{^2}
\newcommand{\A}{A}
\newcommand{\cov}{\mathrm{Cov}}

\newcommand{\Q}{Q}
\newcommand{\Dmatrix}{D}
\newcommand{\avec}{a}
\newcommand{\zerovec}{0}
\newcommand{\sigmadiag}{S}

\newcommand{\hessnonuni}{H_{\Lambda}}
\newcommand{\Zvec}{\begin{bmatrix}Z_1\transpose & Z_2\end{bmatrix}}

\newcommand{\ut}{%
\begin{tikzpicture}[scale=1.2]%
\draw (0,1ex) -- (1ex,1ex);%
\draw (1ex,1ex) -- (1ex,0ex);
\draw (0,1ex) -- (1ex,0ex);%
\end{tikzpicture}%
}

\newcommand{\lt}{%
\begin{tikzpicture}[scale=1.2]%
\draw (0,0) -- (1ex,0);%
\draw (0,0) -- (0,1ex);
\draw (0,1ex) -- (1ex,0ex);%
\end{tikzpicture}%
}

\title{Regularized linear autoencoders recover the principal components, eventually}

\author{%
   Xuchan Bao, \:\:James Lucas, \:\:Sushant Sachdeva, \:\:Roger Grosse \\
   University of Toronto; \:\: Vector Institute\\
   \texttt{\{jennybao,jlucas,sachdeva,rgrosse\}@cs.toronto.edu}\\
}

\begin{document}

\maketitle

\begin{abstract}
Our understanding of learning input-output relationships with neural nets has improved rapidly in recent years, but little is known about the convergence of the underlying representations, even in the simple case of linear autoencoders (LAEs). We show that when trained with proper regularization, LAEs can directly learn the optimal representation -- ordered, axis-aligned principal components. We analyze two such regularization schemes: non-uniform $\ell_2$ regularization and a deterministic variant of nested dropout~\citep{nested-dropout}.
Though both regularization schemes converge to the optimal representation, we show that this convergence is slow due to ill-conditioning that worsens with increasing latent dimension. We show that the inefficiency of learning the optimal representation is not inevitable -- we present a simple modification to the gradient descent update that greatly speeds up convergence empirically.\footnote{The code is available at \url{https://github.com/XuchanBao/linear-ae}}
\end{abstract}

\section{Introduction}
While there has been rapid progress in understanding the learning dynamics of neural networks, most such work focuses on the networks' ability to fit input-output relationships. However, many machine learning problems require learning representations with general utility. For example, the representations of a pre-trained neural network that successfully classifies the ImageNet dataset~\citep{imagenet} may be reused for other tasks. 
It is difficult in general to analyze the dynamics of learning representations, as metrics such as training and validation accuracy reveal little about them. Furthermore, analysis through the Neural Tangent Kernel shows that in some settings, neural networks can learn input-output mappings without finding meaningful representations~\citep{ntk}.

In some special cases, the optimal representations are known, allowing us to analyze representation learning exactly. In this paper, we focus on linear autoencoders (LAE). With specially chosen regularizers or update rules, their optimal weight representations consist of ordered, axis-aligned principal directions of the input data.

It is well known that the unregularized LAE finds solutions in the principal component spanning subspace \citep{baldi1989neural}, but in general, the individual components and corresponding eigenvalues cannot be recovered. This is because any invertible linear transformation and its inverse can be inserted between the encoder and the decoder without changing the loss. \citet{ae-loss-landscape} showed that applying $\ell_2$ regularization on the encoder and decoder reduces the symmetry of the stationary point solutions to the group of orthogonal transformations. The individual principal directions can then be recovered by applying the singular value decomposition (SVD) to the learned decoder weights.

\begin{wrapfigure}{10r}{0.5\textwidth}
    \centering
    \includegraphics[width=\linewidth, trim={3.5cm 1.5cm 2.5cm 2.5cm}, clip]{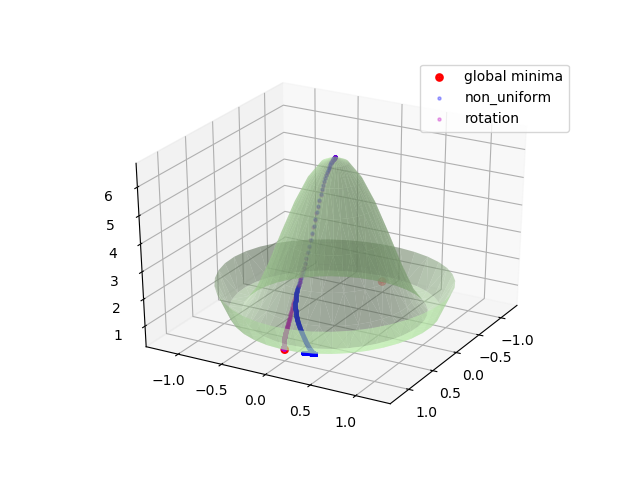}
    \caption{Visualization of the loss surface of an LAE with non-uniform $\ell_2$ regularization, plotted for a 2D subspace that includes a global optimal solution. 
    The narrow valley along the rotation direction causes slow convergence.
    Detailed discussion can be found in Section~\ref{sec:nonuni_slow_convergence}.}
    \label{fig:loss_landscape}
    \vspace{-2em}
\end{wrapfigure}

We
investigate how, with 
appropriate regularization,
gradient-based optimization can further break the symmetry, 
and \emph{directly} learn the individual principal directions.
We analyze two such regularization schemes: \emph{non-uniform} $\ell_2$ regularization and a deterministic variant of nested dropout~\citep{nested-dropout}.

The first regularization scheme we analyze applies \emph{non-uniform} $\ell_2$ regularization on the weights connected to different latent dimensions. We show that at any global minimum, an LAE with non-uniform $\ell_2$ regularization directly recovers the ordered, axis-aligned principal components. We analyze the loss landscape and show that all local minima are global minima.  The second scheme is nested dropout~\citep{nested-dropout}, which is already known to recover the individual principal components in the linear case. 

After establishing two viable models, we ask: how fast can a gradient-based optimizer, such as gradient descent, find the correct representation?  In principle, this ought to be a simple task once the PCA subspace is found, as an SVD on this low dimensional latent space can recover the correct alignment of the principal directions. However, we find that gradient descent applied to either aforementioned regularization scheme converges very slowly to the correct representation, even though the reconstruction error quickly decreases. 
To understand this phenomenon, we analyze the curvature of both objectives at their respective global minima, and show that these objectives cause ill-conditioning that worsens as the latent dimension is increased. Furthermore, we note that this ill-conditioning is nearly invisible in the training or validation loss, analogous to the general difficulty of measuring representation learning for practical nonlinear neural networks. The ill-conditioned loss landscape for non-uniform $\ell_2$ regularization is illustrated in Figure~\ref{fig:loss_landscape}.

While the above results might suggest that gradient-based optimization is ill-suited for efficiently recovering the principal components, we show that this is not the case. 
We propose a simple iterative learning rule that recovers the principal components much faster than the previous methods. The gradient is augmented with a term that explicitly accounts for ``rotation'' of the latent space, and thus achieves a much stronger notion of symmetry breaking than the regularized objectives. 

Our main contributions are as follows. 1) We characterize all stationary points of the non-uniform $\ell_2$ regularized objective, and prove it recovers the optimal representation at global minima (Section~\ref{sec:nonuni_loss_landscape},~\ref{sec:nonuni_identifiability}). 2) Through analysis of Hessian conditioning, we explain the slow convergence of the non-uniform $\ell_2$ regularized LAE to the optimal representation (Section~\ref{sec:nonuni_slow_convergence}). 3) We derive a deterministic variant of nested dropout and explain its slow convergence with similar Hessian conditioning analysis (Section~\ref{sec:nested_dropout}). 4) We propose an update rule that directly accounts for latent space rotation (Section~\ref{sec:rot_gradient}). We prove that the gradient augmentation term globally drives the representation to be axis-aligned, and the update rule has local linear convergence to the optimal representation. We empirically show that this update rule accelerates learning the optimal representation.

\vspace{-0.3cm}
\section{Preliminaries}
\vspace{-0.2cm}
\label{sec:prelim}
We consider linear models consisting of two weight matrices: an encoder $\encoder \in \real^{k \times m}$ and decoder $\decoder \in \real^{m\times k}$ (with $k < m$). The model learns a low-dimensional embedding of the data $\x \in \real^{m\times n}$ (which we assume is zero-centered without loss of generality) by minimizing the objective,
\begin{small}
\begin{equation}\label{eq:loss_ae}
    \mathcal{L}(\encoder, \decoder; \x) = \reconloss
\end{equation}
\end{small}
\vspace{-1em}

We will assume $\sigma_1\sq > \cdots > \sigma_k\sq > 0$ are the $k$ largest eigenvalues of $\frac{1}{n}\x\x\transpose$. The assumption that the $\sigma_1,\dots,\sigma_k$ are positive and distinct ensures identifiability of the principal components, and is common in this setting \citep{ae-loss-landscape}. Let $\sigmadiag = \rm diag (\sigma_1, \dots, \sigma_k)$. The corresponding eigenvectors are the columns of $U\in \real^{m \times k}$. 
Principal Component Analysis (PCA) \citep{pearson1901liii} provides a unique optimal solution to this problem that can be interpreted as the projection of the data along columns of $U$, up to sign changes to the projection directions. 
However, the minima of \eqref{eq:loss_ae} are not unique  in general \citep{ae-loss-landscape}. In fact, the objective is invariant under the transformation $(\encoder, \decoder) \mapsto (A\encoder, \decoder A^{-1})$, for any invertible matrix $A \in \real^{k \times k}$.

\textbf{Regularized linear autoencoders.}~~
\citet{ae-loss-landscape} provide a theoretical analysis of  $\ell_2$-regularized linear autoencoders, where the objective is  as follows,
\begin{equation}
\textstyle
\label{eq:uni_loss_ae}
    \mathcal{L}_\lambda(\encoder, \decoder; \x) = \mathcal{L}(\encoder, \decoder; \x) + \lambda\Vert\encoder\Vert_F^2 + \lambda\Vert\decoder\Vert_F^2.
\end{equation}

\citet{ae-loss-landscape} proved that the set of globally optimal solutions to objective~\ref{eq:uni_loss_ae} exhibit only an orthogonal symmetry through the mapping: $(\encoder, \decoder) \mapsto (O\encoder, \decoder O^\top)$, for orthogonal matrix $O$.

\vspace{-0.5cm}
\section{Related work}
\vspace{-0.2cm}

Previous work has studied the exact recovery of the principal components in settings similar to LAEs. \citet{nested-dropout} show that exact PCA can be recovered with an LAE by applying nested dropout on the hidden units. Nested dropout forces ordered information content in the hidden units. We derive and analyze a deterministic variant of nested dropout in Section~\ref{sec:nested_dropout}. Connections between VAEs and probabilistic PCA (pPCA) have been explored before \citep{dont-blame-elbo, rolinek2019variational, dai2018connections}.
In particular, \citet{dont-blame-elbo} showed that a linear variational autoencoder (VAE)~\citep{kingma2013auto} with diagonal latent covariance trained with the evidence lower bound (ELBO) can learn the axis-aligned pPCA solution~\citep{tipping1999probabilistic}. While this paper focuses on linear autoencoders and the full batch PCA problem, there exists an interesting connection between the proposed non-uniform $\ell_2$ regularization and the approach of~\citet{dont-blame-elbo}, as discussed in Appendix~\ref{appendix:connection_to_vae}. This connection was recently independently pointed out in the work of~\citet{kumar2020implicit}, who analyzed the implicit regularization effect of $\beta$-VAEs~\citep{higgins2017beta}.

\citet{ae-loss-landscape} show that an LAE with uniform $\ell_2$ regularization reduces the symmetry group from $\mathrm{GL}_k(\real)$ to $\mathrm{O}_k(\real)$. They prove that the critical points of the $\ell_2$-regularized LAE are symmetric, and characterize the loss landscape. We adapt their insights to derive the loss landscape of LAEs with non-uniform $\ell_2$ regularization, and to prove identifiability at global optima. Concurrent work~\citep{oftadeheliminating} addresses the identifiability issue in linear autoencoders by proposing a new loss function, that is a special case of deterministic nested dropout (discussed in Section~\ref{sec:nested_dropout}), with a uniform prior distribution. \citet{oftadeheliminating} show that the local minima correspond to ordered, axis-aligned representations. We show this in the general case, and additionally analyze the speed of convergence of this objective.

The rotation augmented gradient (RAG) proposed in Section~\ref{sec:rot_gradient} has connections to several existing algorithms. First, it is closely related to the Generalized Hebbian Algorithm (GHA)~\citep{generalized-hebbian-alg}, which combines  Oja's rule~\citep{oja1982simplified} with the Gram-Schmidt process. The detailed connection is discussed in Section~\ref{sec:rotation_gha}. The GHA can also be used to derive a decentralized algorithm, as proposed in concurrent work \citep{gemp2020eigengame}, which casts PCA as a competitive game.
Also, the RAG update appears to be in a similar form as the gradient of the Brockett cost function~\citep{absil2009optimization} on the Stiefel manifold, as discussed in Appendix~\ref{app:brockett}. However, the RAG update cannot be derived as the gradient of any loss function.
Our proposed RAG update bears resemblance to the gradient masking approach in Spectral Inference Networks (SpIN) \citep{pfau2018spectral}, which aims to learn ordered eigenfunctions of linear operators. The primary motivation of SpIN is to scale to learning eigenfunctions in extremely high-dimensional vector spaces. This is achieved by optimizing the Rayleigh quotient and achieving symmetry breaking through a novel application of the Cholesky decomposition to mask the gradient. This leads to a biased gradient that is corrected through the introduction of a bi-level optimization formulation for learning SpIN. RAG is not designed to learn arbitrary eigenfunctions but is able to achieve symmetry breaking without additional decomposition or bilevel optimization.

In this work, we discuss the weak symmetry breaking of regularized LAEs. \citet{bamler2018improving} address a similar problem for learning representations of time series data, which has weak symmetry in time. Through analysis of the Hessian matrix, they propose a new optimization algorithm -- Goldstone gradient descent (Goldstone-GD) -- that significantly speeds up convergence towards the correct alignment. The Goldstone-GD has interesting connection to the proposed RAG update in Section~\ref{sec:rot_gradient}. RAG is analogous to applying the first order approximation of latent space rotation as an artificial gauge field, simultaneously with the full parameter update. 
We believe this is an exciting direction for future research.

\citet{saxe2018} study the continuous-time learning dynamics of linear autoencoders, and characterize the solutions under strict initialization conditions. \citet{gidel2019implicit} extend this work along several important axes; they characterize the discrete-time dynamics for two-layer linear networks under relaxed (though still restricted) initialization conditions. Both \citet{gidel2019implicit} and \citet{arora2019implicit} also recognized a regularization effect of gradient descent, which encourages minimum norm solutions --- the latter of which provides analysis for depth greater than two. These works provide exciting insight into the capability of gradient-based optimization to learn meaningful representations, even when the loss function does not explicitly require such a representation. However, these works assume the covariance matrices of the input data and the latent code are co-diagonalizable, and do not analyze the dynamics of recovering rotation in the latent space. In contrast, in this work we study how effectively gradient descent is able to recover representations (including rotation in the latent space) in linear auto-encoders that are optimal for a designated objective.

\vspace{-0.5em}
\section{Non-uniform $\ell_2$ weight regularization}
\vspace{-0.2em}
In this section, we analyze linear autoencoders with non-uniform $\ell_2$ regularization where the rows and columns of $\encoder$ and $\decoder$ (respectively) are penalized with different weights. 
Let $0 < \lambda_1 < \cdots < \lambda_k$ be the $\ell_2$ penalty weights, and $\Lambda = \mathrm{diag}(\lambda_1, \dots, \lambda_k)$.
The objective has the following form,
\begin{small}
\begin{align}
\label{eq:nonuni-loss}
\begin{split}
    \loss_{\sigma'}(\encoder, \decoder; \x) = &\reconloss+ ||\Lambda^{1/2} \encoder||_F^2 + ||\decoder \Lambda^{1/2}||_F^2
\end{split}
\end{align}    
\end{small}
\vspace{-1em}

We prove that the objective~\eqref{eq:nonuni-loss} has an ordered, axis-aligned global optimum, which can be learned using gradient based optimization. Intuitively, by penalizing different latent dimensions unequally, we force the LAE to explain higher variance directions with less heavily penalized latent dimensions. 

The rest of this section proceeds as follows. First, we analyze the loss landscape of the objective~\eqref{eq:nonuni-loss} in section~\ref{sec:nonuni_loss_landscape}. Using this analysis, we show in section~\ref{sec:nonuni_identifiability} that the global minimum recovers the ordered, axis-aligned individual principal directions. Moreover, all local minima are global minima. Section~\ref{sec:nonuni_slow_convergence} explains mathematically the slow convergence to the optimal representation, by showing that at global optima, the Hessian of objective~\eqref{eq:nonuni-loss} is ill-conditioned.

\vspace{-0.3em}
\subsection{Loss landscape} 
\label{sec:nonuni_loss_landscape}

The analysis of the loss landscape is reminiscent of~\citet{ae-loss-landscape}. We first prove the Transpose Theorem (Theorem~\ref{th:nonuni_transpose}) for objective~\eqref{eq:nonuni-loss}. 
Then, we prove the Landscape Theorem (Theorem~\ref{th:nonuni_landscape}), which provides the analytical form of all stationary points of~\eqref{eq:nonuni-loss}.
%
\begin{theorem}
\label{th:nonuni_transpose}
(Transpose Theorem)
All stationary points of the objective~\eqref{eq:nonuni-loss} satisfy $W_1 = W_2\transpose$.
\end{theorem}
The proof is similar to that of \citet[Theorem~2.1.]{ae-loss-landscape}, and is deferred to 
Appendix~\ref{app:proofs:transpose}.

Theorem~\ref{th:nonuni_transpose} enables us to proceed with a thorough analysis of the loss landscape of the non-uniform $\ell_2$ regularized LAE model. We fully characterize the stationary points of~\eqref{eq:nonuni-loss} in the following theorem. 

\begin{theorem}[Landscape Theorem]
\label{th:nonuni_landscape}
Assume $\lambda_k < \sigma_k\sq$. All stationary points of~\eqref{eq:nonuni-loss} have the form:
\begin{small}
    \begin{align}
    \label{eq:nonuni_stationary_w1_strong}
        \encoder^* &= P (I - \Lambda \sigmadiag^{-2})^{\frac{1}{2}} U\transpose\\
    \label{eq:nonuni_stationary_w2_strong}
        \decoder^* &= U (I - \Lambda \sigmadiag^{-2})^{\frac{1}{2}} P\transpose
    \end{align}
\end{small}
\vspace{-1.5em}

where $\indset \subset \{1,\cdots, m\}$ is an index set containing the indices of the learned components, and $P \in \real^{k \times k}$ has exactly one $\pm 1$ in each row and each column whose index is in $\indset$ and zeros elsewhere.
\end{theorem}

The full proof is deferred to Appendix~\ref{app:proofs:landscape}. Here we give intuition on this theorem and a proof sketch. 

The uniform regularized objective in \citet{ae-loss-landscape} has orthogonal symmetry that is broken by the non-uniform $\ell_2$ regularization. In Theorem~\ref{th:nonuni_landscape} we prove that the only remaining symmetries are (potentially reduced rank) permutations and reflections of the optimal representation. In fact, we will show in section~\ref{sec:nonuni_identifiability} that at global minima, only reflection remains in the symmetry group.
\vspace{-0.5em}
\begin{proof}[Proof of Theorem~\ref{th:nonuni_landscape} (Sketch)] 

We consider applying a rotation matrix $\Rij$ and its inverse to $\encoder$ and $\decoder$ in the Landscape Theorem in~\citet{ae-loss-landscape}, respectively. $\Rij$ applies a rotation with angle $\theta$ on the plane spanned by the $i^{th}$ and $j^{th}$ latent dimensions. Under this rotation, the objective~\eqref{eq:nonuni-loss} is a cosine function with respect to $\theta$. In order for $\theta=0$ to be a stationary point, the cosine function must have either amplitude 0 or phase $\beta \pi$ ($\beta \in \mathbb{Z})$. 
Finally, we prove that in the potentially reduced rank latent space, the symmetries are reduced to only permutations and reflections. 
\end{proof}

\subsection{Recovery of ordered principal directions at global minima}
\label{sec:nonuni_identifiability}

Following the loss landscape analysis, we prove that the global minima of~\eqref{eq:nonuni-loss} correspond to ordered individual principal directions in the weights. Also, all local minima of~\eqref{eq:nonuni-loss} are global minima.

\begin{theorem}
\label{th:nonuni_global_optima}
Assume $\lambda_k < \sigma_k\sq$. The minimum value of~\eqref{eq:nonuni-loss} is achieved if and only if
$\encoder$ and $\decoder$ are equal to~\eqref{eq:nonuni_stationary_w1_strong} and~\eqref{eq:nonuni_stationary_w2_strong}, with full rank and diagonal $P$.
Moreover, all local minima are global minima.
\end{theorem}
$P$ being full rank and diagonal corresponds to the columns of $\decoder$ (and rows of $\encoder$) being ordered, axis-aligned principal directions. The full proof is shown in Appendix~\ref{app:proofs:global_minima}. Below is a sketch.
\vspace{-1em}
\begin{proof}[Proof (Sketch)]
Extending the proof for Theorem~\ref{th:nonuni_landscape}, in order for $\theta = 0$ to be a local minimum, we first show that $P$ must be full rank.
Then, we show that the rows of $\encoder^*$ (and columns of $\decoder^*$) must be sorted in strictly descending order of magnitude, hence $P$ must be diagonal. It is then straightforward to show that the global optima are achieved if and only if $P$ is diagonal and full rank, and they correspond to ordered $k$ principal directions in the rows of the encoder (and columns of the decoder). Finally, we show that there does not exist a local minimum that is not global minimum.
\end{proof}

\vspace{-0.5em}
\subsection{Slow convergence to global minima}
\label{sec:nonuni_slow_convergence}
Theorem~\ref{th:nonuni_global_optima} ensures that a (perturbed) gradient based optimizer that efficiently escapes saddle points will eventually converge to a global optimum~\citep{pmlr-v40-Ge15, pmlr-v70-jin17a}. However, we show in this section that this convergence is slow, due to ill-conditioning at global optima.

To gain better intuition about the loss landscape, consider Figure~\ref{fig:loss_landscape}. The loss is plotted for a 2D subspace that includes a globally optimal solution of $\encoder$ and $\decoder$. More precisely, we use the parameterization $\encoder = \alpha O (I - \Lambda \sigmadiag^{-2})^{\frac{1}{2}} U\transpose$, and $\decoder = \encoder\transpose$, where $\alpha$ is a scalar, and $O$ is a $2 \times 2$ rotation matrix parameterized by angle $\theta$. The $xy$\nobreakdash-coordinate is obtained by $(\alpha \cos \theta, \alpha \sin \theta)$.

In general, narrow valleys in the loss landscape cause slow convergence. In the figure, we optimize $\encoder$ and $\decoder$ on this 2D subspace. We observe two distinct stages of the learning dynamics. The first stage is fast convergence to the correct subspace -- the approximately circular ``ring" of radius 1 with low loss. The fast convergence results from the steep slope along the radial direction. After converging to the subspace, there comes the very slow second stage of finding the optimal rotation angle --- by moving through the narrow nearly-circular valley. This means that the symmetry breaking caused by the non-uniform $\ell_2$ regularization is a weak one.

We now formalize this intuition for general dimensions. 
The slow convergence to axis-aligned solutions is confirmed experimentally in full linear autoencoders in Section~\ref{sec:experiments}.
\subsubsection{Explaining slow convergence of the rotation}

Denote the Hessian of objective~(\ref{eq:nonuni-loss}) by $H$, and the largest and smallest eigenvalues of $H$ by $s_{\max}$ and $s_{\min}$ respectively. At a local minimum, the condition number $s_{\max}(H) / s_{\min}(H)$ determines the local convergence rate of gradient descent. Intuitively, the condition number characterizes the existence of narrow valleys in the loss landscape. Thus, we analyze the conditioning of the Hessian to better understand the slow convergence under non-uniform regularization.

In order to demonstrate ill-conditioning, we will lower bound the condition number through a lower bound on the largest eigenvalue, and an upper bound on the smallest. 
This is achieved by finding two vectors and computing the Rayleigh quotient, $f_H(v) = v^\top H v / v^\top v$ for each of them. Any Rayleigh quotient value is an upper (lower) bound on the smallest (largest) eigenvalue of $H$.

Looking back to Figure~\ref{fig:loss_landscape}, we notice that the high-curvature direction is radial and corresponds to rescaling of the learned components while the low-curvature direction corresponds to rotation of the component axes. We compute the above Rayleigh quotient along these directions and combine to lower bound the overall condition number. The detailed derivation can be found in Appendix~\ref{app:hessian_curvature}. Ultimately, we show that the condition number can be lower bounded by,
\begin{small}
\[\frac{2 (k-1) (\sigma_1\sq - \sigma_k\sq) \sum_{i=2}^{k-1} (\sigma_i\sq - \sigma_k\sq)}{\sigma_1\sq \sigma_k\sq}.\]
\end{small}
\vspace{-1em}

Depending on the distribution of the $\sigma$ values, as $k$ grows, the condition number quickly worsens.
\footnote{Note that the lower bound is derived assuming the $\lambda$ values are optimally chosen, when $\sigma$ values are known. In practice, this is generally infeasible because 1) we do not have access to the $\sigma$ values, and 2) the $\lambda$ values that minimize the Hessian condition number at global minima may slow down the earlier phase of training, when the weights are far from the global optima, as shown experimentally in Appendix~\ref{app:optimal_lambda}. The difficulty of choosing an optimal set of $\lambda$ values contributes to the weakness of symmetry breaking by the non-uniform $\ell_2$ regularization.} This effect is observed empirically in Figure~\ref{fig:mnist_vary_hidden_size}.
\section{Deterministic nested dropout}
\label{sec:nested_dropout}
The second regularization scheme we study is a deterministic variant of nested dropout~\citep{nested-dropout}. 
Nested dropout is a stochastic algorithm for learning ordered representations in neural networks. In an LAE with $k$ hidden units, a prior distribution $p_B(\cdot)$ is assigned over the indices $1, \dots, k$. When applying nested dropout, first an index $b \sim p_B(\cdot)$ is sampled, then all hidden units with indices $b+1, \dots, k$ are dropped. By imposing this dependency in the hidden unit dropout mask, nested dropout enforces an ordering of importance in the hidden units. \citet{nested-dropout} proved that the global optimum of the nested dropout algorithm corresponds to the ordered, axis-aligned representation.

We propose a deterministic variant to the original nested dropout algorithm on LAEs, by replacing the stochastic loss function with its expectation. Taking the expectation eliminates the variance caused by stochasticity (which prevents the original nested dropout algorithm from converging to the exact PCA subspace), thereby making it directly comparable with other symmetry breaking techniques. Define $\pi_b$ as the operation setting hidden units with indices $b+1, \dots, k$ to zero. We define the loss here, and derive the analytical form in Appendix~\ref{app:deterministic_nd}.
\begin{small}
\begin{align}
\label{eq:nd_expected_form}
\mathcal{L}_{\rm ND}(\encoder,\decoder; \x) = \mathbb{E}_{b\sim p_B(\cdot)} \big[\frac{1}{2n} \fnorm{\x - \decoder \pi_b (\encoder \x)}\sq \big] 
\end{align}
\end{small}
\vspace{-1em}

To find out how fast objective~\eqref{eq:nd_expected_form} is optimized with gradient-based optimizer, we adopt similar techniques as in Section~\ref{sec:nonuni_slow_convergence} to analyze the condition number of the Hessian at the global optima. Derivation details are shown in Appendix~\ref{app:nd_hessian_curvature}. The condition number is lower bounded by $\frac{8 \sigma_1\sq (k-1)\sq}{\sigma_1\sq - \sigma_k\sq}$.

Note that the lower bound assumes that the prior distribution $p_B(\cdot)$ is picked optimally with knowledge of $\sigma_1,\dots,\sigma_k$. However, in practice we do not have access to $\sigma_1,\dots,\sigma_k$ a priori, so the lower bound is loose. Nevertheless, the condition number grows at least quadratically in the latent dimension $k$. While the deterministic nested dropout might find the optimal representation efficiently in low dimensions, it fails to do so when $k$ is large. We confirm this observation empirically in Section~\ref{sec:experiments}.

\section{Rotation augmented gradient for stronger symmetry breaking}
\label{sec:rot_gradient}

The above analysis of regularized objectives may suggest that learning the correct representation in LAEs is inherently difficult for gradient-descent-like update rules. We now show that this is not the case by exhibiting a simple modification to the update rule which recovers the rotation efficiently. 
In particular, since learning the rotation of the latent space tends to be slow for non-uniform $\ell_2$ regularized LAE, we propose the rotation augmented gradient (RAG), which explicitly accounts for rotation in the latent space, as an alternative and more efficient method of symmetry breaking.

The RAG update is shown in Algorithm~\ref{algo:rotation_gradient}. Intuitively, RAG applies a simultaneous rotation on $\encoder$ and $\decoder$, aside from the usual gradient descent update of objective~\eqref{eq:loss_ae}. To see this, notice that $\A_t$ is skew-symmetric, so its matrix exponential is a rotation matrix. 
By Taylor expansion, we can see that RAG applies a first-order Taylor approximation of a rotation on $\encoder$ and $\decoder$.
\begin{small}
\begin{align*}
    \exp (\frac{\alpha}{n}\A_t) &= \identity + \sum_{i=1}^{\infty} \frac{\alpha^i}{i! n^i} A_t^i
\end{align*}
\end{small}
\vspace{-1em}

The rest of this section aims to provide additional insight into RAG. Section~\ref{sec:rotation_gha} makes the connection to the Generalized Hebbian Algorithm (GHA)~\citep{generalized-hebbian-alg}, a multi-dimensional variant of Oja's rule with global convergence. Section~\ref{sec:subspace_learning_invariance} points out an important property that greatly contributes to the stability of the algorithm: the rotation term in RAG conserves the reconstruction loss. Using this insight, Section~\ref{sec:rotation-global-convergence} shows that the rotation term globally drives the solution to be axis-aligned. Finally, section~\ref{sec:local_linear_convergence} proves that RAG has local linear convergence to global minima.

\begin{algorithm}[t]
    \caption{Rotation augmented gradient (RAG)}\label{algo:rotation_gradient}
    \begin{algorithmic}
        \STATE Given learning rate $\alpha_1$\\
        \STATE Initialize $(\encoder)_0$, $(\decoder)_0$\\
        \FOR {$t = 0 \dots T-1$}
        \STATE $\nabla{(\encoder)_t} = \nabla_{\encoder} \loss((\encoder)_t, (\decoder)_t)$
        \STATE $\nabla{(\decoder)_t} = \nabla_{\decoder} \loss((\encoder)_t, (\decoder)_t)$\\
        \vspace{0.5em}
        \STATE $\y_t = (\encoder)_t \x$\\
        \STATE $\A_t = \frac{1}{2}(\ut(\y_t\y_t\transpose) - \lt(\y_t\y_t\transpose))$
        \STATE \small\emph{($\ut$ (or $\lt$) masks the lower (or upper) triangular part of a matrix (excluding the diagonal) with 0.)}
        \vspace{0.5em}
        \STATE $(\encoder)_{t+1} \gets (\identity + \frac{\alpha}{n} \A_t) (\encoder)_t - \alpha \nabla{(\encoder)_t}$\\
        \STATE $(\decoder)_{t+1} \gets (\decoder)_t (\identity - \frac{\alpha}{n} \A_t) - \alpha \nabla{(\decoder)_t}$\\
        \ENDFOR
    \end{algorithmic}
\end{algorithm}

\subsection{Connection to the Generalized Hebbian Algorithm}
\label{sec:rotation_gha}
RAG is closely related to the Generalized Hebbian Algorithm (GHA)~\citep{generalized-hebbian-alg}. To see the connection, we assume $\encoder=\decoder\transpose = W$.\footnote{$\encoder = \decoder\transpose$ is required by the GHA. For RAG, this can be achieved by using balanced initialization $(\encoder)_0 = (\decoder)_0\transpose$, as RAG stays balanced if initialized so.}
For convenience, we drop the index $t$ in Algorithm~\ref{algo:rotation_gradient}. 
As in Algorithm~\ref{algo:rotation_gradient}, $\lt$ denotes the operation that masks the upper triangular part of a matrix (excluding the diagonal) with 0. 
With simple algebraic manipulation, the GHA and the RAG updates are compared below.
\begin{small}
\begin{align*}
    \mathrm{\mathbf{GHA:}}~~&W \leftarrow W + \frac{\alpha}{n} (\y\x\transpose - \lt(\y\y\transpose)W)\\
    \begin{split}
        \mathrm{\mathbf{RAG:}}~~&W \leftarrow W + \frac{\alpha}{n} [(\y\x\transpose - \lt(\y\y\transpose)W)
        - \frac{1}{2} (\y\y\transpose - \mathrm{diag}(\y\y\transpose))W]
    \end{split}
\end{align*}
\end{small}
\vspace{-1em}

Compared to the GHA update, RAG has an additional term which, intuitively, decays certain notion of ``correlation" between the columns in $W$ to zero. This additional term is important. As we will see in Section~\ref{sec:subspace_learning_invariance}, the ``non-reconstruction gradient term'' of RAG conserves the reconstruction loss. This is a property that contributes to the training stability and that the GHA does not possess.

\subsection{Rotation augmentation term conserves the reconstruction loss}
\label{sec:subspace_learning_invariance}
An important property of RAG is that the addition of the rotation augmentation term conserves the reconstruction loss. To see this, we compare the instantaneous update for RAG and plain gradient descent on the unregularized objective~\eqref{eq:loss_ae}.

We drop the index $t$ when analyzing the instantaneous update. We use superscripts $RAG$ and $GD$ to denote the instantaneous update following RAG and gradient descent on~\eqref{eq:loss_ae} respectively. We have,
\begin{align*}
&\dot{W}_1^{RAG} = \dot{W}_1^{GD} + \frac{1}{n} A W_1,~~~\dot{W}_2^{RAG} = \dot{W}_2^{GD} - \frac{1}{n} W_2 A.
\end{align*}
Therefore, the rotation term conserves the reconstruction loss:
\begin{align*}
\frac{d}{dt} (\decoder \encoder)^{RAG} = \dot{W}_2^{RAG} \encoder + \decoder \dot{W}_1^{RAG}
&= \dot{W}_2^{GD} \encoder + \decoder \dot{W}_1^{GD} = \frac{d}{dt} (\decoder \encoder)^{GD}\\
\frac{d}{dt} \mathcal{L}(\encoder, \decoder)^{RAG} &= \frac{d}{dt} \mathcal{L}(\encoder, \decoder)^{GD}.
\end{align*}
\vspace{-1em}

This means that in RAG, learning the rotation is separated from learning the PCA subspace. The former is achieved with the rotation term, and the latter with the reconstruction gradient term. This is a desired property that contributes to the training stability. 

\subsection{Convergence of latent space rotation to axis-aligned solutions}
\label{sec:rotation-global-convergence}
The insight in Section~\ref{sec:subspace_learning_invariance} enables us to consider the subspace convergence and the rotation separately. We now prove that on the orthogonal subspace, the rotation term drives the weights to be axis-aligned. For better readability, we state this result below as an intuitive, informal theorem. The formal theorem and its proof are presented in Appendix~\ref{app:proofs:rotation_global_convergence}. 

\begin{theorem}[(Informal) Global convergence to axis-aligned solutions]
Initialized on the orthogonal subspace $\encoder=\decoder\transpose = OU\transpose$, the instantaneous limit of RAG globally converges to the set of axis-aligned solutions, and the set of ordered, axis-aligned solutions is asymptotically stable.
\end{theorem}

\vspace{-0.3em}
\subsection{Local linear convergence to the optimal representation}
\vspace{-0.3em}
\label{sec:local_linear_convergence}
We show that RAG has local linear convergence to the ordered, axis-aligned solution. We show this in the limit of instantaneous update, and make the following assumptions.

\begin{wrapfigure}{r}{0.57\textwidth}
\vspace{-1.2em}
\fbox{\begin{minipage}{0.55\textwidth}

\begin{assumption}\label{assumption:subspace_converged}
The PCA subspace is recovered, i.e. the gradient due to reconstruction loss is 0.
\end{assumption}

\begin{assumption}\label{assumption:yyt_near_diagonal}
$\y\y\transpose$ is diagonally dominant with factor $0 < \epsilon \ll 1$, i.e.
    $
       \sum_{j\neq i} |(\y\y\transpose)_{ij}| <  \epsilon \cdot (\y\y\transpose)_{ii}
    $.
\end{assumption}

\begin{assumption}\label{assumption:ordered_diagonal}
The diagonal elements of $\y\y\transpose$ are positive and sorted in strict descending order, i.e. $ \forall~i < j$, $(\y\y\transpose)_{ii} > (\y\y\transpose)_{jj} > 0$.
\end{assumption}
\end{minipage}}
\vspace{-1.2em}
\end{wrapfigure}
It is reasonable to make Assumption~\ref{assumption:subspace_converged}, since learning the PCA subspace is usually much more efficient than the rotation. Also, Section~\ref{sec:subspace_learning_invariance} has shown that the rotational update term conserves the reconstruction loss, thus can be analyzed independently. Assumptions~\ref{assumption:yyt_near_diagonal} and~\ref{assumption:ordered_diagonal} state that we focus on the convergence \emph{local} to the ordered, axis-aligned solution.

\begin{definition}
The ``non-diagonality'' of a matrix $M \in \real^{k \times k}$ is $Nd(M) = \sum_{i=1}^{k}\sum_{j=1, j\neq i}^k \vert M_{ij}|$.
\end{definition}

\begin{theorem}[Local Linear Convergence]
\label{th:local_linear_convergence}
Let $g = \min_{i, j, i\neq j} \frac{1}{n}|(\y\y\transpose)_{ii} - (\y\y\transpose)_{jj}|$.
With Assumption~\ref{assumption:subspace_converged}-\ref{assumption:ordered_diagonal} and in the instantaneous limit ($\alpha \rightarrow 0$), for an LAE updated with RAG, $Nd(\frac{1}{n}\y\y\transpose)$ converges to 0 with an instantaneous linear rate of $g$.
\end{theorem}
\vspace{-0.3em}

The proof is deferred to Appendix~\ref{app:proofs:rotation_linear_convergence}.
Note that the optimal representation corresponds to diagonal $\frac{1}{n}\y\y\transpose$ with ordered diagonal elements, which RAG has local linear convergence to.
Note that near the global optimum, $g$ is approximately the smallest ``gap" between the eigenvalues of $\frac{1}{n}\x\x\transpose$.

\begin{figure}[t]
    \subfigure[Axis-alignment]{\label{fig:toy_axis_alignment}\includegraphics[width=0.5\textwidth]{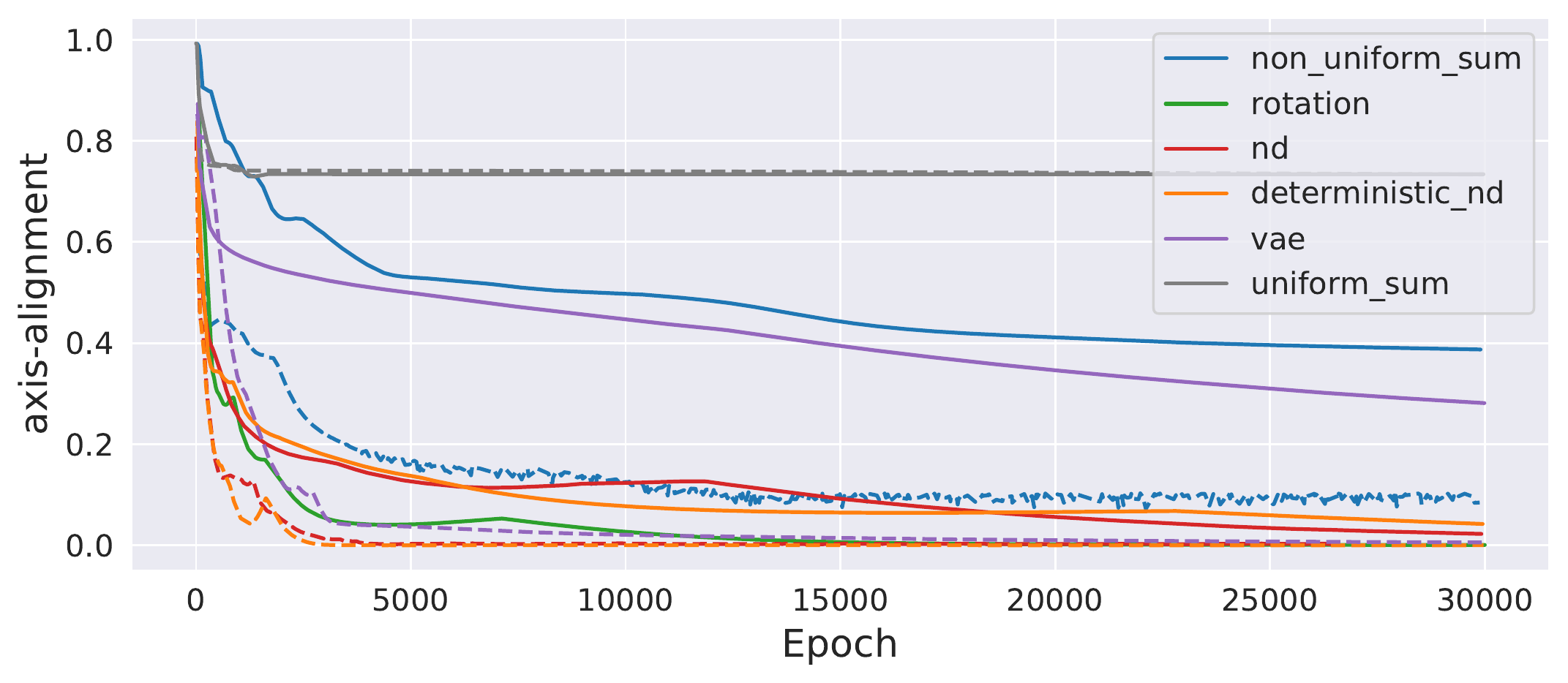}}
    \subfigure[Subspace convergence]{\label{fig:toy_subspace_convergence}\includegraphics[width=0.5\textwidth]{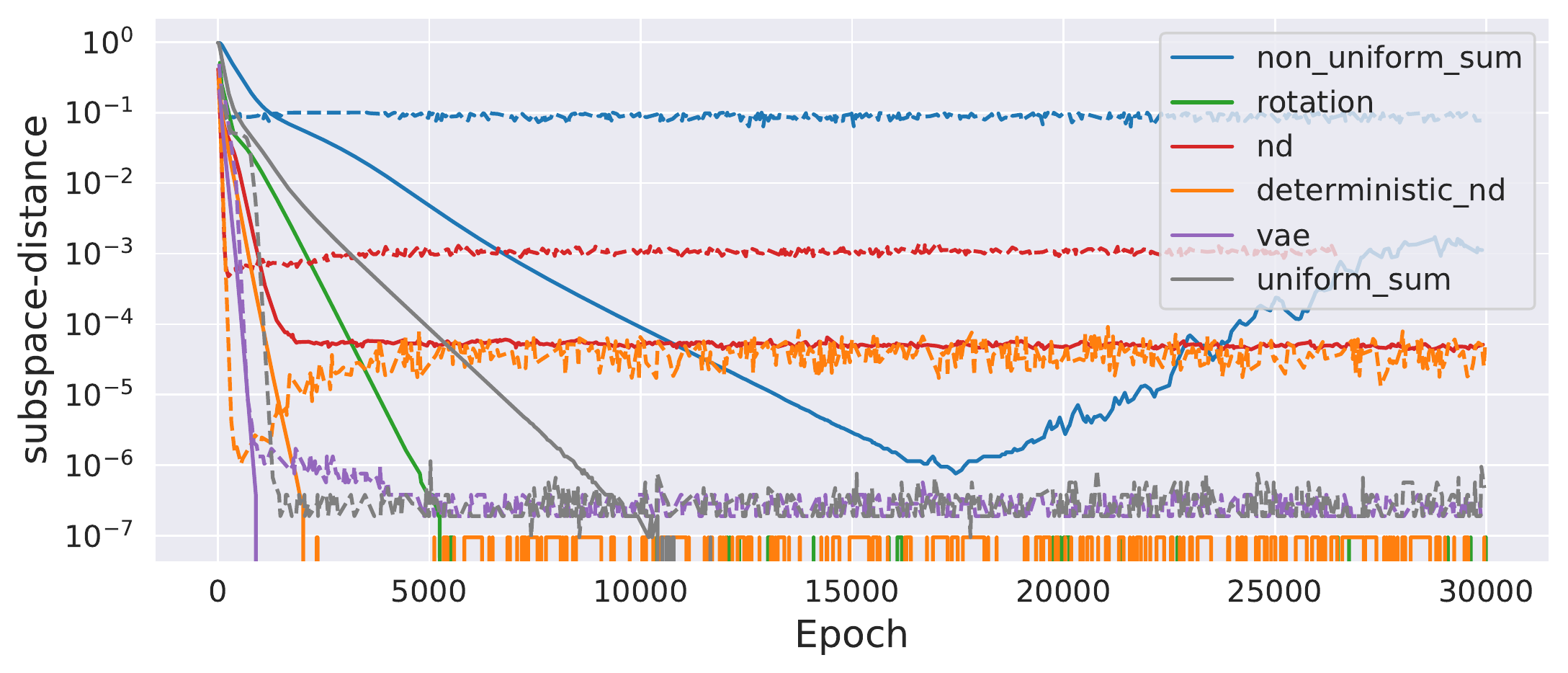}}
    \vspace{-0.8em}
    \caption{Learning dynamics of different LAE / linear VAE models trained on the MNIST ($ k=20 $). Solid lines represent models trained using gradient descent with Nesterov momentum 0.9. Dashed lines represent models trained with Adam optimizer. The learning rate for each model and optimizer has been tuned to have the fastest convergence to axis-alignment. 
    }
    \label{fig:convergence}
    \vspace{-1.5em}
\end{figure}

\vspace{-0.5em}
\section{Experiments}
\vspace{-0.3em}
\label{sec:experiments}

In this section, we seek answers to these questions:
1) What is the empirical speed of convergence of an LAE to the ordered, axis-aligned solution using gradient-based optimization, with the aforementioned objectives or update rules?
2) How is the learning dynamics affected by different gradient-based optimizers?
3) How does the convergence speed scale to different sizes of the latent representations?

First, we define the metrics for axis-alignment and subspace convergence using the learned 
$\decoder$ (Definitions~\ref{def:axis-alignment} and~\ref{def:subspace}). Definition~\ref{def:subspace} is equal to the Definition 1 in~\citet{tang2019exponentially} scaled by $\frac{1}{k}$.

\begin{definition}[Distance to axis-aligned solution]
\label{def:axis-alignment}
We define the distance to the axis-aligned solution as 
$d_{\mathrm{align}}(\decoder, U) = 1 - \frac{1}{k}\sum_{i=1}^k \max_j \frac{(U_i\transpose (\decoder)_j)\sq}{||U_i||_2\sq ||(\decoder)_j||_2\sq}$ (subscripts represent the column index).
\end{definition}

\begin{definition}[Distance to optimal subspace]
\label{def:subspace}
Let $U_{\decoder} \in \real^{m \times k}$ consist of the left singular vectors of $\decoder$. We define the distance to the optimal subspace as
$d_{\mathrm{sub}}(\decoder, U) = 1 - \frac{1}{k}\trace(U U\transpose U_{\decoder} U_{\decoder}\transpose)$.
\end{definition}
\vspace{-1em}

\paragraph{Convergence to optimal representation}
We compare the learning dynamics for six models: uniform and non-uniform $\ell_2$ regularized LAEs, LAE updated with the RAG, LAEs updated with nested dropout and its deterministic variant, and linear VAE with diagonal latent covariance~\citep{dont-blame-elbo}.

Figure~\ref{fig:convergence} and~\ref{fig:matrix_visualization} show the learning dynamics of these model on the MNIST dataset~\citep{mnist}, with $k=20$. Further details can be found in Appendix~\ref{app:exp_details}. We use full-batch training for this experiment, which is sufficient to demonstrate the symmetry breaking properties of these models. For completeness, we also show mini-batch experiments in Appendix~\ref{app:minibatch_exp}. Figure~\ref{fig:convergence} shows the evolution of the two metrics: distance to axis-alignment and to the optimal subspace, when the models are trained with Nesterov accelerated gradient descent and the Adam optimizer~\citep{kingma2014adam}, respectively. Figure~\ref{fig:matrix_visualization} visualizes the matrix $U\transpose \decoder$, and the first 20 learned principal components of MNIST (columns of $\decoder$).

\begin{wrapfigure}{r}{0.49\textwidth}
\vspace{-0.5em}
    \centering
        \includegraphics[width=0.49\textwidth]{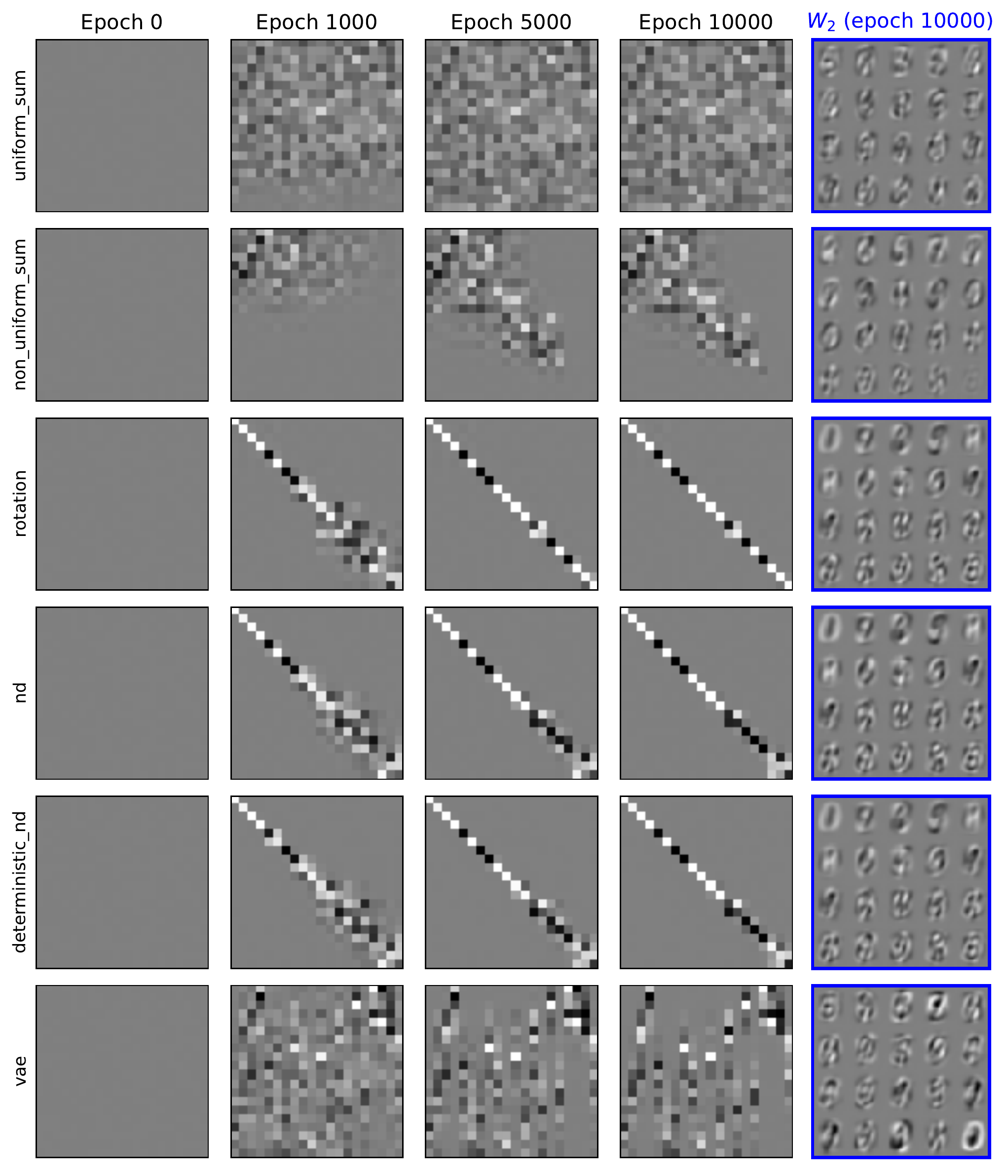}
        \caption{Visualization of $U\transpose \decoder$ and the decoder weights (last column) of LAEs trained on MNIST. All models are trained with Nesterov accelerated gradient descent. Pixel values range between -1 (black) and 1 (white). An ordered, axis-aligned solution corresponds to diagonal $U\transpose \decoder$ with $\pm 1$ diagonal entries. The linear VAE does not enforce order over the hidden dimensions, so $U\transpose \decoder$ will resemble a permutation matrix at convergence.}
        \label{fig:matrix_visualization}
        \vspace{-2em}
\end{wrapfigure}

Unsurprisingly, the uniform $\ell_2$ regularization fails to learn the axis-aligned solutions.
When optimized with Nesterov accelerated gradient descent, the regularized models, especially non-uniform $\ell_2$ regularization, has slow convergence to the axis-aligned solution. The model trained with RAG has a faster convergence. It's worth noting that Adam optimizer accelerates the learning of the regularized models and the linear VAE, but it is not directly applicable to RAG.

\begin{wrapfigure}{r}{0.548\textwidth}
    \vspace{-1.3em}
    \begin{minipage}{0.548\textwidth}
    \centering
    \includegraphics[width=0.9\linewidth]{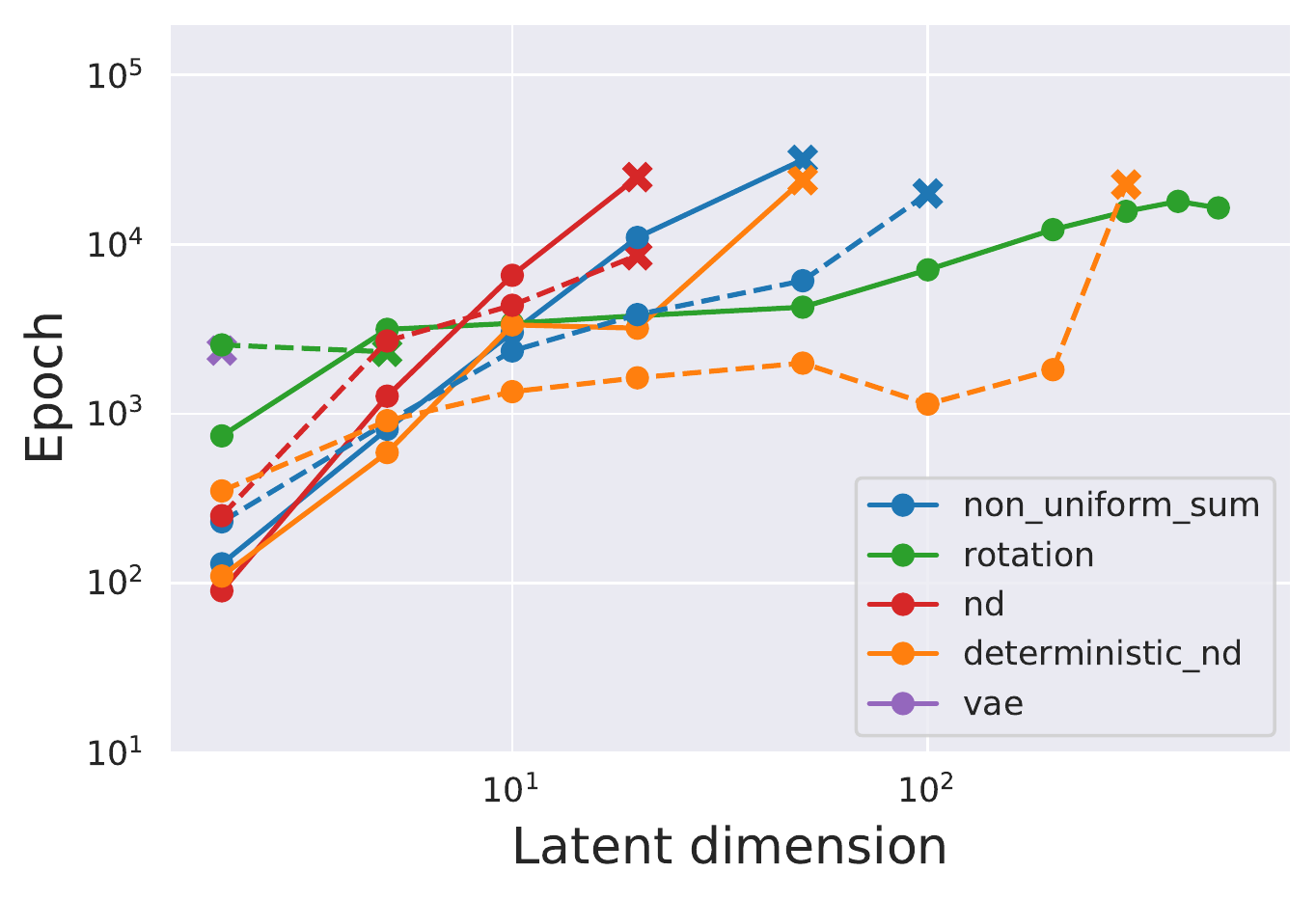}
    \caption{Epochs taken to reach 0.3 axis-alignment distance on the synthetic dataset, for different latent dimensions.
    Solid and dashed lines represent models trained with Nesterov accelerated gradient descent and Adam optimizer respectively. Cross markers indicate that beyond the current latent dimension, the models fail to reach 0.3 axis-alignment distance within 50k epochs. 
    }
    \label{fig:mnist_vary_hidden_size}
    \end{minipage}
    \vspace{-2em}
\end{wrapfigure}

\vspace{-1em}
\paragraph{Scalability to latent representation sizes}
As predicted by the Hessian condition number analysis in Section~\ref{sec:nonuni_slow_convergence} and Section~\ref{sec:nested_dropout}, we expect the models with the two regularized objectives to become much less efficient as the latent dimension grows. 
We test this 
on a synthetic dataset with input dimension $m=1000$. The data singular values are $1, \dots, m$. Full experimental details are in Appendix~\ref{app:exp_details}.
Figure~\ref{fig:mnist_vary_hidden_size} shows how quickly each model converges to the axis-alignment distance of 0.3. When optimized with the Nesterov accelerated gradient descent, the non-uniform $\ell_2$ regularization and the deterministic nested dropout scale poorly with latent dimension compared to RAG. This result is consistent with our Hessian condition number analysis. 
Although Adam optimizer provides acceleration for the regularized objectives, it does not solve the poor scaling with latent dimensions, as both regularized models fail to converge with large latent dimensions.

\vspace{-0.5em}
\section{Conclusion}
\vspace{-0.3em}
Learning the optimal representation in an LAE amounts to symmetry breaking, which is central to general representation learning. In this work, we investigated several algorithms that learn the optimal representation in LAEs, and analyze their strength of symmetry breaking. We showed that naive regularization approaches are able to break the symmetry in LAEs but introduce ill-conditioning that leads to slow convergence. The alternative algorithm we proposed, the rotation augmented gradient (RAG), guarantees convergence to the optimal representation and overcomes the convergence speed issues present in the regularization approaches. Our theoretical analysis provides new insights into the loss landscape of representation learning problems and the algorithmic properties required to perform gradient-based learning of representations.

\section*{Broader Impact}
The contribution of this work is the theoretical understanding of learning the optimal representations in LAEs with gradient-based optimizers. We believe that the discussion of broader impact is not applicable to this work.

\begin{ack}We thank Jonathan Bloom, Richard Zemel, Juhan Bae and Cem Anil for helpful discussions.

XB is supported by a Natural Sciences and Engineering Research council (NSERC) Discovery Grant. JL is supported by grants from NSERC and Samsung. SS's research is supported in part by an NSERC Discovery grant. RG acknowledges support from the CIFAR Canadian AI Chairs program. Part of this research was conducted when SS and RG were visitors at the Special year on Optimization, Statistics, and Theoretical Machine Learning at the School of Mathematics, Institute for Advanced Study, Princeton. 
\end{ack}
\bibliography{references}
\bibliographystyle{abbrvnat}

\appendix
\onecolumn
\section{Table of Notation}
\label{app:notation}

\begin{table}[h]
    \centering
    \noindent\setlength\tabcolsep{4pt}\setlength{\extrarowheight}{5pt}%
    \begin{tabularx}{0.9\linewidth}{c|*{1}{>{\RaggedRight\arraybackslash}X}}
         &  \textbf{Description} \\\hline
        $k$ & Number of latent dimensions in hidden layer of autoencoder \\
        $m$ & Number of dimensions of input data \\
        $n$ & Number of datapoints \\
        $\encoder \in \real^{k \times m}$ & Encoder weight matrix \\
        $\decoder \in \real^{m \times k}$ & Decoder weight matrix \\
        $\x \in \real^{m \times n}$ & Data matrix, with $n$ $m$-dimensional\\
        $\Vert \cdot \Vert_F$ & The Frobenius matrix norm\\
        $\sigma_i\sq$ & The $i^{th}$ eigenvalues of the empirical covariance matrix $\frac{1}{n}\x\x^\top$\\
        $S$ & Diagonal matrix with entries $\sigma_1,\ldots,\sigma_k$ \\
        $U$ & Matrix whose columns are the eigenvectors of $\frac{1}{n}\x\x^\top$, in descending order of corresponding eigenvalues\\
        $\mathcal{L}$ & Linear autoencoder reconstruction loss function\\
        $\mathcal{L}_\lambda$ & Linear autoencoder loss function with uniform $\ell_2$ regularization\\
        $\mathcal{L}_{\sigma'}$ & Linear autoencoder loss function with uniform $\ell_2$ regularization\\
        $\Lambda$ & Diagonal matrix containing non-uniform regularization weights, $\textrm{diag}(\lambda_1,\ldots,\lambda_k)$\\
        $H$ & The Hessian matrix of the non-uniform regularized loss (unless otherwise specified)\\
        $s_{\max}(H)$ & The largest eigenvalue of $H$\\
        $s_{\min}(H)$ & The smallest eigenvalue of $H$\\
        $f_A(v)$ & The Rayleigh quotient, $f_A(v) = v^\top A v / v^\top v$\\
        $\mathcal{L}_{\rm ND}$ & Linear autoencoder with nested dropout loss function\\
        $Y$ & $Y = \encoder\x$, latent representation of linear autoencoder \\
        $\alpha$ & Learning rate of gradient descent optimizer\\
        $\ut(\cdot)$ / $\lt(\cdot)$ & Operator that sets the lower or upper triangular part (excluding the diagonal) to zero of a matrix (respectively)
    \\ \hline
    \end{tabularx}
	\vspace{1em}
    \caption{Summary of notation used in this manuscript, ordered according to introduction in main text.}
    \label{tab:notation}
\end{table}
\section{Conditioning analysis for the regularized LAE}
\label{app:hessian_curvature}

Our goal here is to show that the regularized LAE objective is ill-conditioned, and also to provide insight into the nature of the ill-conditioning. In order to demonstrate ill-conditioning, we will prove a lower bound on the condition number of the Hessian at a minimum, by providing a lower bound on the largest singular value of the Hessian and an upper bound on the smallest singular value. The largest eigenvalue limits the maximum stable learning rate, and thus if the ratio of these two terms is very large then we will be forced to make slow progress in learning the correct rotation. Throughout this section, we will assume that the data covariance is full rank and has unique eigenvalues.

Since the Hessian $H$ is symmetric, we can compute bounds on the singular values through the Rayleigh quotient, $f_H(v) = v^\top H v / v^\top v$. In particular, for any vector $v$ of appropriate dimensions,
\begin{equation}\label{eqn:rayleigh_bound}
    s_{\min}(H) \leq f_H(v) \leq s_{\max}(H).
\end{equation}
Thus, if we exhibit two vectors with Rayleigh quotients $f_H(v_1)$ and $f_H(v_2)$, then the condition number is lower bounded by $f_H(v_1)/f_H(v_2)$.

In order to compute the Rayleigh quotient, we compute the second derivatives of auxiliary functions parameterizing the loss over paths in weight-space, about the globally optimal weights. This can be justified by the following Lemma,

\begin{lemma}\label{lemma:path_curvature}
Consider smooth functions $\ell: \real^n \rightarrow \real$, and $g: \real \rightarrow \real^n$, with $h = \ell \circ g : \real \rightarrow \real$. Assume that $g(0)$ is a stationary point of $\ell$, and let $H$ denote the Hessian of $\ell$ at $g(0)$. Writing $f_H(v)$ for the Rayleigh quotient of $H$ with $v$, we have,
\[f_H(v) = \frac{h''(0)}{ J_g(0)^\top J_g(0)},\]
where $J_g$ denotes the Jacobian of $g$.

\end{lemma}

\begin{proof}
The proof is a simple application of the chain rule and Taylor's theorem. Let $u = g(\alpha)$, then,
\[\frac{d^2h}{d\alpha^2} = J_g^\top \frac{\partial^2 \ell}{\partial^2 u} J_g + \frac{d\ell}{du}^\top\frac{d^2 g}{d\alpha^2}.\]
Thus, by Taylor expanding $h$ about $\alpha=0$,
\begin{align}
    h(\alpha) &= h(0) + \alpha \frac{dh}{d\alpha}\bigg\vert_{\alpha=0} + \frac{\alpha^2}{2} \frac{d^2 h}{d\alpha^2}\bigg\vert_{\alpha=0} + o(\alpha^3)\\
    &= h(0) + \alpha \left(\frac{d\ell}{du}^\top J_g\right)\Bigg\vert_{\alpha=0} + \frac{\alpha^2}{2} \left(J_g^\top \frac{\partial^2 \ell}{\partial^2 u} J_g + \frac{d\ell}{du}^\top\frac{d^2 g}{d\alpha^2}\right)\Bigg\vert_{\alpha=0} + o(\alpha^3)
\end{align}
Now, note that as $g(0)$ is a stationary point of $\ell$, thus $\frac{d\ell}{du}\big\vert_{\alpha=0} = 0$. Differentiating the Taylor expansion twice with respect to $\alpha$, and evaluating at $\alpha=0$ gives,
\[h''(0) = J_g(0)^\top H J_g(0)\]
Thus, dividing by $J_g(0)^\top J_g(0)$ we recover the Rayleigh quotient at $H$.
\end{proof}

\paragraph{Scaling curvature} The first vector for which we compute the Rayleigh quotient corresponds to rescaling of the largest principal component at the global optimum. To do so, we define the auxiliary function,
\begin{align*}
h_Z (\alpha) 
&= \loss_{\sigma'}(\encoder + \alpha Z_1, \decoder + \alpha Z_2; \x) \\
&= \frac{1}{2n}\Vert\x - (\decoder + \alpha  Z_2)(\encoder + \alpha Z_1)\x \Vert_F^2 + \frac{1}{2}||\Lambda^{1/2} (\encoder + \alpha Z_1)||_F^2 + \frac{1}{2}||(\decoder + \alpha Z_2) \Lambda^{1/2}||_F^2  
\end{align*}

Thus, by Lemma~\ref{lemma:path_curvature}, we have $h''_Z(0) = \frac{1}{2}\vecop{\Zvec}^\top H \vecop{\Zvec}$, that is, the curvature evaluated along the direction $\Zvec$. It is easy to see that $h_Z(\alpha)$ is a polynomial in $\alpha$, and thus to evaluate $h''_Z(0)$ we need only compute the terms of order $\alpha^2$ in $h_Z$. Writing the objective using the trace operation,
\begin{align*}
\begin{split}
    h_Z (\alpha) =&  \frac{1}{2n}\trace \big[ (\x - \decoder\encoder\x - \alpha(Z_2\encoder + \decoder Z_1)\x - \alpha^2 Z_2 Z_1 \x)^\top\\
    &~~~~~~~~~~~~(\x - \decoder\encoder\x - \alpha(Z_2\encoder + \decoder Z_1)\x - \alpha^2 Z_2 Z_1 \x) \big] \\
    &+ \frac{1}{2}\trace \left[ \Lambda ((\encoder + \alpha Z_1)(\encoder + \alpha Z_1)^\top + (\decoder + \alpha Z_2)^\top (\decoder + \alpha Z_2))\right]
\end{split}
\end{align*}
Collecting the terms in $\alpha^2$:
\begin{align*}
  \alpha^2 \big(&\frac{1}{2n}\trace\left[ \x^\top (Z_2 \encoder + \decoder Z_1)^\top (Z_2 \encoder + \decoder Z_1)\x - 2\x^\top Z_1^\top Z_2^\top (\x-\decoder\encoder \x)\right] \\
  &+ \frac{1}{2}\trace\left[\Lambda (Z_1 Z^\top_1 + Z_2^\top Z_2)\right]\big)  
\end{align*}
Above we have used permutation invariance of the trace operator to collect together two middle terms.

At this point, we proceed by analyzing the Rayleigh quotient along the direction corresponding to scaling the leading principal component column, at the global optimum:
\[\encoder^\top = \decoder = W = U (I - \Lambda \sigmadiag^{-2})^{\frac{1}{2}}\]

where $U$ are the eigenvectors of the data covariance, and $\sigmadiag^2$ the diagonal matrix containing the corresponding eigenvalues. Additionally, we choose $Z_1$ and $Z_2$ to contain the first column of the decoder ($\weightcolumn_1$), padded with zeros to match the dimension of $\encoder$ and $\decoder$,
\[Z_1^\top = Z_2 = Z =  \left(\begin{array}{cc}
    \weightcolumn_1 & 0 \\
\end{array}\right),\]
We will require the following identities,
\begin{align}\label{eqn:first_hess_go_id}
    \x - \decoder\encoder\x &= n U(\sigmadiag - (I - \Lambda \sigmadiag^{-2})\sigmadiag)V^\top = n U\Lambda \sigmadiag^{-1} V^\top \\
    U^\top W &= (I - \Lambda \sigmadiag^{-2})^{\frac{1}{2}} \\
    U^\top Z &= \left(\begin{array}{cc}
        \sqrt{1 - \lambda_1 \sigma_1^{-2}} & 0 \\
        0 & 0
    \end{array}\right)\\
    W^\top W &= I - \Lambda \sigmadiag^{-2}\\
    Z^\top Z &= \left(\begin{array}{cc}
        1 - \lambda_1 \sigma_1^{-2} & 0 \\
        0 & 0
    \end{array}\right)\\
    \label{eqn:last_hess_go_id}
    Z^\top W &= \left(\begin{array}{cc}
        1 - \lambda_1 \sigma_1^{-2} & 0 \\
        0 & 0
    \end{array}\right)
\end{align}

We now tackle each term in turn. Beginning with the first,
\begin{align*}
    &~~~~\trace\left( \x^\top (Z_2 \encoder + \decoder Z_1)^\top (Z_2 \encoder + \decoder Z_1)\x \right) \\
    &= \trace \left( \x\x^\top(Z W^\top + W Z^\top)(Z W^\top + W Z^\top) \right)\\
    &= n\trace \left(\sigmadiag^2 U^\top (Z W^\top + W Z^\top)(Z W^\top + W Z^\top) U \right)\\
    &= n\trace \left( \sigmadiag^2 (U^\top Z W^\top + U^\top W Z^\top)(Z (U^\top W)^\top + W (U^\top Z)^\top) \right)\\
    &= n\trace\bigl(\sigmadiag^2 ((U^\top Z)(W^\top Z) (U^\top W)^\top + (U^\top Z)(W^\top W) (U^\top Z)^\top \\
    &\:\: + (U^\top W)(Z^\top Z) (U^\top W)^\top + (U^\top W)(Z^\top W) (U^\top Z)^\top \bigr) \\
    &= 4n \sigma_1^2 (1 - \lambda_1 \sigma_1^{-2})^2
\end{align*}
For the second term,
\begin{align*}
    -2\trace\left(\x^\top Z_1^\top Z_2^\top (\x-\decoder\encoder \x)\right) &= -2n\trace \left(V \sigmadiag U^\top Z Z^\top U \Lambda\sigmadiag^{-1} V^\top\right)\\
    &= -2n\trace \left( U^\top Z Z^\top U \Lambda \right)\\
    &= -2n\lambda_1(1 - \lambda_1 \sigma_1^{-2})
\end{align*}
For the final third term,
\begin{align*}
    \trace\left(\Lambda (Z_1 Z^\top_1 + Z_2^\top Z_2)\right) &= 2\trace \left( \Lambda (Z^\top Z)\right) = 2\lambda_1 (1 - \lambda_1 \sigma_1^{-2})
\end{align*}
Combining these,
\[h''_{Z}(0) = (1 - \lambda_1 \sigma_1^{-2})\left( 4\sigma_1^2(1 - \lambda_1\sigma_1^{-2}) + 2\lambda_1 - 2\lambda_1\right) = 4\sigma_1^2(1 - \lambda_1 \sigma_1^{-2})^2\]

Using Lemma~\ref{lemma:path_curvature}, we see that to recover the Rayleigh quotient, we must divide by $\Vert \Zvec\Vert_F^2 = 2(1 - \lambda_1 \sigma_1^{-2})$. Thus, using Equation~\ref{eqn:rayleigh_bound}, we have
\[s_{\max}(H) \geq \frac{\vecop{\Zvec}^\top H \vecop{\Zvec}}{\Vert \vecop{\Zvec} \Vert_F^2} = 2\sigma_1^2(1 - \lambda_1 \sigma_1^{-2}) \geq 2(\sigma_1^2 - \sigma_k^2).\]

\paragraph{Rotation curvature}
To approximate the rotation curvature, we consider paths along the rotation manifold. This corresponds to rotating the latent space of the LAE. Using Lemma~\ref{lemma:path_curvature}, we will compute the Rayleigh quotient $f_H(t)$ for vectors $t$ on the tangent space to this rotation manifold.

Explicitly, we consider an auxiliary function of the form,
\[\gamma_{R}(\theta) = \frac{1}{2n}\Vert\x - \decoder R(\theta)^\top R(\theta)\encoder\x \Vert_F^2 + \frac{1}{2}||\Lambda^{1/2} R(\theta)\encoder||_F^2 + \frac{1}{2}||\decoder R(\theta)^\top \Lambda^{1/2}||_F^2,\]

where $R(\theta)$ is a rotation matrix parameterized by $\theta$. The first term does not depend on $\theta$, as $R$ is orthogonal. Thus, we need only compute the second derivative of the regularization terms. About the global optimum, the regularization terms can be written as,
\[\trace\left( \Lambda R(\theta) W^T W R(\theta)^T \right)\]
We will consider rotations of the $i^{th}$ and $j^{th}$ columns only (a Givens rotation). To reduce notational clutter, we write $\nu_i = (1 - \lambda_i \sigma^{-2}_i)$.
\begin{align*}
    \trace\left( \Lambda R(\theta) W^T W R(\theta)^T \right) &= \trace\left( \Lambda \left[ \begin{array}{cc}
        \nu_i \cos\theta & -\nu_j\sin\theta \\
        \nu_i \sin\theta  & \nu_j \cos\theta 
    \end{array}\right] \left[ \begin{array}{cc}
        \cos\theta & \sin\theta \\
        -\sin\theta & \cos\theta
    \end{array}\right]\right) + \sum_{l \neq i,j} \lambda_l \nu_l\\
    &= \trace\left( \Lambda \left[\begin{array}{cc}
        \nu_i\cos^2\theta  + \nu_j\sin^2\theta  & \cdot \\
        \cdot & \nu_i\sin^2\theta  + \nu_j\cos^2\theta 
    \end{array} \right]\right) + \sum_{l \neq i,j} \lambda_l \nu_l\\
    &= \lambda_i (\nu_i\cos^2\theta + \nu_j\sin^2\theta) + \lambda_j (\nu_i\sin^2\theta +  \nu_j\cos^2\theta ) + \sum_{l \neq i,j} \lambda_l \nu_l\\
    &=  \nu_i(\lambda_i - \lambda_j)\cos^2\theta  + \nu_j (\lambda_i - \lambda_j)\sin^2\theta + \sum_{l \neq i,j} \lambda_l \nu_l
\end{align*}

We proceed to take derivatives.
\[\frac{\partial}{\partial \theta}\trace\left( \Lambda R(\theta) W^T W R(\theta)^T \right) = 2\sin\theta \cos\theta (\nu_j - \nu_i)(\lambda_i - \lambda_j) = \sin 2\theta (\nu_j - \nu_i)(\lambda_i - \lambda_j)\]
Thus, the second derivative, $\gamma''(\theta)$, is given by,
\[2(\nu_j - \nu_i)(\lambda_i - \lambda_j)\cos 2\theta\]
Which, when evaluated at $\theta=0$, gives,
\[\gamma''(0) = 2(\nu_j - \nu_i)(\lambda_i - \lambda_j).\]

Per Lemma~\ref{lemma:path_curvature}, we also require the magnitude of the tangent to the path at $\theta=0$, to compute the Rayleigh quotient. At the global optimum, we have,
\begin{align*}
    \left\Vert W \frac{d}{d \theta}R(\theta)^\top\right\Vert_F^2 &= \left\Vert (I - \Lambda \sigmadiag^{-2})^{1/2} \frac{ d}{d \theta}R(\theta)^\top\right\Vert_F^2 \\
    &= \left\Vert \left[ \begin{array}{cc}
        \nu_{i}^{1/2} & 0 \\ 
        0 & \nu_{j}^{1/2}
    \end{array}\right] \left[ \begin{array}{cc}
        -\sin\theta & \cos\theta \\
        -\cos\theta & -\sin\theta
    \end{array}\right] \right\Vert_f^2 \\
    &= \nu_i + \nu_j
\end{align*}

Thus the Rayleigh quotient is given by,
\[f_H(t) = \frac{\nu_j - \nu_i}{\nu_i + \nu_j}(\lambda_i - \lambda_j).\]  
Without loss of generality, we will pick $i > j$, so that $\lambda_i > \lambda_j$, $\sigma_i < \sigma_j$, and $\nu_i < \nu_j$. Where the last of these inequalities follows from $\lambda_i\sigma_i^{-2} > \lambda_i \sigma^{-2}_j > \lambda_j\sigma^{-2}_j$.

\paragraph{Conditioning of the objective} We can combine the lower bound on the largest singular value with the upper bound on the smallest singular value to give a lower bound on the condition number. The ratio can be written,
\[\frac{2(\sigma_1^2 - \sigma_k^2) (\nu_i + \nu_j)}{(\lambda_i - \lambda_j)(\nu_j - \nu_i)}\]
Thus, the condition number is controlled by our choice of placement of $\{\lambda_j\}_{j=1}^k$ on the interval $(0, \sigma^2_k)$. We lower bound the condition number by the solution to the following optimization problem,
\begin{equation}
\label{eq:nonuni_hessian_minmax}
\textrm{cond}(\hessnonuni) \geq \min_{\lambda_1,\ldots,\lambda_k} \max_{i > j} \frac{2(\sigma_1^2 - \sigma_k^2) (\nu_i + \nu_j)}{(\lambda_i - \lambda_j)(\nu_j - \nu_i)}
\end{equation}
To simplify the problem, we lower bound $\nu_i + \nu_j > 2\nu_i$. Now the inner maximization can be reduced to a search over a single index by setting $i=j+1$, as the entries of $\Lambda$ and each $\nu$ are monotonic (decreasing and increasing respectively).

Further, we can see that at the minimum each of the terms $\nu_{j+1} / \left((\lambda_{j+1} - \lambda_j)(\nu_{j} - \nu_{j+1})\right)$ must be equal --- otherwise we could adjust our choice of $\Lambda$ to reduce the largest of these terms. We denote the equal value as $c_1$. Thus, we can write,
\begin{align}
\nonumber
    &\lambda_k - \lambda_1 = \sum_{j=1}^{k-1} (\lambda_{j+1} - \lambda_j)
    = \frac{1}{c_1} \sum_{j=1}^{k-1} \frac{\nu_{j+1}}{\nu_j - \nu_{j+1}}\\
\label{eq:nonuni_c1}
    &\implies c_1 = \frac{1}{\lambda_k - \lambda_1} \sum_{j=1}^{k-1} \frac{\nu_{j+1}}{\nu_j - \nu_{j+1}}
    > \frac{1}{\sigma_k\sq} \sum_{j=1}^{k-1} \frac{\nu_{j+1}}{\nu_j - \nu_{j+1}}
\end{align}

We can further bound $c_1$ by finding a lower bound for the summation in~\eqref{eq:nonuni_c1}. The minimum of~\eqref{eq:nonuni_c1} can be reached when all terms in the summation are equal. To see this, we let the value of each summation term to be $c_2 > 0$. We have,
\begin{align*}
    &\nu_{j+1} = \frac{c_2}{1 + c_2} \nu_j,~~j = 1,\dots,k-1
\end{align*}
For $l=2,\dots,k-1$, the derivative of~\eqref{eq:nonuni_c1} with respect to $\nu_l$ is zero, and the second derivative is positive.
\begin{align*}
    \frac{\partial}{\partial \nu_l} \frac{1}{\sigma_k\sq} \sum_{j=1}^{k-1} \frac{\nu_{j+1}}{\nu_j - \nu_{j+1}} 
    &= \frac{1}{\sigma_k\sq}\frac{\partial}{\partial \nu_l} \big( 
    \frac{\nu_{l-1}}{\nu_{l-1} - \nu_{l}} + \frac{\nu_{l+1}}{\nu_l - \nu_{l+1}}
    \big)\\
    &= \frac{1}{\sigma_k\sq} \big(\frac{\nu_{l-1}}{(\nu_{l-1} - \nu_l)\sq} - \frac{\nu_{l+1}}{(\nu_l - \nu_{l+1})\sq}\big)\\
    &= \frac{1}{\sigma_k\sq} \cdot \frac{1}{\nu_l} \big(
        \frac{\frac{1+c_2}{c_2}}{(\frac{1+c_2}{c_2} - 1)\sq} - \frac{\frac{c_2}{1 + c_2}}{(1 - \frac{c_2}{1 + c_2})\sq}
    \big)\\
    &= 0
\end{align*}
\begin{align*}
    \frac{\partial\sq}{\partial \nu_l\sq} \frac{1}{\sigma_k\sq} \sum_{j=1}^{k-1} \frac{\nu_{j+1}}{\nu_j - \nu_{j+1}}
    &= \frac{1}{\sigma_k\sq} \big(
    \frac{2\nu_{l-1} (\nu_{l-1} - \nu_l)}{(\nu_{l-1} - \nu_l)^4} + \frac{2\nu_{l+1} (\nu_l - \nu_{l+1})}{(\nu_l - \nu_{l+1})^4}
    \big)
    > 0
\end{align*}
Therefore, the minimum of~\eqref{eq:nonuni_c1} can be reached when all terms in the summation are equal. We bound $c_2$ as follows,
\begin{align*}
    &\nu_1 - \nu_k = \sum_{j=1}^{k-1} (\nu_j - \nu_{j+1}) = \frac{1}{c_2} \sum_{j=1}^{k-1} \nu_{j+1}\\
    &\implies c_2 = \frac{1}{\nu_1 - \nu_k} \sum_{j=1}^{k-1} \nu_{j+1}
    > \frac{1}{\nu_1} \sum_{i=2}^{k} (1 - \frac{\lambda_i}{\sigma_i\sq})
    > \sum_{i=2}^{k} \frac{\sigma_i\sq - \lambda_i}{\sigma_i\sq} 
    > \frac{1}{\sigma_1\sq} \sum_{i=2}^{k-1} (\sigma_i\sq - \sigma_k\sq)
\end{align*}
We bound the condition number by putting the above step together,
\begin{align}
\nonumber
    \mathrm{cond}(\hessnonuni) \geq 2 (\sigma_1\sq - \sigma_k\sq) c_1 
    > 2(\sigma_1\sq - \sigma_k\sq) \frac{k-1}{\sigma_k\sq} c_2
    > \frac{2 (k-1) (\sigma_1\sq - \sigma_k\sq) \sum_{i=2}^{k-1} (\sigma_i\sq - \sigma_k\sq)}{\sigma_1\sq \sigma_k\sq}
\end{align}

\section{Deterministic nested dropout derivation}
\label{app:deterministic_nd}
In this section, we derive the analytical form of the expected LAE loss of the nested dropout algorithm~\citep{nested-dropout}.

As in Section~\ref{sec:nested_dropout}, we define $\pi_b$ as the operation that sets the hidden units with indices $b+1,\dots, k$ to zero. The loss written in the explicit expectation form is,
\begin{align}
\label{eq:nd_expected_form_appendix}
    \mathcal{L}_{\mathrm{ND}} (\encoder, \decoder; \x) = 
    \mathop{\mathbb{E}}_{b\sim p_B(\cdot)} \big[ \frac{1}{2n} \fnorm{\x - \decoder \pi_b(\encoder \x)}\sq \big]
\end{align}

In order to derive the analytical form of the expectation, we replace $\pi_b$ in~\eqref{eq:nd_expected_form} with element-wise masks in the latent space. Let $m_j^{(i)}$ be 0 if the $j^{th}$ latent dimension of the $i^{th}$ data point is dropped out, and 1 otherwise. Define the mask $M \in \{0,1\}^{k \times n}$ as,
\begin{align*}
    \M = \begin{bmatrix}
        m_1^{(1)} & \cdots & m_1^{(n)}\\
        \vdots &  \ddots & \vdots\\
        m_k^{(1)} & \cdots & m_k^{(n)}
    \end{bmatrix}
\end{align*}

We rewrite~\eqref{eq:nd_expected_form_appendix} as the expectation over $M$ (``$\circ$'' denotes element-wise multiplication),
\begin{align}
\label{eq:nd_expectation_M}
    \mathcal{L}_{\mathrm{ND}}(\encoder, \decoder;\x) = \mathbb{E}_{M} \big[ \frac{1}{2n} \fnorm{\x - \decoder (M \circ \encoder \x)}\sq \big]
\end{align}

Define $\tilde{\x} \coloneqq \decoder(\M \circ \encoder \x)$. We apply to~\eqref{eq:nd_expectation_M} the bias-variance breakdown of the prediction $\tilde{\x}$,
\begin{align*}
    \mathcal{L}_{\mathrm{ND}}(\encoder, \decoder;\x)
    &\coloneqq \mathbb{E}_{\M} [\mathcal{L}_{\mathrm{ND}} (\encoder, \decoder, \M)] \\
    &= \frac{1}{2n} \mathbb{E} [\trace ((\x - \tilde{\x}) (\x - \tilde{\x})\transpose]\\
    &= \frac{1}{2n} \trace(\x\transpose\x- 2\x\transpose \mathbb{E}[\tilde{\x}] + \mathbb{E}[\tilde{\x}]\transpose \mathbb{E}[\tilde{\x}])\\
    &= \frac{1}{2n} \trace((\x - \mathbb{E}[\tilde{\x}])\transpose (\x - \mathbb{E}[\tilde{\x}])) + \frac{1}{2}\trace(\cov (\tilde{\x}))
\end{align*}

Define the marginal probability of the latent unit with index $i$ to be kept (not dropped out) as $p_i$,
\begin{align*}
    p_i = 1 - \sum_{j = 1}^{i-1} p_B(b=j)
\end{align*}

We also define the matrices $P_D$ and $P_L$ that will be used in the following derivation,
\begin{align}
\label{eq:def_pd_pl}
    P_D = \begin{bmatrix}
    p_1 \\
    & \ddots \\
    & & p_k
    \end{bmatrix}
    ,~~~~
    P_L = \begin{bmatrix}
    p_1 & p_2 &\cdots& p_k\\
    p_2 & p_2 &  & p_k\\
    \vdots & & & \vdots\\
    p_k & p_k & \cdots & p_k
    \end{bmatrix}
\end{align}

We can compute $\mathbb{E}[\tilde{\x}]$ and $\trace(\cov (\tilde{x}))$ analytically as follows,
\begin{align*}
    \mathbb{E}[\tilde{\x}] &= \mathbb{E}_{\M} [\decoder(\M \circ \encoder \x)] = \decoder P_D \encoder \x\\
    \trace(\cov (\tilde{x})) &= \frac{1}{n}
    \trace(\mathbb{E}[\tilde{\x}\tilde{\x}\transpose])
    - \frac{1}{n}\trace(\mathbb{E}[\tilde{\x}]\mathbb{E}[\tilde{\x}]\transpose)\\
    &= \frac{1}{n}
    \trace (\x\transpose \encoder\transpose 
    (\decoder\transpose \decoder \circ P_L)
    \encoder \x
    )
    - \frac{1}{n}\trace(\x\transpose \encoder\transpose P_D
    \decoder\transpose\decoder P_D
    \encoder \x)
\end{align*}

Finally, we obtain the analytical form of the expected loss,
\begin{align*}
    \mathcal{L}_{\mathrm{ND}}(\encoder, \decoder;\x)
    &= \frac{1}{2n} \trace (\x\transpose \x) 
    - \frac{1}{n} \trace (\x\transpose \decoder P_D \encoder \x)\\
    &~~~~+ \frac{1}{2n} \trace (\x\transpose \encoder\transpose 
    (\decoder\transpose \decoder \circ 
    P_L)\encoder \x
    )
\end{align*}

\section{Conditioning analysis for the deterministic nested dropout}
\label{app:nd_hessian_curvature}

In this section we present an analogous study of the curvature under the Deterministic Nested Dropout objective. We recall from Appendix~\ref{app:deterministic_nd} that the loss can be written as ($P_D$, $P_L$ as defined in~\eqref{eq:def_pd_pl}),
\begin{align*}
    \mathcal{L}_{\mathrm{ND}}(\encoder, \decoder; \x) 
    &= \frac{1}{2n} \trace (\x\transpose \x) 
    - \frac{1}{n} \trace (\x\transpose \decoder P_D \encoder \x)
    \\
    &+ \frac{1}{2n} \trace (\x\transpose \encoder\transpose 
    (\decoder\transpose \decoder \circ 
    P_L)\encoder \x
    )
\end{align*}
Let $Q = \mathrm{diag}(q1,\dots, q_k)$, where $q_i \in \real$, $q_i \neq 0$, for $i = 1,\dots,k$. The global minima of the objective are not unique, and can be expressed as,
\begin{align}
\label{eq:exp_nd_w1*}
    \encoder^* &= Q U\transpose\\
\label{eq:exp_nd_w2*}
    \decoder^* &= U Q^{-1}
\end{align}

We can adopt the same approach as in Appendix~\ref{app:hessian_curvature}. We will compute quadratic forms with the Hessian of the objective, via paths through the parameter space. We will consider paths along scaling and rotation of the parameters.

\paragraph{Scaling curvature} 
Let $g(\alpha) = \mathcal{L}_{\mathrm{ND}}(\encoder^* + \alpha Z_1, \decoder^* + \alpha Z_2; \x)$.
As in Appendix~\ref{app:hessian_curvature}, we need only compute the second order ($\alpha$) terms in $g(\alpha)$,
\begin{align}
\label{eq:exp_nd_2nd_order}
\begin{split}
    \alpha\sq [
    &- \frac{1}{n} \trace(\x\transpose Z_2 P_D Z_1 \x) 
    + \frac{1}{2n}\trace(2\x\transpose Z_1\transpose (((\decoder^*)\transpose Z_2 + Z_2 \transpose \decoder^*) \circ P_L) \encoder^* \x)\\
    &+ \frac{1}{2n}\trace(\x\transpose (\encoder^*)\transpose (Z_2\transpose Z_2 \circ P_L) \encoder^* \x)
    + \frac{1}{2n}\trace(\x\transpose Z_1\transpose ((\decoder^*)\transpose \decoder^* \circ P_L) Z_1 \x)
    ]
\end{split}
\end{align}

Let $Z = \begin{bmatrix} u_1 & 0\end{bmatrix} \in \real^{m \times k}$, where $u_1 \in \real^{m}$ is the first column of $U$. Let $Z_1\transpose = Z_2 = Z$, we have the following identity,
\begin{align}
\label{eq:exp_nd_z_identity}
    Z\transpose Z &= U \transpose Z = \mathrm{diag}(1, 0, \dots, 0) \in \real^{k\times k}
\end{align}

Substituting~\eqref{eq:exp_nd_w1*},~\eqref{eq:exp_nd_w2*} into~\eqref{eq:exp_nd_2nd_order}, and applying identity~\eqref{eq:exp_nd_z_identity}, the second order term in $g(\alpha)$ becomes,
\begin{align*}
        \frac{1}{2} \alpha\sq g''(0) &= \alpha\sq \cdot p_1 \sigma_1\sq (1 + \frac{1}{2}(q_1\sq + \frac{1}{q_1\sq})) \geq \alpha\sq \cdot 2 p_1 \sigma_1\sq\\
        \implies g''(0) &\geq 4 p_1 \sigma_1\sq 
\end{align*}

Applying Lemma~\ref{lemma:path_curvature} and notice that $\fnorm{Z}=1$, we can get a lower bound for the largest singular value of the Hessian $H$,
\begin{align*}
    s_\mathrm{max} (H) 
    &\geq \frac{
        \vecop{\begin{bmatrix}Z_1\transpose & Z_2\end{bmatrix}}\transpose
        H
        \vecop{\begin{bmatrix}Z_1\transpose & Z_2\end{bmatrix}}
    }{
        \fnorm{\begin{bmatrix}Z_1\transpose & Z_2\end{bmatrix}}\sq
    }
    =\frac{g''(0)}{2 \fnorm{Z}\sq}
    \geq 2 p_1 \sigma_1\sq
\end{align*}

\paragraph{Rotation curvature}
We use a similar approach as in Appendix~\ref{app:hessian_curvature} to get a upper bound for the smallest singular value of the Hessian matrix. We consider paths along the (scaled) rotation manifold,
\begin{align*}
    \encoder &= Q R(\theta)Q^{-1}\encoder^*\\
    \decoder &= \decoder^* Q R(\theta)\transpose Q^{-1}
\end{align*}
where $R(\theta)$ is a rotation matrix parameterized by $\theta$, representing the rotation of the $i^{th}$ and $j^{th}$ dimensions only (a Givens rotation). 
\begin{align}
\label{eq:exp_nd_rotation_loss}
\begin{split}
    &\mathcal{L}_{\mathrm{ND}}(\encoder, \decoder; \x) = \mathrm{Const}
    - \frac{1}{n}\trace
    \bigg(
        \x\transpose \decoder^*  Q R(\theta)\transpose Q^{-1} P_D Q R(\theta) Q^{-1} \encoder^* \x
    \bigg)\\
    &+ \frac{1}{2n}\trace\bigg(
        \x\transpose (\encoder^*)\transpose Q^{-1} R(\theta)\transpose Q \bigg(
            Q^{-1} R(\theta) Q (\decoder^*)\transpose \decoder^* Q R(\theta)\transpose Q^{-1} \circ P_L
        \bigg)
        Q R(\theta) Q^{-1} \encoder^* \x
    \bigg)
\end{split}
\end{align}
Without loss of generality, we consider the loss in the $2 \times 2$ case ($i^{th}$ and $j^{th}$ dimensions only), and denote all terms independent of $\theta$ as $\mathrm{Const}$. Substituting~\eqref{eq:exp_nd_w1*} and~\eqref{eq:exp_nd_w2*} into~\eqref{eq:exp_nd_rotation_loss}, 
\begin{align*}
    \mathcal{L}_{\mathrm{ND}}(\encoder, \decoder; \x)
    &= \mathrm{Const} 
    - \frac{1}{2}\trace(\begin{bmatrix}
        \sigma_i\sq\\
        & \sigma_j\sq
    \end{bmatrix}
    R(\theta)\transpose
    \begin{bmatrix}
    p_i\\
    & p_j
    \end{bmatrix}
    R(\theta)
    )\\
    &= \mathrm{Const} - \frac{1}{2} [(\sigma_i\sq p_i + \sigma_j\sq p_j) \cos\sq \theta + (\sigma_j\sq p_i + \sigma_i\sq p_j) \sin\sq \theta]
\end{align*}

We can compute the derivatives of the objective with respect to $\theta$,
\begin{align*}
    \frac{\partial}{\partial \theta} \mathcal{L}_{\mathrm{ND}} (\encoder, \decoder; \x) &= \frac{1}{2}(\sigma_i\sq - \sigma_j\sq) (p_i - p_j) \sin 2\theta\\
    \frac{\partial\sq}{\partial \theta\sq} \mathcal{L}_{\mathrm{ND}} (\encoder, \decoder; \x)\Big\vert_{\theta=0} &= (\sigma_i\sq - \sigma_j\sq) (p_i - p_j) \cos 2\theta \Big\vert_{\theta=0} = (\sigma_i\sq - \sigma_j\sq) (p_i - p_j)
\end{align*}

Also, we compute the Frobenius norm of the path derivative. We use $U_{i,j} \in \real^{m \times 2}$ to denote the matrix containing only the $i^{th}$ and $j^{th}$ columns of $U$.
\begin{align*}
    &\Big\Vert \frac{d}{d \theta} \encoder\transpose \Big\Vert_F\sq 
    = \Big\Vert U_{i,j} R(\theta + \frac{\pi}{2})\transpose Q \Big\Vert_F\sq
    = q_i\sq + q_j\sq\\
    &\Big\Vert \frac{d}{d \theta} \decoder \Big\Vert_F\sq 
    = \Big\Vert U_{i,j} R(\theta + \frac{\pi}{2})\transpose Q^{-1} \Big\Vert_F\sq
    = \frac{1}{q_i\sq} + \frac{1}{q_j\sq}\\
    \implies &\Big\Vert\frac{d}{d \theta} \begin{bmatrix}\encoder\transpose & \decoder \end{bmatrix}\Big\Vert_F\sq 
    = q_i\sq + \frac{1}{q_i\sq} + q_j\sq + \frac{1}{q_j\sq}
    \geq 4
\end{align*}

Applying Lemma~\ref{lemma:path_curvature}, we obtain an upper bound for the smallest singular value of the Hessian,
\begin{align*}
    s_{\mathrm{min}} \leq \frac{
        \frac{\partial\sq}{\partial \theta\sq} \mathcal{L}_{\mathrm{ND}} (\encoder, \decoder; \x)\Big\vert_{\theta=0}
    }{
        \Big\Vert\frac{d}{d \theta} \begin{bmatrix}\encoder\transpose & \decoder \end{bmatrix} \Big\Vert_F\sq \Big\vert_{\theta=0}
    }
    \leq \frac{(\sigma_i\sq - \sigma_j\sq) (p_i - p_j)}{4}
\end{align*}

\paragraph{Conditioning of the objective}
Combining the lower bound of the largest singular value with the upper bound of the smallest singular value of the Hessian matrix, we obtain a lower bound on the condition number,
\begin{align*}
    \frac{8 p_1 \sigma_1\sq}{(\sigma_i\sq - \sigma_j\sq)(p_i - p_j)}
\end{align*}
The condition number is controlled by the choice of the cumulative keep probabilities $p_1, \dots, p_k$. Thus, the condition number can be further lower bounded by the solution of the following optimization problem,
\begin{align*}
    \mathrm{cond}(H) \geq \min_{p_1, \dots, p_k} \max_{i > j} \frac{8 p_1 \sigma_1\sq}{(\sigma_i\sq - \sigma_j\sq)(p_i - p_j)}
\end{align*}
The inner optimization problem can be reduced to a search over a single index $i$, with $j = i + 1$. The minimum of the outer optimization problem is achieved when the inner objective is constant for all $i = 1, \dots, k-1$ (otherwise we can adjust $p_1, \dots, p_k$ to make the inner objective smaller). We denote the constant as $c$, and lower bound it as follows,
\begin{align*}
    &\frac{1}{c (\sigma_i\sq - \sigma_j\sq)} = \frac{(p_i - p_j)}{8 p_1 \sigma_1\sq},~~\forall i = 1, \dots, k-1\\
    \implies &\frac{1}{c} \sum_{i=1}^{k-1} \frac{1}{\sigma_i\sq - \sigma_j\sq} 
    = \sum_{i=1}^{k-1} \frac{(p_i - p_j)}{8 p_1 \sigma_1\sq} 
    = \frac{p_1 - p_k}{8 p_1 \sigma_1\sq}
    < \frac{1}{8 \sigma_1\sq}\\
    \implies &c > 8 \sigma_1\sq \sum_{i=1}^{k-1} \frac{1}{\sigma_i\sq - \sigma_j\sq}
    \geq \frac{8 \sigma_1\sq (k-1)\sq}{\sigma_1\sq - \sigma_k\sq}
\end{align*}

The last inequality is achieved when all terms in the summation are equal.
The lower bound of the condition number of the Hessian matrix is,
\begin{align*}
    \mathrm{cond}(H) > \frac{8 \sigma_1\sq (k-1)\sq}{\sigma_1\sq - \sigma_k\sq}
\end{align*}

Note that this lower bound will be looser if we do not have the prior knowledge of $\sigma_1, \dots, \sigma_k$, in order to set $p_1, \dots, p_k$ appropriately.

\section{Deferred proofs}
\label{appendix:deferred_proofs}

\subsection{Proof of the Transpose Theorem}
\label{app:proofs:transpose}

The proof of the transpose theorem relied on Lemma~\ref{lemma:c_positive_semi_def} (stated below). This result was essentially proved in \citet{ae-loss-landscape}. We reproduce the statement and proof here for completeness, which deviates trivially from the original proof.

\begin{lemma}
\label{lemma:c_positive_semi_def}
The matrix $C = (I-\decoder\encoder)\x\x^\top$ is positive semi-definite at stationary points.
\end{lemma}

\begin{proof}
At stationary points we have,
\[\nabla_{W_2} \mathcal{L}_{\sigma'} = 2(W_2 W_1 - I)XX^T W^T_1 + 2 W_2 \Lambda = 0\]
Multiplying on the right by $\decoder^\top$ and rearranging gives,
\[\x\x^\top (\decoder \encoder)^\top = \decoder\encoder\x\x^\top (\decoder\encoder)^\top + \decoder \Lambda \decoder^\top \]
Both terms on the right are positive definite, thus,
\[ \x\x^\top (\decoder \encoder)^\top \succeq \decoder\encoder\x\x^\top (\decoder\encoder)^\top. \]
By Lemma~B.1 in \citet{ae-loss-landscape}, we can cancel $(\decoder \encoder)^\top$ on the right\footnote{This result is a simple consequence of properties of positive semi-definite matrices.} and recover $C \succeq 0$.
\end{proof}

Using Lemma~\ref{lemma:c_positive_semi_def}, we proceed to prove Theorem~\ref{th:nonuni_transpose} (the Transpose Theorem).
\begin{proof}[Proof of Theorem~\ref{th:nonuni_transpose}]
All stationary points must satisfy,
\begin{align*}
    \nabla_{W_1} \mathcal{L}_{\sigma'} &= \frac{2}{n} W\transpose_2 (W_2 W_1 - I)XX\transpose + 2 \Lambda W_1 = 0 \\
    \nabla_{W_2} \mathcal{L}_{\sigma'} &= \frac{2}{n}(W_2 W_1 - I)XX\transpose W\transpose_1 + 2 W_2 \Lambda = 0
\end{align*}
We have,
\begin{align*}
    0 &= \nabla_{W_1} \mathcal{L}_{\sigma'} - \nabla_{W_2} \mathcal{L}_{\sigma'}\transpose \\
    &= \frac{2}{n}(W_1 - W_2\transpose)(I-W_1 W_2)XX\transpose + 2\Lambda(W_1 - W_2\transpose)
\end{align*}
By Lemma~\ref{lemma:c_positive_semi_def}, we know that $C = \frac{1}{n}(I-W_1 W_2)XX\transpose$ is positive semi-definite. Further, writing $A = W_1 - W_2\transpose$,
\[0 = v\transpose A C A\transpose v + v\transpose\Lambda AA\transpose v, \: \forall v\]

As $ACA\transpose \succeq 0$, we must have $\forall v,~v\transpose\Lambda AA\transpose v \leq 0$. Consider setting $v = e_i$, where $e_i$ is the $i^{th}$ coordinate vector in $\real^{k}$ ($i^{th}$ entry is 1, and all other entries are 0). We have,
\[e_i\transpose \Lambda AA\transpose e_i = \lambda_i \Vert A_i \Vert_2\sq \leq 0,\]
where $A_i$ denotes the $i^{th}$ row of $A$.
Since $\lambda_i > 0$, we have $A_i = 0$. Since this holds for all $i = 1, \dots, k$, we have $A = 0$.
\end{proof}

\subsection{Proof of the Landscape Theorem}
\label{app:proofs:landscape}
Before proceeding with our proof of the Landscape Theorem (Theorem~\ref{th:nonuni_landscape}), we will require the following Lemmas. We begin by proving a weaker version of the landscape theorem (Lemma~\ref{lemma:nonuni_landscape_orth}), which allows for symmetry via orthogonal transformations.

$\indset \subset \{ 1, \cdots, m\}$ contains the indices of the learned dimensions.
We define $\sigmadiag_\indset$, $\Lambda_\indset$, $U_\indset$ and $I_\indset$ similarly as in~\citet{ae-loss-landscape}.
\begin{itemize}
    \item $l = |\indset|$. $i_1 < \cdots < i_l$ are increasing indices in $\indset$. We use subscript $l$ to denote matrices of dimension $l \times l$.
    \item $\sigmadiag_\indset = \mathrm{diag}(\sigma_{i_1}, \dots, \sigma_{i_l}) \in \real^{l \times l}$, 
    $\Lambda_\indset = \mathrm{diag}(\lambda_{i_1}, \dots, \lambda_{i_l}) \in \real^{l \times l}$
    \item $U_\indset \in \real^{m\times l}$ has the columns in $U$ with indices $i_1, \dots, i_l$.
    \item $I_\indset \in \real^{m\times l}$ has the columns in the $m \times m$ identity matrix with indices $i_1, \dots, i_l$.
\end{itemize}
\begin{lemma}[Weak Landscape Theorem]
\label{lemma:nonuni_landscape_orth}
All stationary points of~\eqref{eq:nonuni-loss} have the form:
    \begin{align*}
        \encoder &= O (I_l - \Lambda \sigmadiag_\indset^{-2})^{\frac{1}{2}} U_\indset\transpose\\
        \decoder &= U_\indset (I_l - \Lambda \sigmadiag_\indset^{-2})^{\frac{1}{2}} O\transpose
    \end{align*}
    where $O \in \real^{k \times k}$ is an orthogonal matrix.
\end{lemma}

To prove Lemma~\ref{lemma:nonuni_landscape_orth}, we introduce Lemma~\ref{lemma:q_diagonal} and Lemma~\ref{lemma:q_stationary_form}.

\begin{lemma}
\label{lemma:q_diagonal}
Given a symmetric matrix $\Q \in \real^{m \times m}$, and diagonal matrix $\Dmatrix \in \real^{m \times m}$. If $\Dmatrix$ has distinct diagonal entries, and $\Q, \Dmatrix$ satisfy
\begin{align}\label{eq:q_diagonal_condition}
    2\Q\Dmatrix^2 \Q = \Q^2 \Dmatrix^2 + \Dmatrix^2 \Q^2
\end{align}
Then $\Q$ is diagonal.
\end{lemma}

\begin{proof}[\textbf{Proof of Lemma~\ref{lemma:q_diagonal}}]
We prove Lemma~\ref{lemma:q_diagonal} using induction. We use subscript $l$ to denote matrices of dimension $l \times l$.

When $l=1$, $\Q_l$ is trivially diagonal, and Equation (\ref{eq:q_diagonal_condition}) always holds.

Assume for some $l \geq 1$, $\Q_l$ is diagonal and satisfies (\ref{eq:q_diagonal_condition}) for subscript $l$.

We have for dimension $l \times l$:
\begin{align}
\label{eq:induction_condition_l}
    2\Q_l \Dmatrix_l^2 \Q_l = \Q_l^2 \Dmatrix_l^2 + \Dmatrix_l^2 \Q_l^2
\end{align}
We write $\Q_{l+1}$ and $\Dmatrix^2_{l+1}$ in the following form ($\avec \in \real^{l \times 1}$, $q, s$ are scalars)
\begin{align*}
    \Q_{l+1} = \begin{bmatrix}
        \Q_l                &   \avec\\
        \avec\transpose    &   q
    \end{bmatrix}
    ~~~~~
    \Dmatrix_{l+1}^2 = \begin{bmatrix}
        \Dmatrix_l^2          &  \zerovec \\
        \zerovec\transpose    &   d^2
    \end{bmatrix}
\end{align*}

Expand the LHS and RHS of Equation (\ref{eq:q_diagonal_condition}) for subscript $l+1$:
\begin{align}
\nonumber
    &~~~~2 \Q_{l+1} \Dmatrix_{l+1}^2 \Q_{l+1} 
    = 2
    \begin{bmatrix}
        \Q_l                &   \avec\\
        \avec\transpose    &   q
    \end{bmatrix}\begin{bmatrix}
        \Dmatrix_l^2              &   \zerovec\\
        \zerovec\transpose        &   d^2
    \end{bmatrix}
    \begin{bmatrix}
        \Q_l                &   \avec\\
        \avec\transpose    &   q
    \end{bmatrix}\\
    \label{eq:lhs_expand}
    &= 2 \begin{bmatrix}
        \Q_l \Dmatrix_l^2 \Q_l + d^2 \avec\avec\transpose
    &   \Q_l \Dmatrix_l^2 \avec + d^2 q \avec\\
        \avec\transpose \Dmatrix_l^2 \Q_l + d^2 q \avec\transpose
    &   \avec\transpose \Dmatrix_l^2 \avec + d^2 q^2
    \end{bmatrix}
\end{align}

\begin{align*}
    \Q_{l+1}^2 \Dmatrix_{l+1}^2 + \Dmatrix_{l+1}^2 \Q_{l+1}^2
    &= \begin{bmatrix}
        \Q_l                &   \avec\\
        \avec\transpose    &   q
    \end{bmatrix}^2
    \begin{bmatrix}
        \Dmatrix_l^2              &   \zerovec\\
        \zerovec\transpose        &   d^2
    \end{bmatrix}
    + \begin{bmatrix}
        \Dmatrix_l^2              &   \zerovec\\
        \zerovec\transpose        &   d^2
    \end{bmatrix} \begin{bmatrix}
        \Q_l                &   \avec\\
        \avec\transpose    &   q
    \end{bmatrix}^2\\
    &= \begin{bmatrix}
        \mathrm{RHS}_{1:l, 1:l}  &   \mathrm{RHS}_{1:l, l+1}\\
        \mathrm{RHS}_{l+1, 1:l}  &   \mathrm{RHS}_{l+1, l+1}
    \end{bmatrix}
\end{align*}

\begin{align}
    \label{eq:rhs_1_l}
    \mathrm{RHS}_{1:l, 1:l} &= \Q_l^2\Dmatrix_l^2 + \Dmatrix_l^2\Q_l^2 + \avec\avec\transpose \Dmatrix_l^2 + \Dmatrix_l^2 \avec\avec\transpose
\end{align}


Equate the 1 to $l^{th}$ row and column of LHS and RHS (top-left of Equation (\ref{eq:lhs_expand}) and (\ref{eq:rhs_1_l})), and apply the induction assumption (\ref{eq:induction_condition_l}):
\begin{align*}
    2 d^2 \avec\avec\transpose &= \avec\avec\transpose \Dmatrix_l^2 + \Dmatrix_l^2 \avec\avec\transpose\\
    \Longrightarrow \zerovec &= \avec\avec\transpose (\Dmatrix_l^2 - d^2 \identity) + (\Dmatrix_l^2 - d^2 \identity) \avec\avec\transpose\\
    \Longrightarrow 0 &= 2 a_i^2 (s_i^2 - d^2),~~\forall i = 1, \cdots, l,~~\Dmatrix_l^2 = \mathrm{diag}(s_1^2, \cdots, s_l^2)
\end{align*}

Since $\Dmatrix_{l+1}$ is a diagonal matrix with distinct diagonal entries, $s_i^2 - d^2 \neq 0$ for $\forall i = 1, \cdots, l$. Hence $\avec = \zerovec$, and $\bm{Q}_{l+1}$ is diagonal.

It's easy to check that $\avec = \zerovec$ satisfies Equation (\ref{eq:q_diagonal_condition}), hence diagonal $\Q_{l+1}$ is a valid solution.

By induction, $\Q \in \real^{m \times m}$ is diagonal.
\end{proof}

\begin{lemma}
\label{lemma:q_stationary_form}
Consider the loss function,
\begin{align*}
\begin{split}
    \tilde{L}(Q_1, Q_2) = \mathrm{tr}(Q_2Q_1\sigmadiag^2 Q\transpose_1 Q\transpose_2 - 2Q_2Q_1\sigmadiag^2 \\+ 2Q_1Q_2\Lambda + \sigmadiag^2)
\end{split}
\end{align*}
where $\sigmadiag\sq = \mathrm{diag}(\sigma_1\sq, \dots, \sigma_k\sq)$, $\Lambda = \mathrm{diag}(\lambda_1, \dots, \lambda_k)$ are diagonal matrices with distinct positive elements, and $\forall i = 1, \dots, k, \sigma_i\sq > \lambda_i$. Then all stationary points satisfying $Q_1\transpose=Q_2$ are of the form,
\[Q_1 = O(I_l -\Lambda_\indset\sigmadiag_\indset^{-2})^{\frac{1}{2}}I_\indset\transpose\]

\end{lemma}

\begin{proof}[\textbf{Proof of Lemma~\ref{lemma:q_stationary_form}}]
Taking derivatives,
\begin{align*}
    \frac{\partial \tilde{L}}{\partial Q_1} &= 2 Q_2\transpose Q_2Q_1\sigmadiag^2 - 2Q\transpose_2\sigmadiag^2 + 2 \Lambda^2Q_2\transpose = 0\\
    \frac{\partial \tilde{L}}{\partial Q_2} &= 2 Q_2Q_1\sigmadiag^2Q_1\transpose - 2\sigmadiag^2Q\transpose_1 + 2 Q_1\transpose\Lambda^2 = 0
\end{align*}
Multiplying the first equation on the left by $Q_1\transpose$, and using $Q_2=Q_1\transpose$, we get,
\begin{equation}\label{eqn:q_stationary_condition}
    Q_1\transpose Q_1 Q_1\transpose Q_1 \sigmadiag^2 - Q_1\transpose Q_1\sigmadiag^2 + Q_1\transpose \Lambda^2 Q_1 = 0
\end{equation}
Similarly, multiplying the second equation on the right by $Q_1$,
\begin{equation*}
    Q_1\transpose Q_1 \sigmadiag^2 Q_1\transpose Q_1 - \sigmadiag^2Q_1\transpose Q_1 + Q_1\transpose \Lambda^2 Q_1 = 0
\end{equation*}
Writing $Q = Q_1\transpose Q_1$, and equating through $Q_1\transpose \Lambda^2 Q_1$,
\[ Q \sigmadiag^2 Q = Q^2 \sigmadiag^2 + \sigmadiag^2 Q - Q \sigmadiag^2  \]
Taking the transpose and adding the result,
\[2Q\sigmadiag^2 Q = Q^2 \sigmadiag^2 + \sigmadiag^2 Q^2\]
Applying Lemma~\ref{lemma:q_diagonal}, we have that $Q$ is a diagonal matrix. Following this, $Q$ commutes with both $\sigmadiag^2$ and $\Lambda^2$, thus we can reduce~\eqref{eqn:q_stationary_condition} to,
\begin{align*}
    &Q^2\sigmadiag^2 = Q (\sigmadiag^2 - \Lambda^2)\\
    \Rightarrow \sigmadiag^{2}(\sigmadiag^2 - \Lambda^2)^{-1}Q&Q(\sigmadiag^2 - \Lambda^2)^{-1}\sigmadiag^2 = \sigmadiag^{2}(\sigmadiag^2 - \Lambda^2)^{-1}Q
\end{align*}
Thus, $\sigmadiag^{2}(\sigmadiag^2 - \Lambda^2)^{-1}Q$ is idempotent. From here, we can follow the proof of Proposition~4.3 in \cite{ae-loss-landscape}, with the additional use of the transpose theorem, to determine that,
\[Q_1 = O(I_l-\Lambda_\indset^2\sigmadiag_\indset^{-2})^{\frac{1}{2}}I_\indset\transpose\]
\end{proof}

\paragraph{Proof of Weak Landscape Theorem} We can now proceed with our desired result, the weak landscape theorem.

\begin{proof}[Proof of Lemma~\ref{lemma:nonuni_landscape_orth}]
Let $Q_1 = \encoder U$, and $Q_2 = U\transpose \decoder$. We can write the loss as,
\begin{equation}
\label{eq:nonuni_loss_as_trace}
    \mathcal{L}_{\sigma'} = \trace(Q_2 Q_1 \sigmadiag^2 Q_1\transpose Q_2\transpose - 2 Q_2 Q_1 \sigmadiag + 2 Q_1 Q_2 \Lambda + \sigmadiag^2) + \fnorm{\Lambda^{1/2} (Q_1 - Q_2\transpose)}^2
\end{equation}

To see this, observe that,
\begin{align*}
    \fnorm{\Lambda^{1/2} (Q_1 - Q_2\transpose)}^2 &= \trace(Q_1 Q_1\transpose\Lambda + Q_2\transpose Q_2 \Lambda - 2Q_1 Q_2 \Lambda)\\
    &= \fnorm{\Lambda^{1/2} Q_1}^2 + \fnorm{Q_2 \Lambda^{1/2}}^2 - 2 \trace(Q_1 Q_2 \Lambda)    
\end{align*}
The Transpose Theorem guarantees that the second term in~\eqref{eq:nonuni_loss_as_trace} is zero at stationary points. 
Applying Lemma~\ref{lemma:q_stationary_form}, all stationary points must be of the form:
\begin{align}
\label{eq:nonuni_stationary_encoder}
    \encoder^* &= O(I_l - \Lambda_\indset\sigmadiag_\indset^{-2})^{\frac{1}{2}}U_\indset\transpose\\
\label{eq:nonuni_stationary_decoder}
    \decoder^* &= U_\indset (I_l - \Lambda_\indset\sigmadiag_\indset^{-2})^{\frac{1}{2}}O\transpose
\end{align}
\end{proof}
\paragraph{Proof of the (Strong) Landscape Theorem} We now present our proof of the strong version of the Landscape Theorem, which removes the orthogonal symmetry present in the weaker version.

\begin{proof}[\textbf{Proof of Theorem~\ref{th:nonuni_landscape}}]
By Theorem~\ref{th:nonuni_transpose}, at stationary points, $\encoder = \decoder\transpose$. We write 
$\encoder = \begin{bmatrix}\weightcolumn_1\transpose & \weightcolumn_2\transpose & \cdots & \weightcolumn_k\transpose\end{bmatrix}\transpose$, and 
$\decoder = \begin{bmatrix}\weightcolumn_1 & \weightcolumn_2 & \cdots & \weightcolumn_k \end{bmatrix}$, 
where $\weightcolumn_i$ for $i = 1, \cdots, k$ is the $i^{th}$ column of the decoder.

Define the regularization term in the loss as $\psi (\encoder, \decoder)$.
\begin{align*}
    \psi (\encoder, \decoder) 
    = \Vert\Lambda^{1/2} \encoder\Vert_F^2 + \Vert\decoder \Lambda^{1/2}\Vert_F^2
    = 2\fnorm{\Lambda^{1/2} \encoder}\sq
\end{align*}

Let $\encodertilde = R_{ij} \encoder$ and $\decodertilde = \decoder R_{ij}\transpose$, where $R_{ij}$ is the rotational matrix for the $i^{th}$ and $j^{th}$ components.
\begin{align*}
    R_{ij} = \begin{bmatrix}
        1       &           &       &   \\
                &   \ddots  &       &   \\
                &           & \cos{\theta}  &       &-\sin{\theta}   &   \\
                &           &       &\ddots  \\
                &           & \sin{\theta}  &       &\cos{\theta}\\
                &           &       &   &   &\ddots\\
                &           &       &   &   &       &1\\
    \end{bmatrix}
\end{align*}

\begin{align*}
    \psi(\encodertilde, \decodertilde) 
    &= ||\Lambda^{1/2} \encodertilde||_2^2 + ||\decodertilde \Lambda^{1/2}||_2^2\\
    &= \trace (\Lambda^{1/2} \encodertilde \encodertilde\transpose \Lambda^{1/2}) + \trace (\Lambda^{1/2} \decodertilde\transpose \decodertilde \Lambda^{1/2})\\
    &= \trace (\Lambda^{1/2} \Rij \encoder \encoder\transpose \Rij\transpose \Lambda^{1/2}) + \trace (\Lambda^{1/2} \Rij \decoder\transpose \decoder \Rij\transpose \Lambda^{1/2})\\
    &= 2 \trace (\Lambda^{1/2} 
    \begin{bmatrix}
        \weightcolumn_1\transpose\\
        \vdots\\
        \weightcolumn_i\transpose \cos \theta - \weightcolumn_j\transpose \sin\theta\\
        \vdots\\
        \weightcolumn_i\transpose \sin \theta + \weightcolumn_j\transpose \cos\theta\\
        \vdots\\
        \weightcolumn_k\transpose
    \end{bmatrix}
    \begin{bmatrix}
        \weightcolumn_1\transpose\\
        \vdots\\
        \weightcolumn_i\transpose \cos \theta - \weightcolumn_j\transpose \sin\theta\\
        \vdots\\
        \weightcolumn_i\transpose \sin \theta + \weightcolumn_j\transpose \cos\theta\\
        \vdots\\
        \weightcolumn_k\transpose
    \end{bmatrix}\transpose
    \Lambda^{1/2}
    )\\
    &= 2 \trace (
    \begin{bmatrix}
        \lambda_1^{1/2} \weightcolumn_1\transpose\\
        \vdots\\
        \lambda_i^{1/2} (\weightcolumn_i\transpose \cos \theta - \weightcolumn_j\transpose \sin\theta)\\
        \vdots\\
        \lambda_j^{1/2} (\weightcolumn_i\transpose \sin \theta + \weightcolumn_j\transpose \cos\theta)\\
        \vdots\\
        \lambda_k^{1/2} \weightcolumn_k\transpose
    \end{bmatrix}
    \begin{bmatrix}
        \lambda_1^{1/2} \weightcolumn_1\transpose\\
        \vdots\\
        \lambda_i^{1/2} (\weightcolumn_i\transpose \cos \theta - \weightcolumn_j\transpose \sin\theta)\\
        \vdots\\
        \lambda_j^{1/2} (\weightcolumn_i\transpose \sin \theta + \weightcolumn_j\transpose \cos\theta)\\
        \vdots\\
        \lambda_k^{1/2} \weightcolumn_k\transpose
    \end{bmatrix}\transpose
    )
\end{align*}
\begin{align*}
    &= 2 [
        \lambda_i(\weightcolumn_i\transpose \cos \theta - \weightcolumn_j\transpose \sin\theta)\transpose (\weightcolumn_i\transpose \cos \theta - \weightcolumn_j\transpose \sin\theta)\\
    &~~~~ + \lambda_j (\weightcolumn_i\transpose \sin \theta + \weightcolumn_j\transpose \cos\theta)\transpose (\weightcolumn_i\transpose \sin \theta + \weightcolumn_j\transpose \cos\theta))
    + \sum_{l=1, l\neq i, i\neq j}^k \lambda_l \weightcolumn_l\transpose \weightcolumn_l
    ]\\
    &= 2 [
        (\lambda_i \weightcolumn_i\transpose \weightcolumn_i + \lambda_j \weightcolumn_j\transpose \weightcolumn_j) \cos^2 \theta
        + (\lambda_j \weightcolumn_i\transpose \weightcolumn_i + \lambda_i \weightcolumn_j\transpose \weightcolumn_j) \sin^2 \theta\\
        &~~~~~+ 4 (\lambda_j - \lambda_i) \weightcolumn_i\transpose \weightcolumn_j \sin \theta \cos \theta + \sum_{l=1, l\neq i, i\neq j}^k \lambda_l \weightcolumn_l\transpose \weightcolumn_l
    ]\\
    &= 2 [A \cos (2\theta + B) + C + \sum_{l=1, l\neq i, i\neq j}^k \lambda_l \weightcolumn_l\transpose \weightcolumn_l]
\end{align*}

Where $A, B, C$ satisfy:
\begin{align}
\label{eq:AcosB}
    A \cos B &= \frac{1}{2} (\lambda_j - \lambda_i) (\weightcolumn_j\transpose \weightcolumn_j - \weightcolumn_i\transpose \weightcolumn_i)\\
\label{eq:AsinB}
    A \sin B &= -2 (\lambda_j - \lambda_i) \weightcolumn_i\transpose \weightcolumn_j
\end{align}

In order for $\psi(\encodertilde, \decodertilde)$ to be a stationary point at $\theta = 0$, we need either of the two necessary conditions to be true for $\forall i < j$:
\begin{align*}
    \text{Condition 1: } &A = 0 
        \iff \weightcolumn_i\transpose \weightcolumn_j = 0 ~~\text{and}~~ \weightcolumn_i\transpose\weightcolumn_i = \weightcolumn_j\transpose \weightcolumn_j\\
    \text{Condition 2: } &A \neq 0 ~~\text{and}~~ B = \beta \pi,~~\beta \in \mathbb{Z} 
        \iff \weightcolumn_i\transpose \weightcolumn_j = 0 
        ~~\text{and}~~ \weightcolumn_j\transpose \weightcolumn_j \neq \weightcolumn_i\transpose \weightcolumn_i
\end{align*}

The two conditions can be consolidated to one, i.e. the columns of the decoder needs to be orthogonal.
\begin{align*}
    \forall i, j \in \{1, \cdots, k\}, ~~~~\weightcolumn_i\transpose \weightcolumn_j = 0
\end{align*}

The following Lemma uses such orthogonality to constrain the form that the matrix $O$ in~\eqref{eq:nonuni_stationary_encoder} and~\eqref{eq:nonuni_stationary_decoder} can take.

\begin{lemma}
\label{lemma:nonuni_landscape_diagonal_indset}
Let $\encoder^*$, $\decoder^*$ be in the form of~\eqref{eq:nonuni_stationary_encoder} and~\eqref{eq:nonuni_stationary_decoder}. And let
$\encoder^* = \begin{bmatrix}\weightcolumn_1\transpose & \weightcolumn_2\transpose & \cdots & \weightcolumn_k\transpose\end{bmatrix}\transpose$, and
$\decoder^* = \begin{bmatrix}\weightcolumn_1 & \weightcolumn_2 & \cdots & \weightcolumn_k \end{bmatrix}$, 
where $\weightcolumn_i \in \real^{m}$ for $i = 1, \cdots, k$ is the $i^{th}$ columns of the $\decoder^*$.

If for $\forall i, j \in \{ 1, \cdots, k\}$, $\weightcolumn_i\transpose \weightcolumn_j = 0$, then
$O$ has exactly one entry of $\pm 1$ in each row and at most one entry of $\pm 1$ in each column, and zeros elsewhere.
\end{lemma}

\begin{proof}[Proof of Lemma~\ref{lemma:nonuni_landscape_diagonal_indset}]

\begin{align}
\label{eq:svd_w2Tw2}
    &(\decoder^*)\transpose \decoder^* = 
        O (I_l - \Lambda \sigmadiag_\indset^{-2})^{\frac{1}{2}} U_\indset\transpose 
        U_\indset (I_l^2 - \Lambda \sigmadiag_\indset^{-2})^{\frac{1}{2}} O\transpose
    = O (I_l - \Lambda \sigmadiag_\indset^{-2}) O\transpose
\end{align}

Note that $(I_l - \Lambda \sigmadiag_\indset^{-2})$ is a diagonal matrix with strictly descending positive diagonal entries, so~\eqref{eq:svd_w2Tw2} is an SVD to $(\decoder^*)\transpose \decoder^*$. 

Because $\decoder^*$ has orthogonal columns, $(\decoder^*)\transpose \decoder^*$ is a diagonal matrix. There exists a permutation matrix $P_0 \in \real^{k \times k}$, such that $\decoder^* P_0\transpose$ has columns ordered strictly in descending magnitude. Let $\decoderbar^* = \decoder^* P_0\transpose$, and $\bar{O} = P_0 O$, then
\begin{align}
\nonumber
    (\decoderbar^*)\transpose \decoderbar^*
    &= (\decoder^* P_0\transpose)\transpose \decoder^* P_0\transpose \\
\nonumber
    &= P_0 O (I_l - \Lambda \sigmadiag_\indset^{-2}) O\transpose P_0\transpose\\
\label{eq:rotated_nonuni_scale}
    &= \bar{O} (I_l - \Lambda \sigmadiag_\indset^{-2}) \bar{O}\transpose\\
\label{eq:nonuni_scale}
    &= I_l - \Lambda \sigmadiag_\indset^{-2}
\end{align}

Note that $\bar{O} = P_0 O$ also have orthonormal columns, we have $\bar{O}\transpose \bar{O} = \identity$.
Let $\bar{O} = \begin{bmatrix}o_1\transpose & o_2\transpose & \cdots & o_k\transpose \end{bmatrix}\transpose$, where $o_j \in \real^{1 \times l}$ are rows of $O$.
From~\eqref{eq:rotated_nonuni_scale} and~\eqref{eq:nonuni_scale}, we have for $i \in \{1, \cdots, l\}$, $j\in \{1, \cdots, k\}$:
\begin{align*}
    & \bar{O}\transpose (I_l - \Lambda \sigmadiag_\indset^{-2}) = (I_l - \Lambda \sigmadiag_\indset^{-2}) \bar{O}\transpose\\
    \Longrightarrow~~~~
    &(\bar{O}\transpose (I_l - \Lambda \sigmadiag_\indset^{-2}))_{ij} = ((I_l - \Lambda \sigmadiag_\indset^{-2}) \bar{O}\transpose)_{ij}~~~~\forall i, j \in \{1, \cdots, k\}\\
    \Longrightarrow~~~~
    &(o_j)_i (1 - \lambda_j \sigma_{i_j}^{-2}) = (o_j)_i (1 - \lambda_i \sigma_{i_i}^{-2})\\
    \Longrightarrow~~~~
    & (o_j)_i (\lambda_i \sigma_{i_i}^{-2} - \lambda_j \sigma_{i_j}^{-2}) = 0
\end{align*}

Since $(I_l - \Lambda \sigmadiag_\indset^{-2})$ is a diagonal matrix with strictly descending entries, we have $\lambda_i \sigma_{i_i}^{-2} - \lambda_j \sigma_{i_j}^{-2} \neq 0$ for $i \neq j$. Hence $(o_j)_i = 0$ for $i \neq j$, i.e. $\bar{O}$ is diagonal. Since $\bar{O}$ has orthonormal columns, it has diagonal entries $\pm 1$.
\begin{align*}
    O = P_0^{-1} \bar{O} = P_0\transpose \bar{O}
\end{align*}

Therefore, $O$ has exactly one entry of $\pm 1$ in each row, and at most one entry of $\pm 1$ in each column, and zeros elsewhere.
\end{proof}

We now finish the proof for Theorem~\ref{th:nonuni_landscape}.
Applying Lemma~\ref{lemma:nonuni_landscape_diagonal_indset}, we can rewrite the stationary points using rank $k$ matrices $\sigmadiag$ and $U$:
\begin{align*}
    W^*_1 &= P (I - \Lambda\sigmadiag^{-2})^{\frac{1}{2}}U^T\\
    W^*_2 &= U(I - \Lambda\sigmadiag^{-2})^{\frac{1}{2}} P
\end{align*}
Where $P \in \real^{k \times k}$ has exactly one $\pm 1$ in each row and each column with index in $\indset$, and zeros elsewhere. 
This concludes the proof.

\end{proof}
\subsection{Proof of recovery of ordered, axis-aligned solution at global minima}
\label{app:proofs:global_minima}
\begin{lemma}[Global minima -- necessary condition 1]
\label{lemma:global_optimum_necessary}
Let the encoder ($\encoder^*$) and decoder ($\decoder^*$) of the non-uniform $\ell_2$ regularized LAE have the form in~\eqref{eq:nonuni_stationary_w1_strong} and~\eqref{eq:nonuni_stationary_w2_strong}.
If $0 < \lambda_i < \sigma_i\sq$ for $\forall i = 1, \cdots, k$, then $(\encoder^*, \decoder^*)$ can be at global minima only if $P$ has full rank.
\end{lemma}

\begin{proof}[Proof of Lemma~\ref{lemma:global_optimum_necessary}]
We prove the contrapositive: if $\mathrm{rank}(P) < k$, then $(\encoder^*, \decoder^*)$ in~\eqref{eq:nonuni_stationary_w1_strong} and~\eqref{eq:nonuni_stationary_w2_strong} is not at global minimum.

Since $\mathrm{rank}(P) < k$, there exists a matrix $\delta P \in \real^{k \times k}$ such that $\delta P$ has all but one element equal to 0, and $\delta P_{ij} = h > 0$, for some $i,j\in \{1, \dots, k\}$, where the $i^{th}$ row and $j^{th}$ column of $P$ are all zeros.
\begin{align*}
    \delta \encoder &= \delta P(I - \Lambda\sigmadiag^{-2})^{\frac{1}{2}}U^T\\
    \delta \decoder &= U(I - \Lambda\sigmadiag^{-2})^{\frac{1}{2}} \delta P\transpose
\end{align*}
\begin{align*}
&~~~~\loss_{\sigma'} (\encoder^* + \delta \encoder, \decoder^* + \delta \decoder)\\
\begin{split}
    &= \frac{1}{n} \fnorm{\x - (\decoder^* + \delta \decoder) (\encoder^* + \delta \encoder) \x}\sq\\
    &~~~~+ \fnorm{\Lambda^{1/2} (\encoder^* + \delta \encoder)}\sq + \fnorm{(\decoder^* + \delta \decoder) \Lambda^{1/2}}\sq
\end{split}\\
\begin{split}
    &= \frac{1}{n}\trace((\identity - (\decoder^* + \delta \decoder)(\encoder^* + \delta \encoder)) \x \x\transpose (\identity - (\decoder^* + \delta \decoder)(\encoder^* + \delta \encoder)))\\
    &~~~~+ \trace(\Lambda^{1/2} (\encoder^* + \delta \encoder) (\encoder^* + \delta \encoder)\transpose \Lambda^{1/2}) + \trace (\Lambda^{1/2} (\decoder^* + \delta \decoder)\transpose (\decoder^* + \delta \decoder) \Lambda^{1/2})
\end{split}\\
\begin{split}
    &= \trace(
        (\identity - (I - \Lambda\sigmadiag^{-2}) (P + \delta P)\transpose (P + \delta P))^2
        \sigmadiag^2)\\
    &~~~~+ 2\trace(\Lambda (P + \delta P) (I - \Lambda\sigmadiag^{-2}) (P + \delta P)\transpose)
\end{split}\\
\begin{split}
    &= \loss_{\sigma'} (\encoder^*, \decoder^*) + [(1 - (1 - \lambda_{i} \sigma_{i}^{-2}) h^2)^2 - 1] \sigma_{i}^2 + 2 \lambda_{i} (1 - \lambda_{i} \sigma_{i})^{-2}) h^2
\end{split}\\
&= \loss_{\sigma'} (\encoder^*, \decoder^*) - 2(\sigma_{i}^2 - \lambda_{i})(1 - \lambda_{i} \sigma_{i}^{-2}) h^2 + (1 - \lambda_{i} \sigma_{i}^{-2})^2 \sigma_{i}^2 h^4\\
&= \loss_{\sigma'} (\encoder^* - \delta \encoder, \decoder^* - \delta \decoder)
\end{align*}

The first derivative of $(\encoder^*, \decoder^*)$ along $(\delta \encoder, \delta \decoder)$ is zero:
\begin{align*}
    &~~~~\lim_{h \rightarrow 0} \frac{\loss_{\sigma'} (\encoder^* + \delta \encoder, \decoder^* + \delta \decoder) - \loss_{\sigma'} (\encoder^*, \decoder^*)}{h}\\
    &= \lim_{h \rightarrow 0} \frac{- 2(\sigma_{i}^2 - \lambda_{i})(1 - \lambda_{i} \sigma_{i}^{-2}) h^2 + (1 - \lambda_{i} \sigma_{i}^{-2})^2 \sigma_{i}^2 h^4}{h}\\
    &= 0
\end{align*}

The second derivative of $(\encoder^*, \decoder^*)$ along $(\delta \encoder, \delta \decoder)$ is negative (note that $0 < \lambda_{i} < \sigma_{i}\sq$):
\begin{align*}
    &~~~~\lim_{h \rightarrow 0} \frac{\loss_{\sigma'} (\encoder^* + \delta \encoder, \decoder^* + \delta \decoder) - 2\loss_{\sigma'} (\encoder^*, \decoder^*) + \loss_{\sigma'} (\encoder^* - \delta \encoder, \decoder^* - \delta \decoder)}{h^2}\\
    &= \lim_{h \rightarrow 0} \frac{2 \loss_{\sigma'} (\encoder^* + \delta \encoder, \decoder^* + \delta \decoder) - 2\loss_{\sigma'} (\encoder^*, \decoder^*)}{h^2}\\
    &= \lim_{h \rightarrow 0} 2\frac{- 2(\sigma_{i}^2 - \lambda_{i})(1 - \lambda_{i} \sigma_{i}^{-2}) h^2 + (1 - \lambda_{i} \sigma_{i}^{-2})^2 \sigma_{i}^2 h^4}{h^2}\\
    &= -4 (\sigma_{i}^2 - \lambda_{i})(1 - \lambda_{i} \sigma_{i}^{-2})\\
    &< 0
\end{align*}

Therefore, if $\mathrm{rank}(P) < k$, $(\encoder^*, \decoder^*)$ is not at global minima. The contrapositive states that if $(\encoder^*, \decoder^*)$ is at global minima, then $P$ has full rank.
\end{proof}

\begin{lemma}[Global minima -- necessary condition 2]
\label{lemma:global_optimum_necessary2}
Let the encoder ($\encoder^*$) and decoder ($\decoder^*$) of the non-uniform $\ell_2$ regularized LAE have the form in~\eqref{eq:nonuni_stationary_w1_strong} and~\eqref{eq:nonuni_stationary_w2_strong}, and $P$ has full rank. 
Then $(\encoder^*, \decoder^*)$ can be at global minimum only if $P$ is diagonal.
\end{lemma}

\begin{proof}[Proof of Lemma~\ref{lemma:global_optimum_necessary2}]
Following similar analysis for the proof of Theorem~\ref{th:nonuni_landscape}, we have~\eqref{eq:AcosB} and~\eqref{eq:AsinB}. In order for $\theta = 0$ to be a global optimum, it must be a local optimum. Therefore, for $\forall i < j$, we need either of the following necessary conditions to be true:
\begin{align*}
    \text{Condition 1: } &A = 0 
        \iff \weightcolumn_i\transpose \weightcolumn_j = 0 ~~\text{and}~~ \weightcolumn_i\transpose\weightcolumn_i = \weightcolumn_j\transpose \weightcolumn_j\\
    \text{Condition 2: } &A \cos B < 0 ~~\text{and}~~ B = \beta \pi,~~\beta \in \mathbb{Z} 
        \iff \weightcolumn_i\transpose \weightcolumn_j = 0 
        ~~\text{and}~~ \weightcolumn_i\transpose \weightcolumn_i > \weightcolumn_j\transpose \weightcolumn_j
\end{align*}

The two conditions can be consolidated to the following ($i < j$):
\begin{align*}
    \weightcolumn_i\transpose \weightcolumn_j = 0 
    ~~\text{and}~~ 
    \weightcolumn_i\transpose \weightcolumn_i \geq \weightcolumn_j\transpose \weightcolumn_j
\end{align*}

Then, $(\decoder^*)\transpose (\decoder^*)$ is a diagonal matrix with non-negative diagonal entries sorted in descending order.
\begin{align}
\label{eq:svd_w2Tw2_P}
    (\decoder^*)\transpose (\decoder^*) 
    &= P (I_l - \Lambda \sigmadiag_\indset^{-2}) P\transpose
\end{align}
Since the diagonal entries of $(I_l - \Lambda \sigmadiag_\indset^{-2})$ are positive and sorted in strict descending order, and that~\eqref{eq:svd_w2Tw2_P} is an SVD of $(\decoder^*)\transpose (\decoder^*)$, we have:
\begin{align*}
    (\decoder^*)\transpose (\decoder^*) = (I_l - \Lambda \sigmadiag_\indset^{-2})
\end{align*}

We can use the same technique as the proof of Lemma~\ref{lemma:nonuni_landscape_diagonal_indset} to prove that $P$ must be diagonal.
\end{proof}

\begin{lemma}[Global minima -- sufficient condition]
\label{lemma:global_minima_sufficient}
Let $\bar{I} \in \real^{k \times k}$ be a diagonal matrix with diagonal elements equal to $\pm 1$.The encoder ($\encoder^*$) and decoder ($\decoder^*$) of the following form are at global minima of the non-uniform $\ell_2$ LAE objective.
\begin{align}
\label{eq:sufficient_global_minima_w1}
    W^*_1 &= \bar{I} (I - \Lambda\sigmadiag^{-2})^{\frac{1}{2}}U^T\\
\label{eq:sufficient_global_minima_w2}
    W^*_2 &= U(I - \Lambda\sigmadiag^{-2})^{\frac{1}{2}} \bar{I}
\end{align}
\end{lemma}

\begin{proof}[Proof of Lemma~\ref{lemma:global_minima_sufficient}]
Because the objective of the non-uniform regularized LAE is differentiable everywhere for $\encoder$ and $\decoder$, all local minima (therefore also global minima) must occur at stationary points. Theorem~\ref{th:nonuni_landscape} shows that the stationary points must be of the form~\eqref{eq:nonuni_stationary_w1_strong} and~\eqref{eq:nonuni_stationary_w2_strong}. Lemma~\ref{lemma:global_optimum_necessary} further shows that a necessary condition for the global minima is when $l=k$, i.e. the encoder and decoder must be of the form in~\eqref{eq:sufficient_global_minima_w1} and~\eqref{eq:sufficient_global_minima_w2}.

In order to prove that~\eqref{eq:sufficient_global_minima_w1} and~\eqref{eq:sufficient_global_minima_w2} are sufficient condition for global minima, it is sufficient to show that all $\encoder^*$, $\decoder^*$ that satisfy~\eqref{eq:sufficient_global_minima_w1} and~\eqref{eq:sufficient_global_minima_w2} (i.e. all $\bar{I}$) result in the same loss. Notice that $\bar{I}^2 = \identity$, then:

\begin{align}
\nonumber
    \loss_{\sigma'} (\encoder^*, \decoder^*) &= \frac{1}{n}\fnorm{\x - \decoder^* \encoder^* \x}\sq + \fnorm{\Lambda^{1/2} \encoder^*}\sq + \fnorm{\decoder^* \Lambda^{1/2}}\sq\\
\nonumber
    \begin{split}
    &= \frac{1}{n} \fnorm{\x - \decoder^* \encoder^* \x}\sq + \trace (\Lambda^{1/2} \encoder^* (\encoder^*)\transpose \Lambda^{1/2}) \\
    &~~~~ + \trace (\Lambda^{1/2} (\decoder^*)\transpose \decoder^* \Lambda^{1/2})
    \end{split}\\
\nonumber
    \begin{split}
        &= \frac{1}{n} \fnorm{\x - U(I - \Lambda\sigmadiag^{-2})^{\frac{1}{2}} \bar{I}^2 (I - \Lambda\sigmadiag^{-2})^{\frac{1}{2}}U^T \x}\sq\\
        &~~~~ + 2 \trace (\Lambda^{1/2} \bar{I}  (I - \Lambda\sigmadiag^{-2})^{\frac{1}{2}}U^T U(I - \Lambda\sigmadiag^{-2})^{\frac{1}{2}} \bar{I}\transpose \Lambda^{1/2})
    \end{split}\\
\label{eq:nonuniloss_const_wrt_Ibar}
    \begin{split}
        &= \frac{1}{n} \fnorm{\x - U(I - \Lambda\sigmadiag^{-2})U^T \x}\sq + 2\trace(\Lambda (I - \Lambda\sigmadiag^{-2}))
    \end{split}
\end{align}

According to~\eqref{eq:nonuniloss_const_wrt_Ibar}, $\loss_{\sigma'} (\encoder^*, \decoder^*)$ is constant with respect to $\bar{I}$. Hence,~\eqref{eq:sufficient_global_minima_w1} and~\eqref{eq:sufficient_global_minima_w2} are sufficient conditions for global minima of the non-uniform $\ell_2$ regularized LAE objective.
\end{proof}

\begin{proof}[\textbf{Proof of Theorem~\ref{th:nonuni_global_optima}}]
From Lemma~\ref{lemma:global_optimum_necessary},~\ref{lemma:global_optimum_necessary2}, and~\ref{lemma:global_minima_sufficient}, we conclude that the global minima of the non-uniform $\ell_2$ regularized LAE are achieved if and only if the encoder ($\encoder^*$) and decoder ($\decoder^*$) are of the form in~\eqref{eq:sufficient_global_minima_w1} and~\eqref{eq:sufficient_global_minima_w2}, i.e. ordered, axis-aligned individual principal component directions.

We have proven in Lemma~\ref{lemma:global_optimum_necessary} that for $l < k$, there exists a direction for which the second derivative of the objective is negative. 
We have proven also that stationary points with $l = k$ are either global optima, or saddle points (Lemma~\ref{lemma:global_optimum_necessary2},~\ref{lemma:global_minima_sufficient}). Hence, there do not exist local minima that are not global minima. 
\end{proof}
\subsection{Proof of local linear convergence of RAG}
\label{app:proofs:rotation_linear_convergence}
\begin{proof}[Proof of Theorem~\ref{th:local_linear_convergence}]
Applying Assumption~\ref{assumption:subspace_converged}, the instantaneous update for RAG is,
\begin{align*}
    &\dot{W}_1 = \frac{1}{n}A \encoder\\
    &\dot{W}_2 = \frac{1}{n} \decoder A
\end{align*}
The instantaneous update for $\y\y\transpose$ is,
\begin{align*}
    \frac{d}{dt}(\y\y\transpose) = \frac{1}{n} (A \y\y\transpose + \y\y\transpose A\transpose)
\end{align*}

Let $y_{ij}$ be the $i, j^{th}$ element of $\y\y\transpose$, and $i < j$, then,

\begin{align}
\nonumber
    &\frac{d}{dt} y_{ii} = \frac{2}{n} (- \sum_{l=1}^{i-1} y_{il}\sq + \sum_{l=i+1}^{k} y_{il}\sq)\\
\label{eq:d_yyt_ij}
    &\frac{d}{dt} y_{ij} = - \frac{1}{n} (y_{ii} - y_{jj}) y_{ij}
        + \frac{2}{n} (- \sum_{l=1}^{i-1} y_{il}y_{jl} + \sum_{l=j+1}^{k} y_{il}y_{jl})
\end{align}
With Assumption~\ref{assumption:yyt_near_diagonal}, we can write~\eqref{eq:d_yyt_ij} as:
\begin{align}
    \label{eq:d_yyt_ij_scale}
    \frac{d}{dt}y_{ij} =
    - \frac{1}{n} (y_{ii} - y_{jj}) y_{ij} + \mathcal{O}(\frac{\epsilon\sq}{k})
\end{align}

The first term in~\eqref{eq:d_yyt_ij_scale} collects the products of diagonal and off-diagonal elements, and is of order $\mathcal{O}(\frac{\epsilon}{k})$. The second term in~\eqref{eq:d_yyt_ij_scale} collects second-order off-diagonal terms. With $0 < \epsilon \ll 1$, we can drop the second term. 

Also, applying Assumption~\ref{assumption:ordered_diagonal}, we have $y_{ii} > y_{jj}$.
\begin{align*}
    \frac{d}{dt} \vert y_{ij} \vert \approx - \frac{1}{n} (y_{ii} - y_{ij}) \vert y_{ij} \vert
\end{align*}

The instantaneous change of the ``non-diagonality'' $Nd(\frac{1}{n} \y\y\transpose)$ is,
\begin{align*}
    \frac{d}{dt} Nd(\frac{1}{n}\y\y\transpose)
    &= \frac{d}{dt} \bigg( 2 \sum_{i=1}^{k-1} \sum_{j=i+1}^k \frac{1}{n} \vert y_{ij} \vert \bigg)
    =  2 \sum_{i=1}^{k-1} \sum_{j=i+1}^k \frac{1}{n} \bigg(\frac{d}{dt}\vert y_{ij} \vert \bigg)\\
    &\approx 2 \sum_{i=1}^{k-1} \sum_{j=i+1}^k \frac{1}{n} \bigg( - \frac{1}{n}(y_{ii} - y_{jj}) \vert y_{ij} \vert \bigg)\\
    &\leq - g \cdot \bigg( 2 \sum_{i=1}^{k-1} \sum_{j=i+1}^k \frac{1}{n} \vert y_{ij} \vert \bigg)\\
    &= - g \cdot Nd(\frac{1}{n}\y\y\transpose)
\end{align*}

Hence, $Nd(\frac{1}{n}\y\y\transpose)$ converges to 0 with an instantaneous linear rate of $g$.
\end{proof}
\subsection{Convergence of latent space rotation to axis-aligned solutions}
\label{app:proofs:rotation_global_convergence}
We first state LaSalle's invariance principle~\citep{khalil2002nonlinear} in Lemma~\ref{lemma:lesalle}, which is used in Theorem~\ref{th:rotation_asymp_stability} to prove the convergence of latent space rotation to the set of axis-aligned solutions.

\begin{lemma}[LaSalle's invariance principle (local version)]
\label{lemma:lesalle}
Given dynamical system $\dot{x} = f(x)$ where $x$ is a vector of variables, and $f(x^*) = 0$. If a continuous and differentiable real-valued function $V(x)$ satisfies,
\begin{align*}
    \dot{V}(x) \leq 0~\mathrm{for}~\forall~x
\end{align*}
Then $\dot{V}(x) \rightarrow 0$ as $t \rightarrow \infty$.

Moreover, if there exists a neighbourhood $N$ of $x^*$ such that for $x \in N$,
\begin{align*}
    V(x) > 0 ~\mathrm{if}~ x \neq x^*
\end{align*}
And,
\begin{align*}
    \dot{V}(x) = 0 ~\forall~t \geq 0 \implies x(t) = x^* ~\forall~t\geq 0
\end{align*}
Then $x^*$ is locally asymptotically stable.
\end{lemma}

In Section~\ref{sec:rotation-global-convergence}, we gave an informal statement of Theorem~\ref{th:rotation_asymp_stability}. Here, we state the theorem formally.

\begin{customthm}{4}[Global convergence to axis-aligned solutions]
\label{th:rotation_asymp_stability}
Let $O_0 \in \real^{k \times k}$ be an orthogonal matrix, $W \in \real^{k \times m}$ ($k < m$). $\x$ and $U$ are as defined in Section~\ref{sec:prelim}.
$\ut(\cdot)$ and $\lt(\cdot)$ are as defined in Algorithm~\ref{algo:rotation_gradient}. 
Consider the following dynamical system,
\begin{align}
\label{eq:rotation_only_matrix_update}
    &\dot{W} = \frac{1}{2n} (\ut (W \x \x W\transpose) - \lt (W \x \x W\transpose)) W\\
    \label{eq:rotation_initial_condition}
    &W(0) = O_0 U\transpose    
\end{align}
Then $W(t) \rightarrow PU\transpose$ as $t \rightarrow \infty$, where $P \in \real^{k\times k}$ is a permutation matrix with non-zero elements $\pm 1$.
Also, the dynamical system is asymptotically stable at $\tilde{I}U\transpose$, where $\tilde{I}$ is a diagonal matrix with diagonal entries $\pm 1$.
\end{customthm}

It is straightforward to show that~\eqref{eq:rotation_only_matrix_update} and~\eqref{eq:rotation_initial_condition} are equivalent to the instantaneous limit of RAG on the orthogonal subspace $\encoder=\decoder\transpose=OU\transpose$ ($O$ is an orthogonal matrix). To see this, notice that on the orthogonal subspace, the gradient of $\encoder$ and $\decoder$ with respect to the reconstruction loss are zero,
\begin{align*}
    &\nabla_{\encoder} \mathcal{L}(\encoder=OU\transpose,\decoder=UO\transpose;\x) = 0\\
    &\nabla_{\decoder} \mathcal{L}(\encoder=OU\transpose,\decoder=UO\transpose;\x) = 0
\end{align*}

Theorem~\ref{th:rotation_asymp_stability} states that in the instantaneous limit, an LAE that is initialized on the orthogonal subspace and is updated by Algorithm~\ref{algo:rotation_gradient} globally converges to the set of axis-aligned solutions. Moreover, the convergence to the set of \emph{ordered} axis-aligned solutions is asymptotically stable. We provide the proof below.

\begin{proof}

We first show that $W(t)$ remains on the orthogonal subspace, i.e. $W(t) = O(t) U\transpose$ for $\forall~t$, where $O(t)$ is orthogonal.
To reduce the notation clutter, we define $A(W) = \frac{1}{2n} (\mathrm{\ut} (W \x \x W\transpose) - \mathrm{\ut} (W \x \x W\transpose))$. We take the time derivative of $W W\transpose$,
\begin{align*}
    \frac{d (WW\transpose)}{dt} &= \dot{W}W\transpose + W\dot{W}\transpose = A(W) W W\transpose + W W\transpose A(W)\transpose = A(W) W W\transpose - W W\transpose A(W)
\end{align*}

The last inequality follows from the observation that $A(W)$ is skew-symmetric, so that $A(W)\transpose = - A(W)$. Since $W(0)W(0)\transpose = I$, and $W W\transpose = I \implies \frac{d (W W\transpose)}{dt} = 0$, we have,
\begin{align*}
    W(t)W(t)\transpose = I ~\mathrm{for}~ \forall~t\geq 0
\end{align*}
From the dynamical equation~\eqref{eq:rotation_only_matrix_update}, we know that $W(t)$ has the form $W(t) = G(t)U\transpose$ for some matrix $G(t) \in \real^{k\times k}$. We have, 
\begin{align*}
    W(t)W(t)\transpose = G(t)U\transpose U G(t)\transpose = G(t)G(t)\transpose = I
    \implies G(t) \textrm{ is orthogonal.}
\end{align*}

We move on to use LaSalle's invariance principle to prove Theorem~\ref{th:rotation_asymp_stability}. The rest of the proof is divided into two parts. In the first part, we prove that $W(t) \rightarrow PU\transpose$ as $t \rightarrow \infty$, i.e. $W(t)$ globally converges to axis-aligned solutions.
In the second part, we prove that the ordered, axis-aligned solution $\tilde{I}U\transpose$ is locally asymptotically stable.

Let $\Sigma = \frac{1}{n}\x\x\transpose$. We define $V(W)$ as,
\begin{align}
\label{eq:v(w)}
    V(W) = \trace ((\sigmadiag\sq - W \Sigma W\transpose) D)
\end{align}
Where $\sigmadiag$ is as defined in Section~\ref{sec:prelim}, and $D = \mathrm{diag} (d_1, \dots, d_k)$, with $d_1 > \cdots > d_k > 0$. 

Note that definition~\eqref{eq:v(w)} is the Brockett cost function~\citep{absil2009optimization} with an offset. The Brockett cost function achieves minimum when the rows of $W$ are the eigenvectors of $\Sigma$. See Appendix~\ref{app:brockett} for a detailed discussion of the connection between the rotation augmented gradient and the Brockett cost function.

\paragraph{Part 1 (global convergence to axis-aligned solutions)} In this part, we compute $\dot{V}(W)$, and invoke the first part of LaSalle's invariance principle to show global convergence to axis-aligned solutions.

Denote the (transposed) $i^{th}$ row of $W$ as $w_i \in \real^{m \times 1}$. We rewrite~\eqref{eq:rotation_only_matrix_update} in terms of rows of $W$,
\begin{align*}
    \dot{w}_i = - \frac{1}{2} \sum_{j=1}^{i-1} (w_i\transpose \Sigma w_j) w_j + \frac{1}{2} \sum_{j=i+1}^{k} (w_i\transpose \Sigma w_j) w_j
\end{align*}

We proceed to compute $\dot{V}(W)$,
\begin{align*}
    \dot{V}(W) &= - 2 \sum_{i=1}^k d_i w_i\transpose \Sigma \dot{w}_i
    = \sum_{i=1}^k d_i \bigg[\sum_{j=1}^{i-1} (w_i\transpose \Sigma w_j)\sq - \sum_{j=i+1}^{k} (w_i\transpose \Sigma w_j)\sq\bigg]\\
    &= \sum_{i=2}^k \sum_{j=1}^{i-1} d_i (w_i\transpose \Sigma w_j)\sq - \sum_{i=1}^{k-1} \sum_{j=i+1}^k d_i (w_i\transpose \Sigma w_j)\sq\\
    &= \sum_{i=2}^k \sum_{j=1}^{i-1} d_i (w_i\transpose \Sigma w_j)\sq - \sum_{j=1}^{k-1} \sum_{i=j+1}^k d_j (w_i\transpose \Sigma w_j)\sq\\
    &= \sum_{i=2}^k \sum_{j=1}^{i-1} d_i (w_i\transpose \Sigma w_j)\sq - \sum_{i=2}^{k} \sum_{j=1}^{i-1} d_j (w_i\transpose \Sigma w_j)\sq\\
    &= \sum_{i=2}^k \sum_{j=1}^{i-1} (d_i - d_j) (w_i\transpose \Sigma w_j)\sq 
\end{align*}

Since $d_i < d_j$ for $\forall~i > j$, we have, 
\begin{align}
\label{eq:v(w)<=0}
 \dot{V}(W) \leq 0
\end{align}

The equality in~\eqref{eq:v(w)<=0} holds if and only if $\forall~i \neq j$, $w_i\transpose \Sigma w_j = 0$, or, written in matrix form, $W\x\x\transpose W \transpose$ is diagonal.
\begin{align}
\label{eq:v_dot=0_implication}
    \dot{V}(W) = 0 \iff W\x\x\transpose W\transpose ~\mathrm{is ~diagonal}
\end{align}

Since we also have $W = OU\transpose$, and using the SVD of $\x$, we can see that~\eqref{eq:v_dot=0_implication} is equivalent to,
\begin{align*}
    W = PU\transpose
\end{align*}
Also, $W = PU\transpose$ are stationary points of the dynamical equation~\eqref{eq:rotation_only_matrix_update}. By LaSalle's invariance principle, we have,
\begin{align*}
    \dot{V}(W) \rightarrow 0 ~\mathrm{as}~ t \rightarrow \infty
    \implies W(t) \rightarrow PU\transpose ~\mathrm{as}~ t \rightarrow \infty
\end{align*}
$W(t)$ globally converges to the set of axis-aligned solutions. This concludes the first part of the proof. 

\paragraph{Part 2 (asymptotic convergence to optimal representation)} We break down this part of the proof into two steps. First, we show that $V(W)$ is positive definite locally at $\tilde{I}U\transpose$. Then, we show that $\tilde{I}U\transpose$ is the only solution to $\dot{V}(W) = 0$ in its neighbourhood.

We first show that $V(W)$ is positive definite at $W = \tilde{I}U\transpose$. Note that columns of $U$ contain the ordered left singular vectors of $\x$. We can rewrite~\eqref{eq:v(w)} as,
\begin{align}
\label{eq:v(w)_as_sum}
    V(W) = - \trace (O \sigmadiag\sq O\transpose D) + \sum_{i=1}^k d_i \sigma_i\sq = - \sum_{i=1}^k \sum_{j=1}^k d_i \sigma_j\sq O_{ij}\sq + \sum_{i=1}^k d_i \sigma_i\sq
\end{align}
We use $O_{ij}$ to denote the component with row and column index $i$, $j$ respectively. \eqref{eq:v(w)_as_sum} is minimized when $O = \tilde{I}$ and takes value zero. It is positive everywhere else, and thus, $V(W)$ is positive definite at $W = \tilde{I} U\transpose$.

Now, we show that $W = \tilde{I}U\transpose$ is the only solution to $\dot{V}(W) = 0$ within some neighbourhood around itself. Since permutation matrices $P$ are finite and distinct, we can find a neighbourhood around each $\tilde{I}$ on the Stiefel manifold $V_k(\real^k)$, in which $W = \tilde{I}U\transpose$ is the unique solution for $\dot{V}(W)$. We mathematically state this below,
\begin{align*}
    &\exists\textrm{ some neighbourhood $N$ on $V_k(\real^k)$ around $\tilde{I}$, such that}\\
    &\big[O \in N, ~\dot{V}(OU\transpose) = 0 ~\forall~t \geq 0 \big] \implies O = \tilde{I}
\end{align*}

This means that local to $W = \tilde{I} U\transpose$, $\dot{V}(W) = 0$ for $\forall~t \geq 0$ implies $W = \tilde{I}U\transpose$.

We have satisfied all the necessary conditions to invoke LaSalle's invariance principle. Thus, $W = \tilde{I}U\transpose$ is locally asymptotically stable.
\end{proof}

\section{Connection of non-uniform $\ell_2$ regularization to linear VAE with diagonal covariance}
Consider the following VAE model,
\begin{align*}
    p(x | z) &= \mathcal{N} (W z + \mu, \sigma^2 \identity)\\
    q(z | x) &= \mathcal{N} (V (x - \mu), D)
\end{align*}

Where $W$ is the decoder, $V$ is the encoder, and $D$ is the diagonal covariance matrix. The ELBO objective is,
\begin{align*}
    \mathrm{ELBO} = - \mathit{KL}(q(z | x) || p(z)) + \mathbb{E}_{q(z | x)}[\log p(x | z)]
\end{align*}

\label{appendix:connection_to_vae}
It's shown in~\cite{dont-blame-elbo} that such a linear VAE with diagonal latent covariance can learn axis-aligned principal component directions. We show in this section that training such a linear VAE with ELBO is closely related to training a non-uniform $\ell_2$ regularized LAE.

As derived in Appendix C.2 of~\cite{dont-blame-elbo}, the gradients of the ELBO with respect to $D, V$ and $W$, are,
\begin{align*}
    \nabla D &= \frac{n}{2} (D^{-1} - \identity - \frac{1}{\sigma^2}\mathrm{diag}(W\transpose W))\\
    \nabla V &= \frac{n}{\sigma^2} (W\transpose  - (W\transpose W + \sigma^2 \identity)V)\Sigma\\
    \nabla W &= \frac{n}{\sigma^2} (\Sigma V\transpose - DW  - WV\Sigma V\transpose)
\end{align*}
Where $\Sigma = \frac{1}{n} \x\x\transpose$.
The optimal $D^* = \sigma^2 (\mathrm{diag}(W\transpose W) + \sigma^2 \identity)^{-1}$. 
The ``balanced" weights in this case is $V = M^{-1} W\transpose$, $M = W\transpose W + \sigma^2 \identity$

Assume optimal $D = D^*$ and balanced weights, we can rewrite the gradients. First, look at the gradient for $V$,
\begin{align*}
    \nabla V 
    &=  \frac{n}{\sigma^2} (W\transpose  - (W\transpose W + \sigma^2 \identity)V)\Sigma\\
    &=  \frac{n}{\sigma^2} ((W\transpose W + \sigma^2 \identity) V  - (W\transpose W + \sigma^2 \identity)V)\Sigma\\
    &= 0
\end{align*}
The gradient for $V$ simply forces $V$ to be ``balanced" with $W$.
Then for $W$,
\begin{align*}
    \nabla W 
    &=  \frac{n}{\sigma^2} (\Sigma V\transpose - DW  - WV\Sigma V\transpose)\\
    &=  \frac{n}{\sigma^2} (
        \Sigma V\transpose 
        - \sigma^2 (\mathrm{diag}(W\transpose W) + \sigma^2 \identity)^{-1}W
        - WV\Sigma V\transpose)\\
    &= \frac{1}{\sigma^2} (
        \x\x\transpose V\transpose 
        - n \sigma^2 \mathrm{diag}(M)^{-1} W
        - WV\x\x\transpose V\transpose)\\
    &= \frac{1}{\sigma^2} (
        \x \y\transpose
        - n \sigma^2 \mathrm{diag}(M)^{-1} W
        - W\y\y\transpose)\\
    &= \frac{1}{\sigma^2} (\x - W\y)\y\transpose
        - n \cdot \mathrm{diag}(M)^{-1} W
\end{align*}

This is exactly non-uniform $\ell_2$ regularization on $W$. The $\ell_2$ weights are dependent on $W$.
\begin{align*}
     \mathrm{diag}(M)^{-1} &= \mathrm{diag} (W\transpose W + \sigma^2 \identity)^{-1}
\end{align*}
\section{Connection between the rotation augmented gradient and the Brockett cost function}
\label{app:brockett}

In this section, we discuss the connection between our rotation augmented gradient and the gradient of the Brockett cost function. In particular, we show that the two updates share similar forms.

Since the Brockett cost function is defined on the Stiefel manifold, we assume throughout this section that $ \encoder = \decoder\transpose$, and $\decoder\transpose \decoder = \identity$. Let $\Sigma = \frac{1}{n} \x\x\transpose$ be the data covariance, the Brockett cost function is,
\begin{align*}
    \trace (\decoder\transpose \Sigma \decoder N)~~~\text{subj. to}~~~~ \decoder\transpose \decoder = \identity_k ~~(\text{i.e.}~ \decoder \in \mathrm{St}(k, m))
\end{align*}

Where $ N = \mathrm{diag}(\mu_1,\dots, \mu_k) $, and $ 0 < \mu_1 < \cdots < \mu_k $ are constant coefficients. To make the gradient form more consistent with the rotation augmented gradient, we switch the sign of the loss, and reverse the ordering of the diagonal matrix $ N $. This does not change the optimization problem, due to the constraint that $ \decoder $ is on the Stiefel manifold. We define,
\[
\mathcal{L}_B (\decoder) = - \trace (\decoder \transpose \Sigma \decoder D)~~~\text{subj. to}~~~~ \decoder\transpose \decoder = \identity_k
\]
Where $ D = \mathrm{diag}(d_1, \dots, d_k) $, $ d_1 > \cdots > d_k > 0 $.
Let $ \mathrm{skew}(M) = \frac{1}{2} (M - M\transpose) $, the gradient of the cost function on the Stiefel manifold is,

\[
\nabla_{\decoder} \mathcal{L}_B = - 2(\identity - \decoder \decoder \transpose) \Sigma \decoder D - \decoder \mathrm{skew}(2 \decoder\transpose \Sigma \decoder D)
\]

The gradient descent update in the continuous time limit is,
\begin{equation}\label{eq:brocket_update}
\dot{\decoder} = 2(\identity - \decoder\decoder\transpose) \Sigma \decoder D + 2 \decoder \mathrm{skew}(\decoder\transpose \Sigma \decoder D)
\end{equation}

\paragraph{Rotation augmented gradient}
With $ \encoder\transpose = \decoder $, the rotation augmented gradient update is,

\begin{equation}\label{eq:rotation_update}
\dot{\decoder} = 2(\identity - \decoder\decoder\transpose) \Sigma \decoder - 2 \decoder \mathrm{skew}( \ut(\decoder\transpose\Sigma \decoder))
\end{equation}

The updates~\eqref{eq:brocket_update} and~\eqref{eq:rotation_update} appear to have similar forms. We can make the connection more obvious with further manipulation. We express the second term in~\eqref{eq:brocket_update} with the triangular masking operations $ \ut $ and $ \lt $,
\begin{align*}
\mathrm{skew}(\decoder\transpose \Sigma \decoder D)
&= \mathrm{skew}(\ut(\decoder\transpose \Sigma \decoder D) + \lt(\decoder\transpose \Sigma \decoder D))\\
&= \mathrm{skew}(\ut(\decoder\transpose \Sigma \decoder D) - \lt(\decoder\transpose \Sigma \decoder D)\transpose)\\
&= \mathrm{skew}(\ut(\decoder\transpose \Sigma \decoder) D - \big(\lt(\decoder\transpose \Sigma \decoder) D\big)\transpose)\\
&= \mathrm{skew}(\ut(\decoder\transpose \Sigma \decoder) D - D \lt(\decoder\transpose \Sigma \decoder)\transpose)\\
&= \mathrm{skew}(\ut(\decoder\transpose \Sigma \decoder) D - D \ut(\decoder\transpose \Sigma \decoder))
\end{align*}

Then, we write the masks explicitly with element-wise multiplications,
\begin{align*}
\mathrm{skew}(\decoder\transpose \Sigma \decoder D)
&= \mathrm{skew} (
	\bigg(
		\begin{bmatrix}
		1 & \cdots & 1\\
		  & \ddots & \vdots\\
		  && 1
		\end{bmatrix}
		\circ 
		\decoder\transpose \Sigma \decoder
	\bigg)
	\begin{bmatrix}
		d_1 \\
		& \ddots\\
		&& d_k
	\end{bmatrix}\\
	&~~- 
	\begin{bmatrix}
		d_1 \\
		& \ddots\\
		&& d_k
	\end{bmatrix}
	\bigg(
		\begin{bmatrix}
			1 & \cdots & 1\\
			& \ddots & \vdots\\
			&& 1
		\end{bmatrix}
		\circ 
		\decoder\transpose \Sigma \decoder
	\bigg)
)\\
&= \mathrm{skew}\bigg(
	\begin{bmatrix}
		0	&	d_2 - d_1	& d_3 - d_1	&	\cdots 	& d_k - d_1\\
			&	0			& d_3 - d_2	&	\cdots	& d_k - d_2\\
			&				&	\ddots	&			&	\vdots\\
			&				&			&	0		& d_k - d_{k-1}\\
			&				&			&			&	0
	\end{bmatrix}
	\circ
	\decoder\transpose \Sigma \decoder
\bigg)
\end{align*}

Finally, we compare the two updates below,
\paragraph{Brockett update}
\[
\dot{\decoder} = 2(\identity - \decoder\decoder\transpose) \Sigma \decoder D - 2 \decoder \mathrm{skew}\bigg(
	\begin{bmatrix}
	0	&	d_1 - d_2	& d_1 - d_3	&	\cdots 	& d_1 - d_k\\
	&	0			& d_2 - d_3	&	\cdots	& d_2 - d_k\\
	&				&	\ddots	&			&	\vdots\\
	&				&			&	0		& d_{k-1} - d_k\\
	&				&			&			&	0
	\end{bmatrix}
\circ
\decoder\transpose \Sigma \decoder
\bigg)
\]

\paragraph{Rotation augmented gradient update}
\[
\dot{\decoder} = 2(\identity - \decoder\decoder\transpose) \Sigma \decoder - 2 \decoder \mathrm{skew}\bigg(
\begin{bmatrix}
	0 &	1 & \cdots & 1\\
	& 0 & \ddots  & \vdots\\
	&	& \ddots  & 1\\
	&	&	&	0
\end{bmatrix}
 \circ (\decoder\transpose\Sigma \decoder)\bigg)
\]

Both algorithms account for the rotation using the off-diagonal part of $ \decoder\transpose \Sigma \decoder $. The rotation augmented gradient applies binary masking, whereas the Brockett update introduces additional coefficients ($d_1, \dots, d_k$) that ``weights'' the rotation.
\section{Experiment details}
\label{app:exp_details}
We provide the experiment details in this section. The code is provided at \url{https://github.com/XuchanBao/linear-ae}.

\subsection{Convergence to optimal representation}
\label{app:exp_mnist}
In this section, we give the details of experiments for convergence to the optimal representation on the MNIST dataset (Figure~\ref{fig:convergence} and~\ref{fig:matrix_visualization}). 

The dataset is the MNIST training set, consisting of 60,000 images of size $28\times 28$ ($m=784)$). The latent dimension is $k=20$. The data is pixel-wise centered around zero. Training is done in full-batch mode. 

The regularization parameters $\lambda_1, \dots, \lambda_k$ for the non-uniform $\ell_2$ regularization are chosen to be $\sqrt{\lambda_1}=0.1$, $\sqrt{\lambda_k}=0.9$, and $\sqrt{\lambda_2}, \dots,\sqrt{\lambda_{k-1}}$ equally spaced in between.

The prior probabilities for the nested dropout and the deterministic variant of nested dropout are both chosen to be: $p_B(b) = \rho^b (1-\rho)$ for $b < k$, and $p_B(k) = 1 - \sum_{b=1}^{k-1} p_B(b)$. We choose $\rho=0.9$ for our experiments. This is consistent with the geometric distribution recommended in~\citet{nested-dropout}, due to its memoryless property.

The network weights are initialized independently with $\mathcal{N}(0, 10^{-4})$. We experiment with two optimizers: Nesterov accelerated gradient descent with momentum 0.9, and Adam optimizer. The learning rate for each model and each optimizer is searched to be optimal. See Table~\ref{tab:lr_mnist} for the search details, and the optimal learning rates.

\subsection{Scalability to latent representation sizes}
The details of the experiments for scalability to latent representation sizes correspond to Figure~\ref{fig:mnist_vary_hidden_size}.

The synthetic dataset has 5000 randomly generated data points, each with dimension $m=1000$. The singular values of the data are equally spaced between $1$ and $100$. In order to test the scalability of different models to the latent representation sizes, we run experiments with 10 different latent dimension sizes: $k=2,5,10,20,50,100,200,300,400,500$.

The regularization parameters $\lambda_1,\dots,\lambda_k$ for the non-uniform $\ell_2$ regularization are chosen to be $\sqrt{\lambda_1}=0.1$, $\sqrt{\lambda_k}=10$, and $\sqrt{\lambda_2}, \dots,\sqrt{\lambda_{k-1}}$ equally spaced in between.

The prior probabilities for the nested dropout and the deterministic variant of nested dropout, the initialization scheme for the network weights, and the optimizers are chosen in the same way as in Section~\ref{app:exp_mnist}.

We perform a search to find the optimal learning rates for each model, each optimizer with different latent dimensions. See Table~\ref{tab:lr_synth} for the search details, and Table~\ref{tab:lr_synth_final} for the learning rates used in the experiments.

\begin{table}[ht]
    \centering
    \caption{Learning rate search values for experiments on MNIST (Figure~\ref{fig:convergence} and~\ref{fig:matrix_visualization}). The optimal learning rates are labelled in boldface. Note that the Adam optimizer does not apply to RAG.}
    \label{tab:lr_mnist}
    \begin{tabular}{ccc}
    \toprule
    Model & Nesterov learning rates & Adam learning rates\\\midrule
    Uniform $\ell_2$ 
    & $\mathbf{1\mathrm{\mathbf{e}}{-3}}$
    & $\mathbf{1\mathrm{\mathbf{e}}{-3}}$
    \\
    Non-uniform $\ell_2$ 
    & $1\mathrm{e}{-4}$, $\mathbf{3\mathrm{\mathbf{e}}{-4}}$, $1\mathrm{e}{-3}$, $3\mathrm{e}{-3}$
    & $1\mathrm{e}{-3}$, $3\mathrm{e}{-3}$, $\mathbf{1\mathrm{\mathbf{e}}{-2}}$, $3\mathrm{e}{-2}$\\
    Rotation  
    & $1\mathrm{e}{-3}$, $\mathbf{3\mathrm{\mathbf{e}}{-3}}$, $1\mathrm{e}{-2}$
    & --- \\
    Nested dropout (nd) 
    & $1\mathrm{e}{-2}$, $\mathbf{3\mathrm{\mathbf{e}}{-2}}$, $1\mathrm{e}{-1}$ 
    & $3\mathrm{e}{-3}$, $\mathbf{1\mathrm{\mathbf{e}}{-2}}$, $3\mathrm{e}{-2}$, $1\mathrm{e}{-1}$\\
    Deterministic nd
    & $1\mathrm{e}{-2}$, $\mathbf{3\mathrm{\mathbf{e}}{-2}}$, $1\mathrm{e}{-1}$
    & $3\mathrm{e}{-3}$, $\mathbf{1\mathrm{\mathbf{e}}{-2}}$, $3\mathrm{e}{-2}$, $1\mathrm{e}{-1}$\\
    Linear VAE 
    & $3\mathrm{e}{-4}$, $\mathbf{1\mathrm{\mathbf{e}}{-3}}$, $3\mathrm{e}{-3}$
    & $3\mathrm{e}{-4}$, $\mathbf{1\mathrm{\mathbf{e}}{-3}}$, $3\mathrm{e}{-3}$
    \\\bottomrule
    \end{tabular}
\end{table}

\begin{table}[ht]
\centering
\caption{Learning rate search values for experiments on the synthetic dataset (Figure~\ref{fig:mnist_vary_hidden_size}). The optimal learning rates are labelled in boldface. Note that Adam optimizer does not apply to RAG, even though the experiments are shown here.}
\label{tab:lr_synth}
\subtable[$k=20$]{
    \centering
        \begin{tabular}{ccc}
        \toprule
        Model & Nesterov learning rates & Adam learning rates\\\midrule
        Non-uniform $\ell_2$ 
        & $1\mathrm{e}{-4}$, $3\mathrm{e}{-4}$, $\mathbf{1\mathrm{\mathbf{e}}{-3}}$, $3\mathrm{e}{-3}$
        & $1\mathrm{e}{-3}$, $\mathbf{3\mathrm{\mathbf{e}}{-3}}$, $1\mathrm{e}{-2}$, $3\mathrm{e}{-2}$
        \\
        Rotation  
        & $3\mathrm{e}{-5}$, $\mathbf{1\mathrm{\mathbf{e}}{-4}}$, $3\mathrm{e}{-4}$, $1\mathrm{e}{-3}$
        & $1\mathrm{e}{-4}$, $\mathbf{3\mathrm{\mathbf{e}}{-4}}$, $1\mathrm{e}{-3}$
        \\
        Nested dropout (nd) 
        & $1\mathrm{e}{-4}$, $3\mathrm{e}{-4}$, $\mathbf{1\mathrm{\mathbf{e}}{-3}}$, $3\mathrm{e}{-3}$ 
        & $1\mathrm{e}{-3}$, $3\mathbf{e}{-3}$, $\mathbf{1\mathrm{\mathbf{e}}{-2}}$, $3\mathrm{e}{-2}$
        \\
        Deterministic nd
        & $1\mathrm{e}{-4}$, $3\mathrm{e}{-4}$, $\mathbf{1\mathrm{\mathbf{e}}{-3}}$, $3\mathrm{e}{-3}$ 
        & $1\mathrm{e}{-3}$, $\mathbf{3\mathrm{\mathbf{e}}{-3}}$, $1\mathrm{e}{-2}$, $3\mathrm{e}{-2}$
        \\
        Linear VAE 
        & $3\mathrm{e}{-5}$, $1\mathrm{e}{-4}$, $\mathbf{3\mathrm{\mathbf{e}}{-4}}$, $1\mathrm{e}{-3}$
        & $3\mathrm{e}{-4}$, $1\mathrm{e}{-3}$, $\mathbf{3\mathrm{\mathbf{e}}{-3}}$, $1\mathrm{e}{-2}$
        \\\bottomrule
        \end{tabular}
        \label{tab:lr_synth_k=20}
    }
    
    \subtable[$k=200$]{
    \centering
        \begin{tabular}{ccc}
        \toprule
        Model & Nesterov learning rates & Adam learning rates\\\midrule
        Non-uniform $\ell_2$ 
        & $1\mathrm{e}{-4}$, $3\mathrm{e}{-4}$, $\mathbf{1\mathrm{\mathbf{e}}{-3}}$, $3\mathrm{e}{-3}$
        & $1\mathrm{e}{-3}$, $\mathbf{3\mathrm{\mathbf{e}}{-3}}$, $1\mathrm{e}{-2}$, $3\mathrm{e}{-2}$
        \\
        Rotation  
        & $3\mathrm{e}{-5}$, $\mathbf{1\mathrm{\mathbf{e}}{-4}}$, $3\mathrm{e}{-4}$, $1\mathrm{e}{-3}$
        & $1\mathrm{e}{-4}$, $\mathbf{3\mathrm{\mathbf{e}}{-4}}$, $1\mathrm{e}{-3}$
        \\
        Nested dropout (nd) 
        & $1\mathrm{e}{-4}$, $3\mathrm{e}{-4}$, $\mathbf{1\mathrm{\mathbf{e}}{-3}}$, $3\mathrm{e}{-3}$ 
        & $3\mathrm{e}{-4}$, $1\mathbf{e}{-3}$, $\mathbf{3\mathrm{\mathbf{e}}{-3}}$, $1\mathrm{e}{-2}$
        \\
        Deterministic nd
        & $1\mathrm{e}{-4}$, $3\mathrm{e}{-4}$, $\mathbf{1\mathrm{\mathbf{e}}{-3}}$, $3\mathrm{e}{-3}$
        & $1\mathrm{e}{-3}$, $3\mathrm{e}{-3}$, $\mathbf{1\mathrm{\mathbf{e}}{-2}}$, $3\mathrm{e}{-2}$
        \\
        Linear VAE 
        & $3\mathrm{e}{-5}$, $1\mathrm{e}{-4}$, $\mathbf{3\mathrm{\mathbf{e}}{-4}}$, $1\mathrm{e}{-3}$
        & $3\mathrm{e}{-4}$, $\mathbf{1\mathrm{\mathbf{e}}{-3}}$, $3\mathrm{e}{-3}$, $1\mathrm{e}{-2}$
        \\\bottomrule
        \end{tabular}
        \label{tab:lr_synth_k=200}
    }
    
    \subtable[$k=500$]{
    \centering
        \begin{tabular}{cc}
        \toprule
        Model & Adam learning rates\\\midrule
        Deterministic nd
        & $3\mathrm{e}{-3}$, $\mathbf{1\mathrm{\mathbf{e}}{-2}}$, $3\mathrm{e}{-2}$, $1\mathrm{e}{-1}$
        \\\bottomrule
        \end{tabular}
        \label{tab:lr_synth_k=500}
    }
\end{table}

\begin{table}[ht]
    \centering
    \caption{Learning rate used for experiments on the synthetic dataset (Figure~\ref{fig:mnist_vary_hidden_size}). Note that Adam optimizer does not apply to RAG, even though the experiments are shown here.}
    \label{tab:lr_synth_final}
    \subtable[Nesterov accelerated gradient descent ($k\leq 50$)]{
    \begin{tabular}{cccccc}
        \toprule
        $k$ & 2 & 5 & 10 & 20 & 50
        \\\midrule
        Non-uniform $\ell_2$ &
        $1\mathrm{e}{-3}$ & $1\mathrm{e}{-3}$ & $1\mathrm{e}{-3}$ & $1\mathrm{e}{-3}$ & $1\mathrm{e}{-3}$
        \\\hline
        Rotation  
        & $1\mathrm{e}{-4}$ & $1\mathrm{e}{-4}$ & $1\mathrm{e}{-4}$ & $1\mathrm{e}{-4}$ & $1\mathrm{e}{-4}$
        \\\hline
        Nested dropout (nd) 
        & $1\mathrm{e}{-3}$ & $1\mathrm{e}{-3}$ & $1\mathrm{e}{-3}$ & $1\mathrm{e}{-3}$ & $1\mathrm{e}{-3}$
        \\\hline
        Deterministic nd
        & $1\mathrm{e}{-3}$ & $1\mathrm{e}{-3}$ & $1\mathrm{e}{-3}$ & $1\mathrm{e}{-3}$ & $1\mathrm{e}{-3}$
        \\\hline
        Linear VAE 
        & $3\mathrm{e}{-4}$ & $3\mathrm{e}{-4}$ & $3\mathrm{e}{-4}$ & $3\mathrm{e}{-4}$ & $3\mathrm{e}{-4}$
        \\\bottomrule
        \end{tabular}
    }
    
    \subtable[Nesterov accelerated gradient descent ($k\geq 100$)]{
    \begin{tabular}{cccccc}
        \toprule
        $k$ & 100 & 200 & 300 & 400 & 500
        \\\midrule
        Non-uniform $\ell_2$ &
        $1\mathrm{e}{-3}$ & $1\mathrm{e}{-3}$ & $1\mathrm{e}{-3}$ & $1\mathrm{e}{-3}$ & $1\mathrm{e}{-3}$
        \\\hline
        Rotation  
        & $1\mathrm{e}{-4}$ & $1\mathrm{e}{-4}$ & $1\mathrm{e}{-4}$ & $1\mathrm{e}{-4}$ & $1\mathrm{e}{-4}$
        \\\hline
        Nested dropout (nd) 
        & $1\mathrm{e}{-3}$ & $1\mathrm{e}{-3}$ & $1\mathrm{e}{-3}$ & $1\mathrm{e}{-3}$ & $1\mathrm{e}{-3}$
        \\\hline
        Deterministic nd
        & $1\mathrm{e}{-3}$ & $1\mathrm{e}{-3}$ & $1\mathrm{e}{-3}$ & $1\mathrm{e}{-3}$ & $1\mathrm{e}{-3}$
        \\\hline
        Linear VAE 
        & $3\mathrm{e}{-4}$ & $3\mathrm{e}{-4}$ & $3\mathrm{e}{-4}$ & $3\mathrm{e}{-4}$ & $3\mathrm{e}{-4}$
        \\\bottomrule
        \end{tabular}
    }
    
    \subtable[Adam optimizer ($k\leq 50$)]{
    \begin{tabular}{cccccc}
        \toprule
        $k$ & 2 & 5 & 10 & 20 & 50
        \\\midrule
        Non-uniform $\ell_2$ 
        & $3\mathrm{e}{-3}$ & $3\mathrm{e}{-3}$ & $3\mathrm{e}{-3}$ & $3\mathrm{e}{-3}$ & $3\mathrm{e}{-3}$
        \\\hline
        Rotation  
        & $3\mathrm{e}{-4}$ & $3\mathrm{e}{-4}$ & $3\mathrm{e}{-4}$ & $3\mathrm{e}{-4}$ & $3\mathrm{e}{-4}$
        \\\hline
        Nested dropout (nd) 
        & $1\mathrm{e}{-2}$ & $1\mathrm{e}{-2}$ & $1\mathrm{e}{-2}$ & $1\mathrm{e}{-2}$ & $3\mathrm{e}{-3}$
        \\\hline
        Deterministic nd
        & $3\mathrm{e}{-3}$ & $3\mathrm{e}{-3}$ & $3\mathrm{e}{-3}$ & $3\mathrm{e}{-3}$ & $3\mathrm{e}{-3}$
        \\\hline
        Linear VAE 
        & $3\mathrm{e}{-3}$ & $3\mathrm{e}{-3}$ & $3\mathrm{e}{-3}$ & $3\mathrm{e}{-3}$ & $1\mathrm{e}{-3}$
        \\\bottomrule
        \end{tabular}
    }
    
    \subtable[Adam optimizer ($k\geq 100$)]{
    \begin{tabular}{cccccc}
        \toprule
        $k$ & 100 & 200 & 300 & 400 & 500
        \\\midrule
        Non-uniform $\ell_2$ & $3\mathrm{e}{-3}$ & $3\mathrm{e}{-3}$ & $3\mathrm{e}{-3}$ & $3\mathrm{e}{-3}$ & $3\mathrm{e}{-3}$ 
        \\\hline
        Rotation  
        & $3\mathrm{e}{-4}$ & $3\mathrm{e}{-4}$ & $3\mathrm{e}{-4}$ & $3\mathrm{e}{-4}$ & $3\mathrm{e}{-4}$
        \\\hline
        Nested dropout (nd) 
        & $3\mathrm{e}{-3}$ & $3\mathrm{e}{-3}$ & $3\mathrm{e}{-3}$ & $3\mathrm{e}{-3}$ & $3\mathrm{e}{-3}$
        \\\hline
        Deterministic nd
        & $1\mathrm{e}{-2}$ & $1\mathrm{e}{-2}$ & $1\mathrm{e}{-2}$ & $1\mathrm{e}{-2}$ & $1\mathrm{e}{-2}$
        \\\hline
        Linear VAE 
        & $1\mathrm{e}{-3}$ & $1\mathrm{e}{-3}$ & $1\mathrm{e}{-3}$ & $1\mathrm{e}{-3}$ & $1\mathrm{e}{-3}$
        \\\bottomrule
        \end{tabular}
    }
\end{table}
\section{Additional experiments}

\subsection{Non-uniform $\ell_2$ regularization with optimal penalty weights (at global minima)}
\label{app:optimal_lambda}
In this section, we show the experimental results of the learning dynamics of the non-uniform $\ell_2$ regularization on MNIST, with ``optimally'' chosen $\ell_2$ penalty weights. Specifically, we set the latent dimension $k = 20$, and obtain the $\lambda_1, \dots, \lambda_k$ values by solving the $\min \max$ optimization problem in~\eqref{eq:nonuni_hessian_minmax}. These choices of the $\ell_2$ penalty weights are optimal at global minima, because the condition number of the Hessian of the objective at global minima is minimized.

\begin{figure}[ht]
    \centering
    \includegraphics[width=0.5\textwidth]{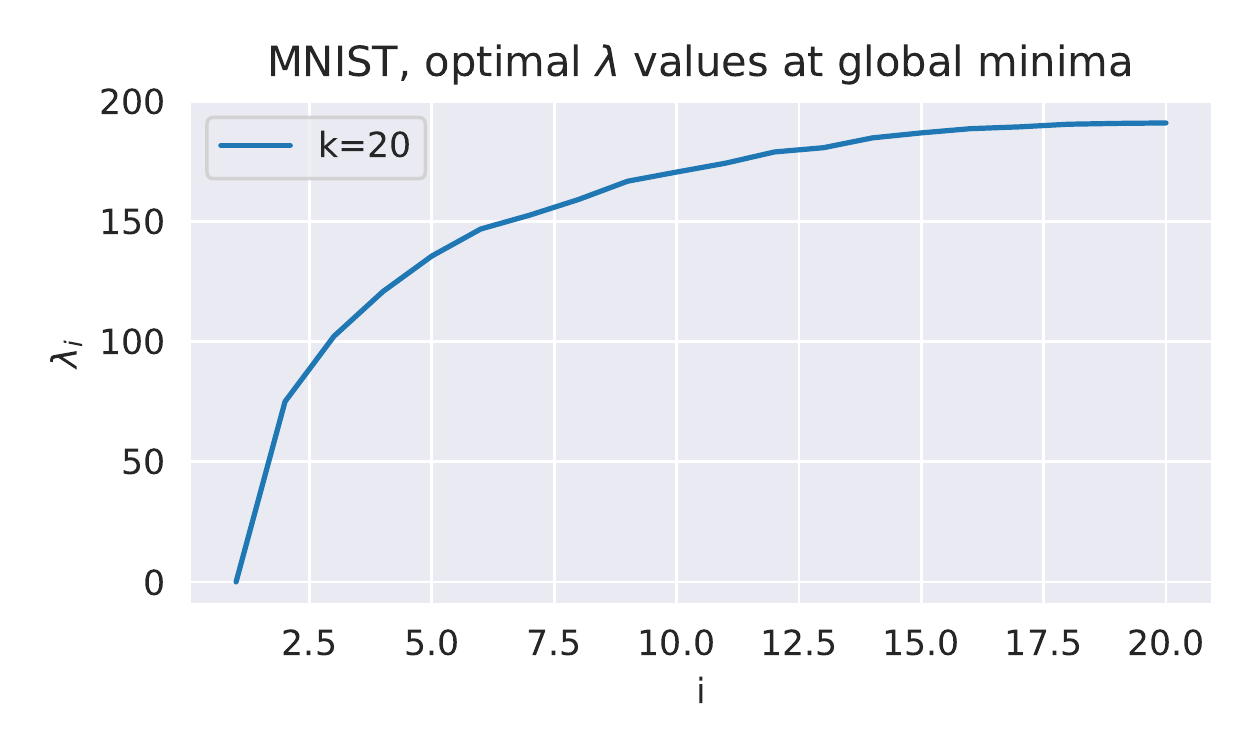}
    \caption{Optimal $\ell_2$ penalty weights on MNIST, with $k=20$.}
    \label{fig:optimal_lambda_mnist}
\end{figure}

In practice, the $\ell_2$ penalty weights in Figure~\ref{fig:optimal_lambda_mnist} are not accessible without knowing the $\sigma$ values of the dataset. However, we show in Figure~\ref{fig:optimal_lambda_convergence} that even with this knowledge, using the $\lambda$ values optimal at global optima significantly slows down the initial phase of training. This means that these $\lambda$ values are suboptimal away from global optima. In general, it is difficult to determine the $\lambda$ values that are optimal for the overall training process. This contributes to the weakness of symmetry breaking by the non-uniform $\ell_2$ regularization.

\begin{figure}[t]
    \subfigure[Axis-alignment]{\label{fig:optimal_lambda_axis_alignment}\includegraphics[width=0.5\textwidth]{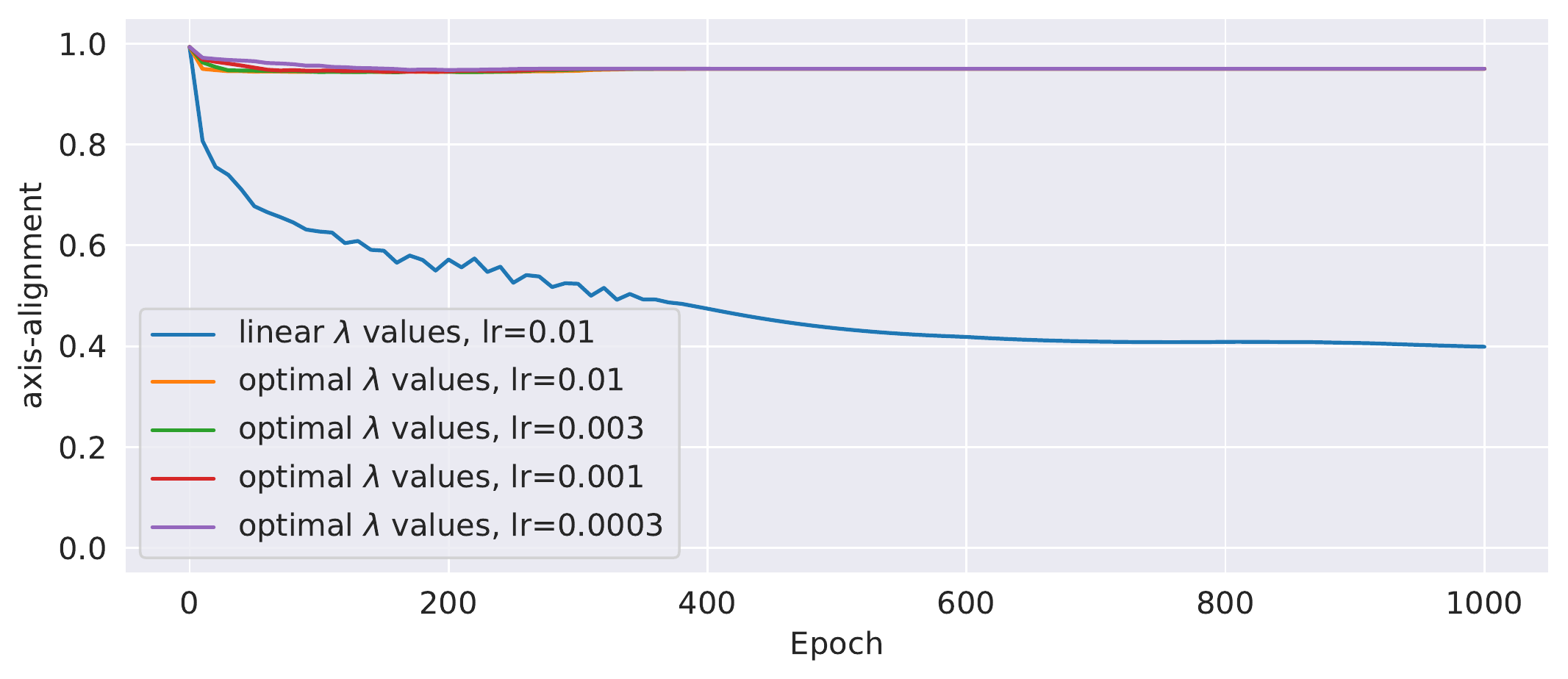}}
    \subfigure[Subspace convergence]{\label{fig:optimal_lambda_subspace_distance}\includegraphics[width=0.5\textwidth]{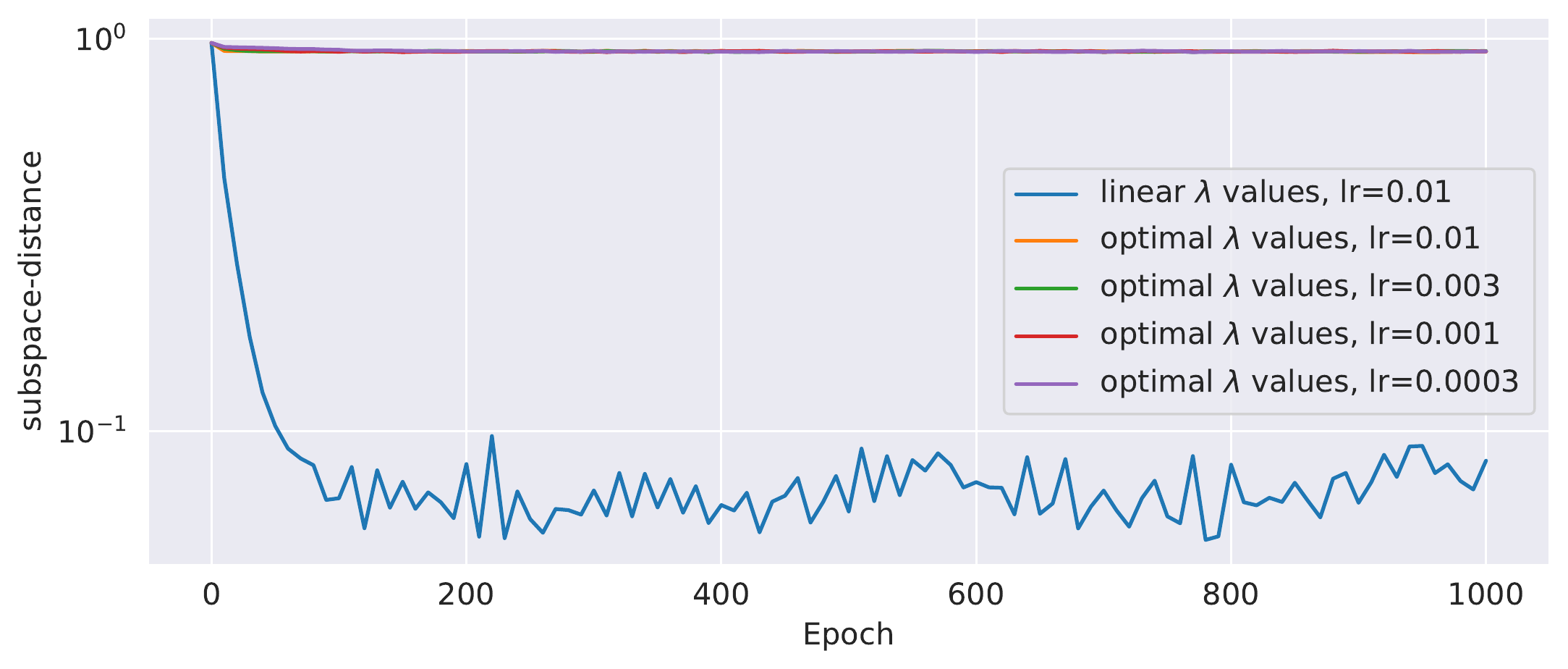}}
    \caption{Learning dynamics of non-uniform $\ell_2$ regularized LAEs on the MNIST ($ k=20 $), with different choices of penalty weight values. All models are trained with Adam optimizer for 1000 epochs. The optimal $\lambda$ values are as in Figure~\ref{fig:optimal_lambda_mnist}. Results with different learning rates are shown, provided that the learning rates are small enough to maintain training stability.}
    \label{fig:optimal_lambda_convergence}
\end{figure}

\subsection{Mini-batch training on MNIST}
\label{app:minibatch_exp}

In this section, we show the learning dynamics of the models in Section~\ref{sec:experiments} trained on MNIST using mini-batches. The uniform $\ell_2$ regularized LAE is not included, as it doesn't recover the axis-aligned solutions. Figure~\ref{fig:minibatch_1000} and~\ref{fig:minibatch_100} show the learning dynamics with $k=20$ and mini-batch size 1000 and 100, respectively. We observe similar results as in the full-batch setting (Figure~\ref{fig:convergence}), with additional stochasticity introduced by mini-batch training.

\begin{figure}[ht]
    \centering
    \subfigure[Axis-alignment]{\label{fig:minibatch_1000_axis_alignment}\includegraphics[width=0.49\textwidth]{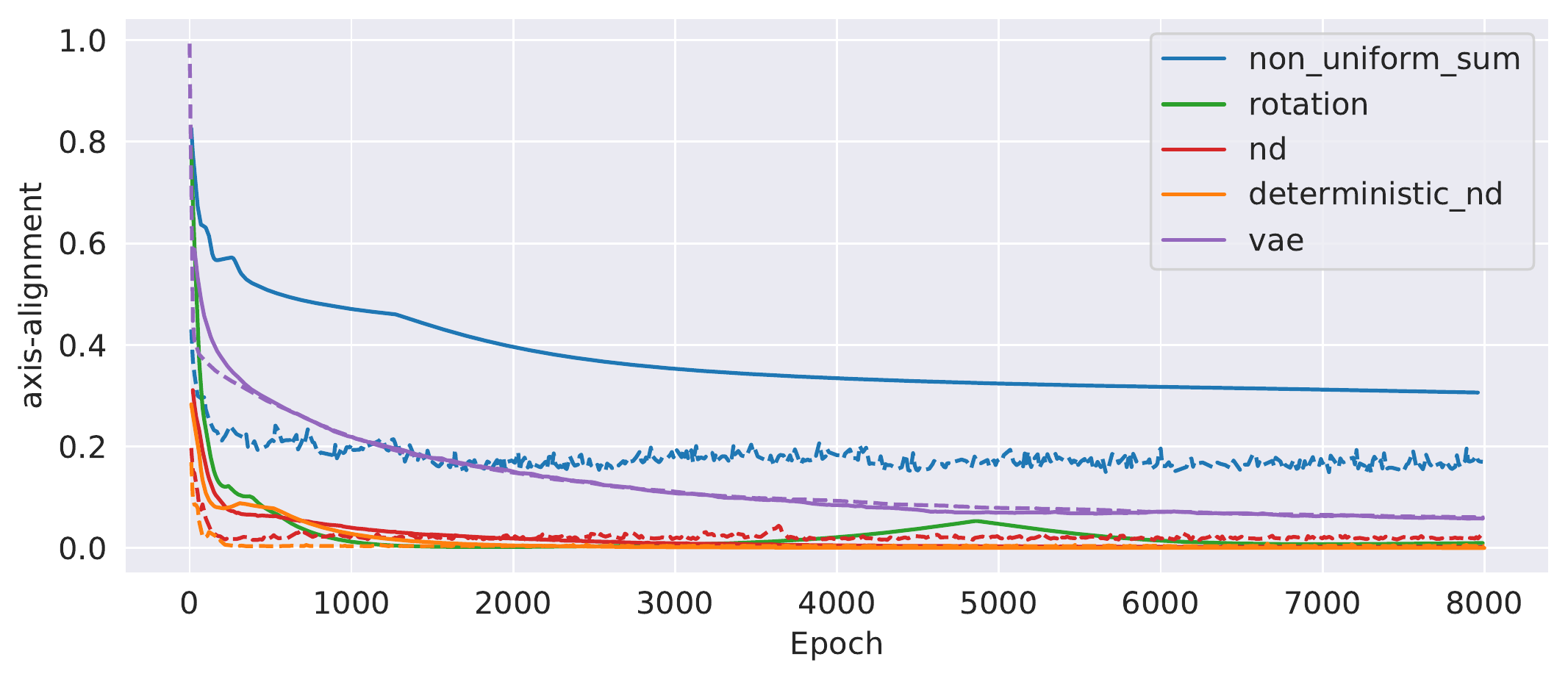}}
    \subfigure[Subspace convergence]{\label{fig:minibatch_1000_subspace_distance}\includegraphics[width=0.49\textwidth]{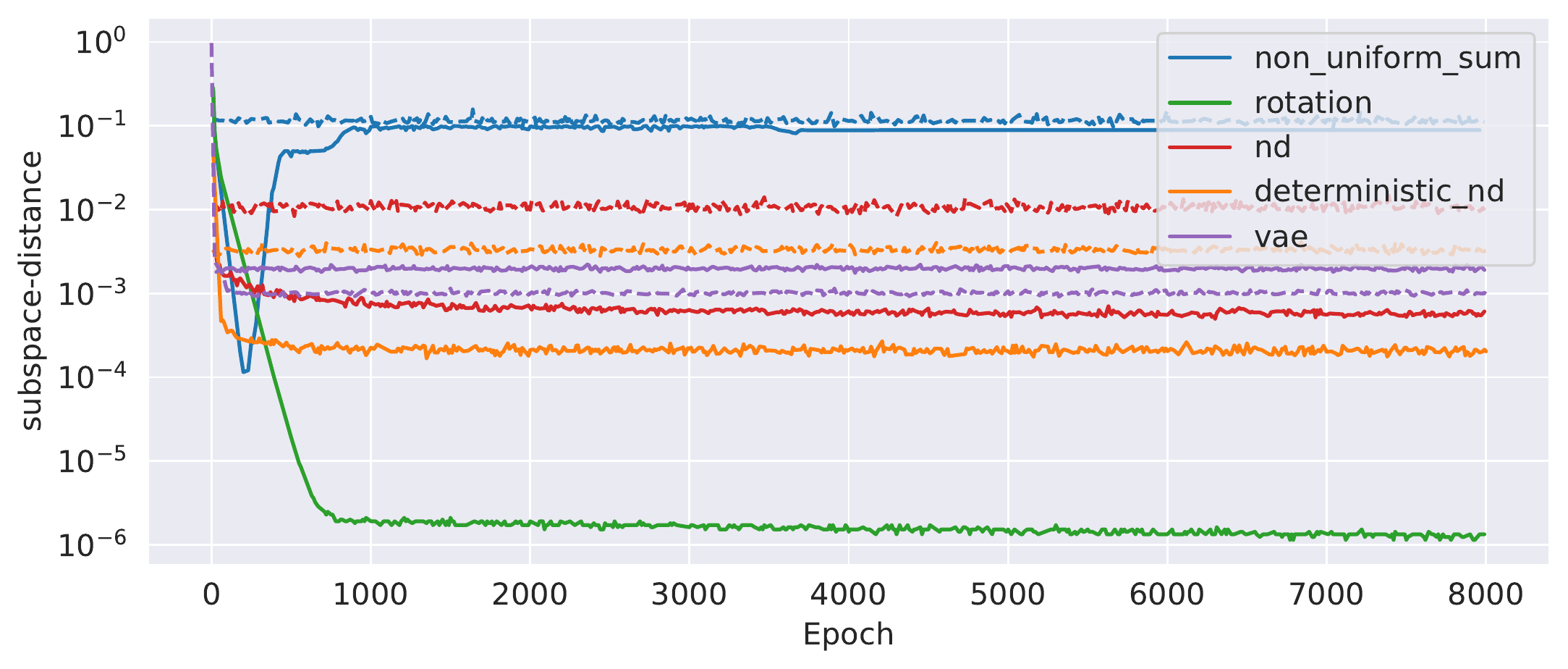}}
    \caption{Learning dynamics of different LAE / linear VAE models trained on MNIST ($k=20$), with mini-batch size 1000. Solid lines represent models trained using gradient descent with Nesterov momentum 0.9. Dashed lines represent models trained with Adam optimizer.}
    \label{fig:minibatch_1000}
\end{figure}

\begin{figure}[ht]
    \centering
    \subfigure[Axis-alignment]{\label{fig:minibatch_100_axis_alignment}\includegraphics[width=0.49\textwidth]{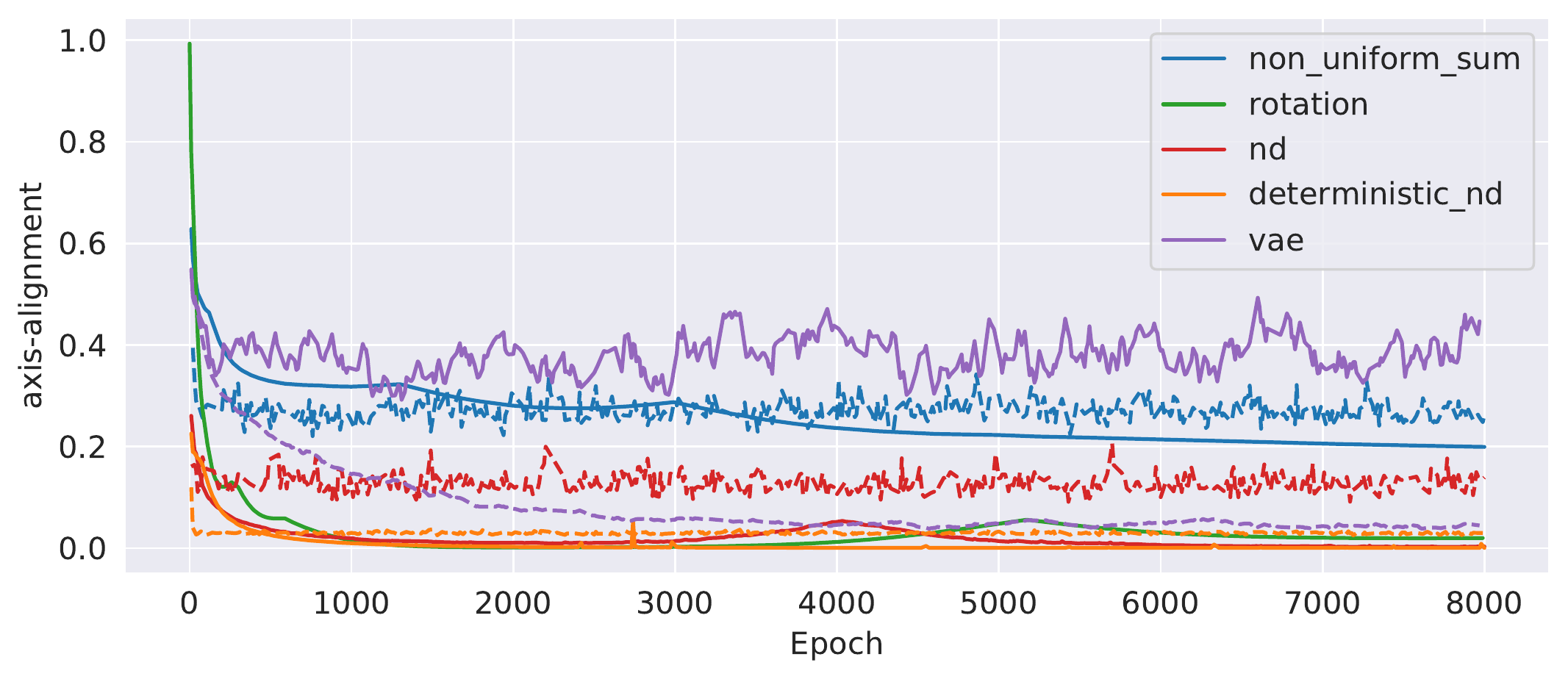}}
    \subfigure[Subspace convergence]{\label{fig:minibatch_100_subspace_distance}\includegraphics[width=0.49\textwidth]{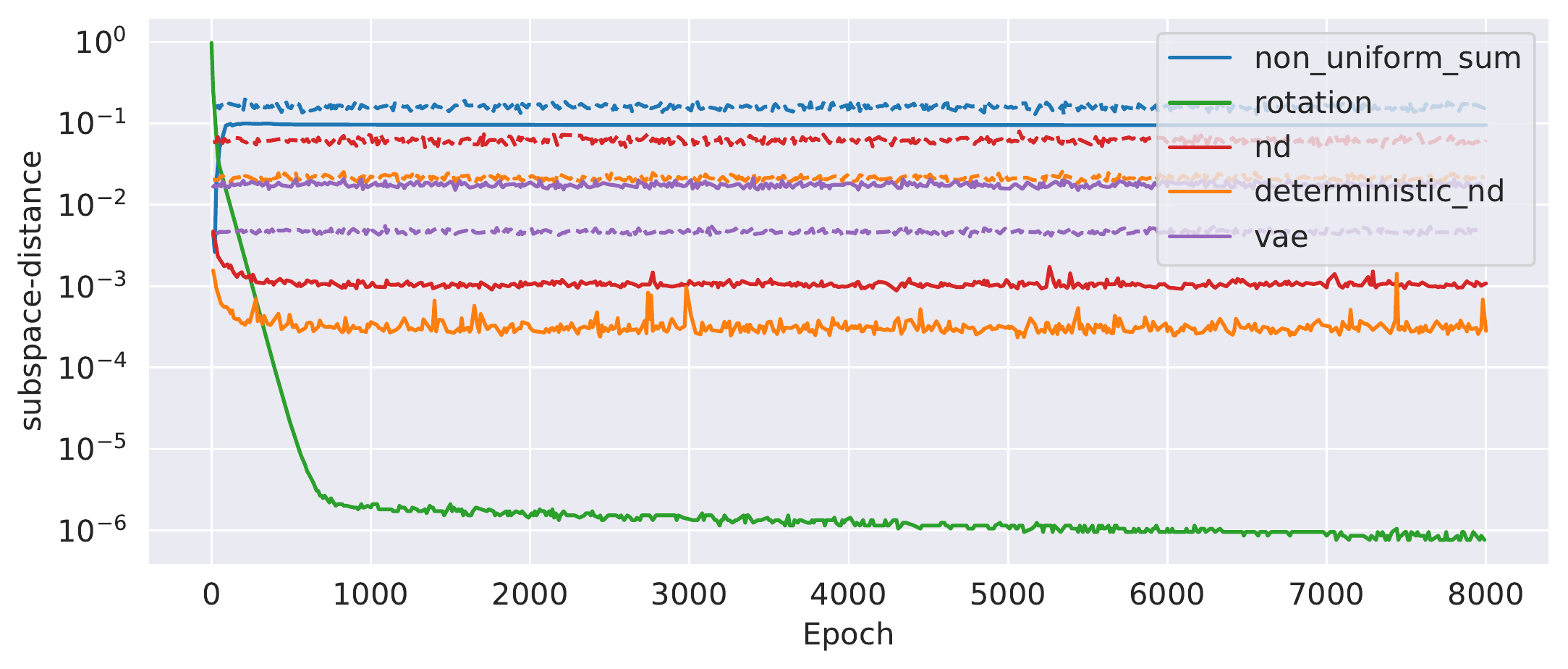}}
    \caption{Learning dynamics of different LAE / linear VAE models trained on MNIST ($k=20$), with mini-batch size 100. Solid lines represent models trained using gradient descent with Nesterov momentum 0.9. Dashed lines represent models trained with Adam optimizer.}
    \label{fig:minibatch_100}
\end{figure}
\end{document}